%% file: overparams_training_deep.tex
\documentclass{article}

\usepackage{microtype}
\usepackage{graphicx}
\usepackage{subfigure}
\usepackage{booktabs} 

\usepackage{hyperref}

\usepackage{natbib}
\usepackage{amsmath}
\usepackage{amsthm}
\usepackage{amssymb}
\usepackage{mathtools}
\usepackage{tikz}
\usepackage{xcolor}
\usetikzlibrary{arrows}

\newcommand{\poly}{\mathrm{poly}}

\newcommand{\E}{\mathbb{E}}
\newcommand{\diag}{\mat{diag}}

\newcommand{\tr}{\mathrm{tr}}

\def\cN{\mathcal{N}}

\newcommand{\mat}[1]{\mathbf{#1}}
\newcommand{\vect}[1]{\mathbf{#1}}
\newcommand{\norm}[1]{\left\|#1\right\|}
\newcommand{\inner}[1]{\left\langle#1\right\rangle}
\newcommand{\abs}[1]{\left|#1\right|}
\newcommand{\expect}{\mathbb{E}}

\newcommand{\vectorize}[1]{\text{vec}\left(#1\right)}

\newcommand{\variance}{\text{Var}}
\newcommand{\params}{\vect{\theta}}

\newcommand{\relu}[1]{\sigma\left(#1\right)}

\newcommand{\lambdamin}{\lambda_{\min}\left(\mat{K}^{(H)}\right)}

\newtheorem{thm}{Theorem}[section]
\newtheorem{lem}{Lemma}[section]

\newtheorem{prop}{Proposition}[section]

\newtheorem{defn}{Definition}[section]

\newtheorem{rem}{Remark}[section]

\newtheorem{condition}{Condition}[section]

\newcommand{\gaussian}{\mathcal{P}}
\newcommand{\linfunc}{\mathcal{L}}
\newcommand{\linsub}{\mathcal{W}}
\newcommand{\detmap}{\mathcal{D}}
\newcommand{\activate}{\rho}
\newcommand{\bias}{b}
\newcommand{\error}{\mathcal{E}}
\newcommand{\wbound}{\mathfrak{W}}
\newcommand{\rhobound}{\Lambda}
\newcommand{\gaussianspace}{{\mathcal{L}^2}}

\usepackage[accepted]{icml2019}

\icmltitlerunning{Gradient Descent Finds Global Minima of Deep Neural Networks}

\allowdisplaybreaks

\begin{document}

\twocolumn[
\icmltitle{Gradient Descent Finds Global Minima of Deep Neural Networks}



\icmlsetsymbol{equal}{*}

\begin{icmlauthorlist}
\icmlauthor{Simon S. Du}{equal,cmu}
\icmlauthor{Jason D. Lee}{equal,usc}
\icmlauthor{Haochuan Li}{equal,physics,bigdata}
\icmlauthor{Liwei Wang}{equal,key,bigdata}
\icmlauthor{Xiyu Zhai}{equal,mit}

\end{icmlauthorlist}

\icmlaffiliation{cmu}{Machine Learning Department, Carnegie Mellon University}
\icmlaffiliation{usc}{Data Science and Operations Department, University of Southern California}
\icmlaffiliation{physics}{School of Physics, Peking University}
\icmlaffiliation{key}{Key Laboratory of Machine Perception, MOE, School of EECS, Peking University}
\icmlaffiliation{bigdata}{Center for Data Science, Peking University, Beijing Institute of Big Data Research}
\icmlaffiliation{mit}{Department of EECS, Massachusetts Institute of Technology}
\icmlcorrespondingauthor{Simon S. Du}{ssdu@cs.cmu.edu}

\icmlkeywords{gradient descent, non-convex, convergence, deep neural network}

\vskip 0.3in
]



\printAffiliationsAndNotice{\icmlEqualContribution} 

\begin{abstract}
\label{sec:abs}
\input{abstract.tex}
\end{abstract}

\section{Introduction}
\label{sec:intro}
\input{intro.tex}

\section{Related Works}
\label{sec:rel}
\input{rel.tex}

\section{Preliminaries}
\label{sec:pre}
\input{pre.tex}

\section{Technique Overview}
\label{sec:technique}
\input{technique.tex}

\section{Warm Up: Convergence Result of GD for Deep Fully-connected Neural Networks}
\label{sec:deep_discrete}
\input{deep_discrete.tex}

\section{Convergence Result of GD for ResNet}
\label{sec:resnet}
\input{resnet.tex}

\section{Convergence Result of GD for Convolutional ResNet}
\label{sec:conv_resnet}
\input{conv_resnet.tex}

\section{Conclusion}
\label{sec:conclusion}
\input{conclusion.tex}

\section*{Acknowledgments}
\input{ack.tex}

\bibliography{simonduref}
\bibliographystyle{icml2019}

\onecolumn
\newpage
\appendix
\input{general_appendix.tex}

\section{Proof Sketch}
\label{sec:proof_sketch}
\input{proof_sketch.tex}

\section{Proofs for Section~\ref{sec:deep_discrete}}
\label{sec:deep_discrete_proof}
	\input{deep_discrete_proof.tex}

	\subsection{Proofs of Lemmas}
	\label{sec:technical_proofs}

\input{technical_proofs.tex}

\section{Proofs for Section~\ref{sec:resnet}}
\label{sec:resnet_proof}
\input{resnet_proof.tex}

	\subsection{Proofs of Lemmas}
	\label{sec:technical_proofs_resnet}
	\input{technical_proofs_resnet.tex}

\section{Proofs for Section~\ref{sec:conv_resnet}}
\label{sec:conv_resnet_proof}
\input{conv_resnet_proof.tex}

	\subsection{Proofs of Lemmas}
	\label{sec:technical_proofs_conv_resnet}
	\input{technical_proofs_conv_resnet.tex}

\section{Analysis of Random Initialization}
	\label{sec:general_formulation}
	\input{general_formulation.tex}

\section{Full Rankness of $\mat{K}^{(h)}$}
\label{sec:justification}
\input{justification.tex}

\section{Useful Technical Lemmas}
\label{sec:useful_lemmas}
\input{useful_lemmas.tex}

\end{document}

%% file: abstract.tex
Gradient descent finds a global minimum in training deep neural networks despite the objective function being non-convex. The current paper proves gradient descent achieves zero training loss in polynomial time for a deep over-parameterized neural network with residual connections (ResNet). Our analysis relies on the particular structure of the Gram matrix induced by the neural network architecture. This structure allows us to show the Gram matrix is stable throughout the training process and this stability implies the global optimality of the gradient descent algorithm. We further extend our analysis to deep residual convolutional neural networks and obtain a similar convergence result.

%% file: intro.tex
One of the mysteries in deep learning is randomly initialized first-order methods like gradient descent achieve zero training loss, even if the labels are arbitrary~\citep{zhang2016understanding}.
Over-parameterization is widely believed to be the main reason for this phenomenon as only if the neural network has a sufficiently large capacity, it is possible for this neural network to fit all the training data.
For example, \citet{NIPS2017_7203} proved that except for a measure zero set, all functions cannot be approximated by ReLU networks with a width less than the input dimension.
In practice, many neural network architectures are highly over-parameterized.
For example, Wide Residual Networks have 100x parameters than the number of training data~\citep{zagoruyko2016wide}.

The second mysterious phenomenon  in training deep neural networks is ``deeper networks are harder to train."
To solve this problem, \citet{he2016deep} proposed the deep residual network (ResNet) architecture which enables randomly initialized first order method to train neural networks with an order of magnitude more layers.
Theoretically, \citet{hardt2016identity} showed that residual links in linear networks prevent gradient vanishing in a large neighborhood of zero, but for neural networks with non-linear activations,  the advantages of using residual connections are not well understood.

In this paper, we demystify these two mysterious phenomena.
We consider the setting where there are $n$ data points, and the neural network has $H$ layers with width $m$.
We focus on the least-squares loss and assume the activation function is Lipschitz and smooth.
This assumption holds for many activation functions including the soft-plus and sigmoid.
Our contributions are summarized below.
\begin{itemize}
\item As a warm-up, we first consider a fully-connected feedforward network. We show if $m = \Omega\left(\poly(n)2^{O(H)}\right)$\footnote{The precise polynomials and data-dependent parameters are stated in Section~\ref{sec:deep_discrete},~\ref{sec:resnet},~\ref{sec:conv_resnet}.}, then randomly initialized gradient descent converges to zero training loss at a linear rate.
\item Next, we consider the ResNet architecture. 
We show as long as $m = \Omega\left(\poly(n,H)\right)$, then randomly initialized gradient descent converges to zero training loss at a linear rate.
Comparing with the first result, the dependence on the number of layers improves exponentially for ResNet.
This theory demonstrates the advantage of using residual architectures.
\item Lastly, we apply the same technique to analyze convolutional ResNet. 
We show if $m = \poly(n,p,H)$ where $p$ is the number of patches, then randomly initialized gradient descent achieves zero training loss.
\end{itemize}
Our proof builds on two ideas from previous work on gradient descent for two-layer neural networks. First, we use the observation by \cite{li2018learning} that if the neural network is over-parameterized, every weight matrix is close to its initialization.
Second, following \cite{du2018gradient}, we analyze the dynamics of the predictions whose convergence is determined by the least eigenvalue of the Gram matrix induced by the neural network architecture and to lower bound the least eigenvalue, it is sufficient to bound the distance of each weight matrix from its initialization.

Different from these two works, in analyzing deep neural networks, we need to exploit more structural properties of deep neural networks and develop new techniques for analyzing both the initialization and gradient descent dynamics.
In Section~\ref{sec:technique} we give an overview of our proof technique.

\subsection{Organization}
\label{sec:org}
This paper is organized as follows.
In Section~\ref{sec:rel}, we discuss related works.
In Section~\ref{sec:pre}, we formally state the problem setup.
In Section~\ref{sec:technique}, we present our main analysis techniques.
In Section~\ref{sec:deep_discrete}, we give a warm-up result for the deep fully-connected neural network.
In Section~\ref{sec:resnet}, we give our main result for the ResNet.
In Section~\ref{sec:conv_resnet}, we give our main result for the convolutional ResNet.
We conclude in Section~\ref{sec:conclusion} and defer all proofs to the appendix.

%% file: rel.tex
Recently, many works try to study the optimization problem in deep learning.
Since optimizing a neural network is a non-convex problem, one approach is first to develop a general theory for a class of non-convex problems which satisfy desired geometric properties and then identify that the neural network optimization problem belongs to this class.
One promising candidate class is the set of functions that satisfy: a) all local minima are global and b) there exists a negative curvature for every saddle point.
For this function class, researchers have shown  (perturbed) gradient descent~\citep{jin2017escape,ge2015escaping,lee2016gradient,du2017gradient} can find a global minimum.
Many previous works thus  try to study the optimization landscape of neural networks with different activation functions~\citep{soudry2017exponentially,safran2017spurious,safran2016quality,zhou2017critical,freeman2016topology,hardt2016identity,nguyen2017loss,kawaguchi2016deep,venturi2018neural,soudry2016no,du2018power,soltanolkotabi2018theoretical,haeffele2015global}.
However, even for a three-layer linear network, there exists a saddle point that does not have a negative curvature~\citep{kawaguchi2016deep}, so it is unclear whether this geometry-based approach can be used to obtain the global convergence guarantee of first-order methods.

Another way to attack this problem is to study the dynamics of a specific algorithm for a specific neural network architecture.
Our paper also belongs to this category.
Many previous works put assumptions on the input distribution and assume the label is generated according to a planted neural network.
Based on these assumptions, one can obtain global convergence of gradient descent for some shallow neural networks~\citep{tian2017analytical,soltanolkotabi2017learning,brutzkus2017globally,du2017spurious,li2017convergence,du2017convolutional}.
Some local convergence results have also been proved~\citep{zhong2017learning,zhong2017recovery,zhang2018learning}.
In comparison, our paper does not try to recover the underlying neural network. 
Instead, we focus on minimizing the training loss and rigorously prove that randomly initialized gradient descent can achieve zero training loss.

The most related papers are \cite{li2018learning,du2018gradient} who observed that when training an over-parametrized two-layer fully-connected neural network, the weights do not change a large amount, which
we also use to show the stability of the Gram matrix.
They used this observation to obtain the convergence rate of gradient descent on a two-layer over-parameterized neural network for the cross-entropy and least-squares loss. 
More recently, \citet{allen2018convergence} generalized ideas from \cite{li2018learning} to derive convergence rates of training recurrent neural networks.

Our work extends these previous results in several ways: a) we consider deep networks,  b) we generalize to ResNet architectures, and c) we generalize to  convolutional networks. To improve the width dependence $m$ on sample size $n$, we utilize a smooth activation (e.g. smooth ReLU). For example, our results specialized to depth $H=1$ improve upon \cite{du2018gradient} in the required amount of overparametrization from $m =\Omega\left( n^6\right)$ to $m=  \Omega\left( n^4\right)$. See Theorem \ref{thm:main_gd} for the precise statement.

 \citet{chizat2018note} brought to our attention the paper of \citet{jacot2018neural} which proved a similar weight stability phenomenon for deep networks, but only in the asymptotic setting of infinite-width networks and gradient flow run for a finite time. \citet{jacot2018neural} do not establish the convergence of gradient flow to a global minimizer.  In lieu of their results, our work can be viewed as a generalization of their result to: a) finite width, b) gradient descent as opposed to gradient flow, and c) convergence to a global minimizer.

\citet{mei2018mean,chizat2018global,sirignano2018mean,rotskoff2018neural,wei2018margin} used optimal transport theory to analyze gradient descent on over-parameterized models.
However, their results are limited to two-layer neural networks and may require an exponential amount of over-parametrization.

\citet{daniely2017sgd} developed the connection between deep neural networks with kernel methods and showed stochastic gradient descent can learn a function that is competitive with the best function in the conjugate kernel space of the network. 
\citet{andoni2014learning} showed that gradient descent can learn networks that are competitive with polynomial classifiers.
However, these results do not imply gradient descent can find a global minimum for the empirical loss minimization problem. Our analysis of the Gram matrices at random initialization is closely related to prior work on the analysis of infinite-width networks as Gaussian Processes~\cite{raghu2016expressive,matthews2018gaussian,lee2017deep,schoenholz2016deep}. Since we require the initialization analysis for three distinct architectures (ResNet, feed-forward, and convolutional ResNet), we re-derive many of these prior results in a unified fashion in Appendix \ref{sec:general_formulation}.

Finally, in concurrent work,  \citet{allen2018convergencetheory} also analyze gradient descent on deep neural networks. The primary difference between the two papers is that we analyze general smooth activations, and \citet{allen2018convergencetheory} develop specific analysis for ReLU activation. The two papers also differ significantly on their data assumptions. \textit{We wish to emphasize a fair comparison is not possible due to the difference in setting and data assumptions. We view the two papers as complementary since they address different neural net architectures.} 

For ResNet, the primary focus of this manuscript, the required width per layer for \citet{allen2018convergencetheory} is  $m \gtrsim n^{30} H^{30} \log^2 \frac{1}{\epsilon}$  and for this paper's Theorem \ref{thm:resnet_gd} is $m \gtrsim n^4 H^2$.\footnote{In all comparisons, we ignore the polynomial dependency on data-dependent parameters which only depends on the input data and the activation function. The two papers use different measures and are not directly comparable.}
Our paper requires a width $m$ that does not depend on the desired accuracy $\epsilon$. As a consequence, Theorem \ref{thm:resnet_gd} guarantees the convergence of gradient descent to a global minimizer. 
The iteration complexity of \citet{allen2018convergencetheory} is $T \gtrsim  n^6 H^2 \log \frac{1}{\epsilon}$ and of Theorem \ref{thm:resnet_gd} is $T \gtrsim n^2 \log \frac{1}{\epsilon}$.

For fully-connected networks, \citet{allen2018convergencetheory} requires width $m \gtrsim n^{30} H^{30} \log^2 \frac{1}{\epsilon}$ and iteration complexity $T \gtrsim n^6 H^2 \log \frac{1}{\epsilon}$. Theorem \ref{thm:gd-mlp} requires width $m \gtrsim n^4 2^{O(H)}$ and iteration complexity $T \gtrsim n^2 2^{O(H)}  \log \frac{1}{\epsilon}$. 
The primary difference is for very deep fully-connected networks, \citet{allen2018convergencetheory} has milder dependence on $H$, but worse dependence on $n$.
Commonly used fully-connected networks such as VGG are not extremely deep ($H=16$), yet the dataset size such as ImageNet ($n\sim 10^6$) is very large.

In a second concurrent work, \citet{zou2018stochastic} also analyzed the convergence of gradient descent on fully-connected networks with ReLU activation. The emphasis is on different loss functions (e.g. hinge loss), so the results are not directly comparable. Both \citet{zou2018stochastic} and \citet{allen2018convergencetheory} train a subset of the layers, instead of all the layers as in this work, but also analyze stochastic gradient.

%% file: pre.tex
\subsection{Notations}
We Let $[n] = \{1, 2, \ldots, n\}$.
We use $N(\vect{0},\mat{I})$ to denote the standard Gaussian distribution.
For a matrix $\vect A$, we use $\mat A_{ij}$ to denote its $(i, j)$-th entry. We will also use $\mat{A}_{i,:}$ to denote the $i$-th row vector of $\mat{A}$ and define $\mat{A}_{i,j:k}=(\mat{A}_{i,j},\mat{A}_{i,j+1},\cdots,\mat{A}_{i,k})$ as part of the vector. 
Similarly $\mat{A}_{:,i}$ is the $i$-th column vector and $\mat{A}_{j:k,i}$ is a part of $i$-th column vector.
For a vector $\vect{v}$, we use $\norm{\vect{v}}_2$ to denote the Euclidean norm.
For a matrix $\mat{A}$ we use $\norm{\mat{A}}_F$ to denote the Frobenius norm and $\norm{\mat{A}}_2$ to denote the operator norm.
If a matrix $\mat{A}$ is positive semi-definite, we use $\lambda_{\min}(\mat{A})$ to denote its smallest eigenvalue.
We use $\langle \cdot, \cdot \rangle$ to denote the standard Euclidean inner product between two vectors or matrices.
We let $O(\cdot)$ and $\Omega\left(\cdot\right)$ denote standard Big-O and Big-Omega notations, only hiding  constants.
In this paper we will use  $C$ and $c$ to denote constants. The specific value can be different from line to line.

\subsection{Activation Function}
We use $\relu{\cdot}$ to denote the activation function.
In this paper we impose some technical conditions on the activation function.
The guiding example is softplus: $\relu{z} = \log(1+\exp(z))$.
\begin{condition}[Lipschitz and Smooth]\label{cond:lip_and_smooth}
There exists a constant $c>0$ such that	$\abs{\relu{0}} \le c$ and for any $z,z' \in \mathbb{R}$,
\begin{align*}
\abs{\relu{z} -\relu{z'} }\le &c\abs{z-z'}, \\
\text{ and }\abs{\sigma'(z)-\sigma'(z)} \le &c \abs{z-z'}.
\end{align*}
\end{condition}
These two conditions will be used to show the stability of the training process.
Note for softplus both Lipschitz constant and smoothness constant are $1$.
In this paper, we view all activation function related parameters as constants.

\begin{condition}\label{cond:analytic}
$\relu{\cdot}$ is analytic and is not a polynomial function.
\end{condition}
This assumption is used to guarantee the positive-definiteness of certain Gram matrices which we will define later.
Softplus function satisfies this assumption by definition.

\subsection{Problem Setup}
In this paper, we focus on the empirical risk minimization problem with the quadratic loss function\begin{align}
	\min_{\params} L(\params) = \frac{1}{2}\sum_{i=1}^{n}(f(\params,\vect{x}_i)-y_i)^2 \label{eqn:loss}
\end{align}where $\left\{\vect{x}_i\right\}_{i=1}^n$ are the training inputs, $\left\{y_i\right\}_{i=1}^n$ are the labels, $\params$ is the parameter we optimize over and $f$ is the prediction function, which in our case is a neural network.
We consider the following architectures. 
\begin{itemize}
\item \textbf{Multilayer fully-connected neural networks:}
Let $\vect{x} \in \mathbb{R}^{d}$ be the input, $\mat{W}^{(1)} \in \mathbb{R}^{m \times d}$ is the first weight matrix, $\mat{W}^{(h)} \in \mathbb{R}^{m \times m}$ is the weight at the $h$-th layer for $2\le h\le H$, $\vect{a} \in \mathbb{R}^{m}$ is the output layer and $\relu{\cdot}$ is the activation function.\footnote{We assume intermediate layers are square matrices for simplicity. It is not difficult to generalize our analysis to rectangular weight matrices.} 
We define the prediction function recursively (for simplicity we let $\vect{x}^{(0)}=\vect{x}$).
\begin{align}
\vect{x}^{(h)}&=\sqrt{\frac{c_{\sigma}}{m}} \relu{\mat{W}^{(h)}\vect{x}^{(h-1)}}, 1\le h \le H\nonumber\\
f(\vect{x},\params)&=\vect{a}^\top \vect{x}^{(H)}.
\label{eqn:mlp}
\end{align}
where 
$c_{\sigma}=\left(\expect_{x\sim N(0,1)}\left[\sigma(x)^2\right]\right)^{-1}$ is a scaling factor to normalize the input in the initialization phase.

\item \textbf{ResNet}\footnote{We will refer to this architecture as ResNet, although this differs by the standard ResNet architecture since the skip-connections at every layer, instead of every two layers. 
This architecture was previously studied in \cite{hardt2016identity}.
We study this architecture for the ease of presentation and analysis.
It is not hard to generalize our analysis to architectures with skip-connections are every two or more layers.
}:
We use the same notations as the multilayer fully connected neural networks.
We define the prediction recursively.
\begin{align}
\vect{x}^{(1)}&=\sqrt{\frac{c_{\sigma}}{m}} \relu{\mat{W}^{(1)}\vect{x}}, \nonumber\\
\vect{x}^{(h)} & =\vect{x}^{(h-1)}+\frac{c_{res}}{H\sqrt{m}} \relu{\mat{W}^{(h)}\vect{x}^{(h-1)}} \nonumber\\
& ~~~~~~~~~~~~~~~~~~~~~~~~~~~~~~~\text{ for } 2\le h\le H, \nonumber\\
f_{res}(\vect{x},\params)&=\vect{a}^\top \vect{x}^{(H)}
 \label{eqn:resnet}
\end{align}
where  $0< c_{res} < 1$ is a small constant.
Note here we use a $\frac{c_{res}}{H\sqrt{m}}$ scaling. This scaling plays an important role in guaranteeing the width per layer only needs to scale polynomially with $H$.
In practice, the small scaling is enforced by a small initialization of the residual connection~\citep{hardt2016identity,zeroinit2018}, which obtains state-of-the-art performance for deep residual networks. We choose to use an explicit scaling, instead of altering the initialization scheme for notational convenience.

\item \textbf{Convolutional ResNet}:
Lastly, we consider the convolutional ResNet architecture.
Again we define the prediction function in a recursive way.

	Let $\vect{x}^{(0)} \in \mathbb{R}^{d_0\times p}$ be the input, where $d_0$ is the number of input channels and $p$ is the number of pixels.
	For $h \in [H]$, we let the number of channels be $d_h = m$ and number of pixels be $p$.
	Given $\vect{x}^{(h-1)} \in \mathbb{R}^{d_{h-1}\times p}$ for $h \in [H]$,
	we first use an operator $\phi_h(\cdot)$ to divide $\vect{x}^{(h-1)} $ into $p$ patches.
	Each patch has size $qd_{h-1}$ and this  implies a map $\phi_h(\vect{x}^{(h-1)})\in\mathbb{R}^{qd_{h-1}\times p}$. 
	For example, when the stride is $1$ and $q=3$
	\begin{align*}
	&\phi_h(\mat{x}^{(h-1)}) \\
	=&\begin{pmatrix}
	\left(\mat{x}^{(h-1)}_{1,0:2}\right)^\top, 
	&\ldots &, \left(\mat{x}^{(h-1)}_{1,p-1:p+1}\right)^\top\\
	\ldots, 
	& \ldots, & \ldots \\
	\left(\mat{x}^{(h-1)}_{d_{h-1},0:2}\right)^\top, 
	&\ldots, &  \left(\mat{x}^{(h-1)}_{d_{h-1},p-1:p+1}\right)^\top
	\end{pmatrix}
	\end{align*}

where we let $\mat{x}^{(h-1)}_{:,0}=\mat{x}^{(h-1)}_{:,p+1}=\vect{0}$, i.e., zero-padding.
Note this operator has the property\begin{align*}
 \norm{\vect{x}^{(h-1)}}_F \le \norm{\phi_h(\vect{x}^{(h-1)})}_F \le \sqrt{q}\norm{\vect{x}^{(h-1)}}_F .
	\end{align*}
because each element from $\vect{x}^{(h-1)}$ at least appears once and at most appears $q$ times.
In practice, $q$ is often small like $3\times 3$, so throughout the paper we view $q$ as a constant in our theoretical analysis.
To proceed, let $\mat{W}^{(h)} \in \mathbb{R}^{d_h \times qd_{h-1}}$, we have \begin{align*}
	\vect{x}^{(1)}  =&\sqrt{\frac{c_{\sigma}}{m}} \relu{\mat{W}^{(1)} \phi_1(\vect{x})} \in \mathbb{R}^{m\times p},\\
		\vect{x}^{(h)}  =&\vect{x}^{(h-1)}+\frac{c_{res}}{H\sqrt{m}} \relu{\mat{W}^{(h)} \phi_h(\vect{x}^{(h-1))}} \in \mathbb{R}^{m\times p} \\
&~~~~~~~~~~~~~~~~~~~~~~~~~~~~~~~~~~~~~~~~~~~~		\text{for }2\le h\le H,
	\end{align*} where  $0< c_{res} < 1$ is a small constant.
	Finally, for $\vect{a} \in \mathbb{R}^{m \times p}$, the output is defined as \begin{align*}
	f_{cnn}(\vect{x},\params) = \langle \vect{a}, \vect{x}^{(H)} \rangle.
	\end{align*}
	Note here we use the similar scaling $O(\frac{1}{H\sqrt{m}})$ as ResNet.

\end{itemize}

To learn the deep neural network, we consider the randomly initialized gradient descent algorithm to find the global minimizer of the empirical loss~\eqref{eqn:loss}.
Specifically, we use the following random initialization scheme.
For every level $h \in [H]$, each entry is sampled from a standard Gaussian distribution, $\mat{W}_{ij}^{(h)} \sim N(0,1)$ and each entry of the output layer $\vect{a}$ is also sampled from $N(0,1)$.
In this paper, we train all layers by gradient descent, for $k=1,2,\ldots,$ and $h\in[H]$\begin{align*}
\mat{W}^{(h)}(k) &= \mat{W}^{(h)}(k-1) - \eta \frac{\partial L(\params(k-1))}{\partial \mat{W}^{(h)}(k-1)},\\
\vect{a}(k) &= \vect{a}(k-1) - \eta \frac{\partial L(\params(k-1))}{\partial \vect{a}(k-1)}
\end{align*} where $\eta >0 $ is the step size.

%% file: technique.tex
In this section, we describe our main idea of proving the global convergence of gradient descent.
Our proof technique is inspired by \citet{du2018gradient} who proposed to study the dynamics of differences between labels and predictions.
Here the individual prediction at the $k$-th iteration is \[
u_i(k) = f(\params(k),\vect{x}_i)
\] and we denote $\vect{u}(k)=\left(u_1(k),\ldots,u_n(k)\right)^\top \in \mathbb{R}^n$.
\citet{du2018gradient} showed that for two-layer fully-connected neural network, the sequence  $\left\{\vect{y}-\vect{u}(k)\right\}_{k=0}^\infty$ admits the following dynamics\begin{align*}
\vect{y}-\vect{u}(k+1) =  \left(\mat{I}-\eta\mat{H}(k)\right)\left(\vect{y}-\vect{u}(k)\right)
\end{align*} where $\mat{H}(k) \in \mathbb{R}^{n \times n}$ is a Gram matrix with\footnote{This formula is for the setting that only the first layer is trained. } \[\mat{H}_{ij}(k) = \left\langle\frac{\partial u_i(k)}{\partial \mat{W}^{(1)}(k)},\frac{\partial u_j(k)}{\partial \mat{W}^{(1)}(k)} \right\rangle.\]
The key finding in \cite{du2018gradient} is that if $m$ is sufficiently large, $\mat{H}(k) \approx \mat{H}^\infty$ for all $k$ where $\mat{H}^{\infty}$ is defined as $\mat{H}^\infty_{ij} = \expect_{\vect{w}\sim N(\vect{0},\mat{I})}\left[\sigma'\left(\vect{w}^\top \vect{x}_i\right)\sigma'\left(\vect{w}^\top \vect{x}_j\right) \vect{x}_i ^\top \vect{x}_j\right]$.
Notably, $\mat{H}^\infty$ is a fixed matrix which only depends on the training input, but \emph{does not} depend on neural network parameters $\params$.
As a direct result, in the large $m$ regime,   the dynamics of $\left\{\vect{y}-\vect{u}(k)\right\}_{k=0}^\infty$ is approximately \emph{linear}\begin{align*}
	\vect{y}-\vect{u}(k+1) \approx \left(\mat{I} - \eta\mat{H}^\infty\right)\left(\vect{y}-\vect{u}(k)\right).
\end{align*}
For this linear dynamics, using standard analysis technique for power method, one can show $\left\{\vect{y}-\vect{u}(k)\right\}_{k=0}^\infty$ converges to $\vect{0}$ where the rate is determined by the least eigenvalue of $\mat{H}^\infty$ and the step size $\eta$.

We leverage this insight to our deep neural network setting.
Again we consider the sequence $\{\vect{y}-\vect{u}(k)\}_{k=0}^\infty$, which admits the dynamics \[
\vect{y}-\vect{u}(k+1)= \left(\mat{I}-\eta \mat{G}(k)\right)\left(\vect{y}-\vect{u}(k)\right)
\]
where \begin{align*}
&\mat{G}_{ij}(k) \\
= &\inner{\frac{\partial u_i(k)}{\partial \params(k)},\frac{\partial u_j(k)}{\partial \params(k)}} \\
= &\sum_{h=1}^{H}\inner{\frac{\partial u_i(k)}{\partial \mat{W}^{(h)}(k)},\frac{\partial u_j(k)}{\partial \mat{W}^{(h)}(k)}} + \inner{\frac{\partial u_i(k)}{\partial \vect{a}(k)},\frac{\partial u_j(k)}{\partial \vect{a}(k)}} \\
\triangleq & \sum_{h=1}^{H+1}\mat{G}^{(h)}_{ij}(k).
\end{align*}
Here we define $\mat{G}^{(h)} \in \mathbb{R}^{n \times n}$ with $\mat{G}_{ij}^{(h)}(k)= \inner{\frac{\partial u_i(k)}{\partial \mat{W}^{(h)}(k)}, \frac{\partial u_j(k)}{\partial \mat{W^{(h)}}(k)} }$ for $h=1,\ldots,H$ and $\mat{G}_{ij}^{(H+1)}(k) = \inner{\frac{\partial u_i(k)}{\partial \vect{a}(k)},\frac{\partial u_j(k)}{\partial \vect{a}(k)}}$.
Note for all $h\in [H+1]$, each entry of $\mat{G}^{(h)}(k)$ is an inner product.
Therefore, $\mat{G}^{(h)}(k)$ is a positive semi-definite (PSD) matrix for $h \in [H+1]$.
Furthermore, if there exists one $h \in [H]$ that $\mat{G}^{(h)}(k)$ is strictly positive definite, then if one chooses the step size $\eta$ to be sufficiently small, the loss decreases at the $k$-th iteration according the analysis of power method.
In this paper we focus on $\mat{G}^{(H)}(k)$, the gram matrix induced by the weights from $H$-th layer for simplicity at the cost of a minor degradation in convergence rate.\footnote{Using the contribution of all the gram matrices to the minimum eigenvalue can potentially improve the convergence rate.}

We use the similar observation in \cite{du2018gradient} that we show if the width is large enough for all layers, for all $k=0,1,\ldots$, $\mat{G}^{(H)}(k)$ is close to a fixed matrix $\mat{K}^{(H)} \in \mathbb{R}^{n \times n}$ which depends on the input data, neural network architecture and the activation but does not depend on neural network parameters $\params$.
According to the analysis of the power method, once we establish this, as long as $\mat{K}^{(H)}$ is strictly positive definite, then the gradient descent enjoys a linear convergence rate.
We will show for $\mat{K}^{(H)}$ is strictly positive definite as long as the training data is not degenerate (c.f. Proposition~\ref{prop:fullrank_fc} and~\ref{prop:resnet-depth-ind}). 

While following the similar high-level analysis framework proposed by \citet{du2018gradient}, analyzing the convergence of gradient descent for \emph{deep} neural network is significantly more involved and requires new technical tools.
To show $\mat{G}^{(H)}(k)$ is close to $\mat{K}^{(H)}$, we have two steps.
First, we show in the initialization phase $\mat{G}^{(H)}(0)$ is close to $\mat{K}^{(H)}$.
Second, we show during training $\mat{G}^{(H)}(k)$ is close to $\mat{G}^{(H)}(0)$ for $k=1,2,\ldots$.
Below we give overviews of these two steps.
\paragraph{Analysis of Random Initialization}
Unlike \cite{du2018gradient} in which they showed $\mat{H}(0)$ is close to $\mat{H}^\infty$ via a simple concentration inequality, showing $\mat{G}^{(H)}(0)$ is close to $\mat{K}^{(H)}$ requires more subtle calculations.
First, as will be clear in the following sections, $\mat{K}^{(H)}$ is a recursively defined matrix.
Therefore, we need to analyze how the perturbation (due to randomness of initialization and finite $m$) from lower layers propagates to the $H$-th layer.
Second, this perturbation propagation involves non-linear operations due to the activation function.
To quantitatively characterize this perturbation propagation dynamics, we use induction and leverage techniques from Malliavin calculus~\citep{malliavin1995gaussian}.
We derive a general framework that allows us to analyze the initialization behavior for the fully-connected neural network, ResNet, convolutional ResNet and other potential neural network architectures in a unified way.

One important finding in our analysis is that ResNet architecture makes the ``perturbation propagation" more stable.
The high level intuition is the following.
For fully connected neural network, suppose we have some perturbation $\norm{\mat{G}^{(1)}(0)-\mat{K}^{(1)}}_2\le \error_1$ in the first layer.
This perturbation propagates to the $H$-th layer admits the form
\begin{align}
\norm{\mat{G}^{(H)}(0)-\mat{K}^{(H)}}_2\triangleq \error_H \lesssim 2^{O(H)}\error_1. \label{eqn:why_not_fc}
\end{align}
Therefore, we need to have $\error_1 \le \frac{1}{2^{O(H)}}$ and this makes $m$ have exponential dependency on $H$.\footnote{We not mean to imply that fully-connected networks necessarily depend exponentially on $H$, but simply to illustrate in our analysis why the exponential dependence arises. For specific activations such as ReLU and careful initialization schemes, this exponential dependence may be avoided.}
 
On the other hand, for ResNet the perturbation propagation admits the form\begin{align}
\error_H \lesssim \left(1+O\left(\frac{1}{H}\right)\right)^{H}\epsilon_1 = O\left(\epsilon_1\right) \label{eqn:why_resnet}
\end{align}
Therefore we do not have the exponential explosion problem for ResNet.
We refer readers to Section~\ref{sec:general_formulation} for details.

\paragraph{Analysis of Perturbation of During Training}
The next step is to show $\mat{G}^{(H)}(k)$ is close to $\mat{G}^{(H)}(0)$ for $k=0,1,\ldots$. 
Note $\mat{G}^{(H)}$ depends on weight matrices from all layers, so to establish that $\mat{G}^{(H)}(k)$ is close to $\mat{G}^{(H)}(0)$, we need to show $\mat{W}^{(h)}(k) - \mat{W}^{(h)}(0)$ is small for all $h \in [H]$ and $\vect{a}(k)-\vect{a}(0)$ is small.

In the two-layer neural network setting~\citep{du2018gradient},  they are able to show \emph{every} weight vector of the first layer is close to its initialization, i.e., $\norm{\mat{W}^{(1)}(k)-\mat{W}^{(1)}(0)}_{2,\infty}$ is small for $k=0,1,\ldots$.
While establishing this condition for two-layer neural network is not hard, this condition may not hold for multi-layer neural networks.
In this paper, we show instead, the averaged Frobenius norm
\begin{align}
\frac{1}{\sqrt{m}}\norm{\mat{W}^{(h)}(k)-\mat{W}^{(h)}(0)}_F \label{eqn:close_to_init} 
\end{align} is small for all $k=0,1,\ldots$.

Similar to the analysis in the initialization, showing Equation~\eqref{eqn:close_to_init} is small is highly involved because again, we need to analyze how the perturbation propagates.
We develop a unified proof strategy for the fully-connected neural network, ResNet and convolutional ResNet.
Our analysis in this step again sheds light on the benefit of using ResNet architecture for training.
The high-level intuition is similar to Equation~\eqref{eqn:why_resnet}.
See Section~\ref{sec:deep_discrete_proof},~\ref{sec:resnet_proof}, and~\ref{sec:conv_resnet_proof} for details.

%% file: deep_discrete.tex
In this section, as a warm up, we show gradient descent with a constant positive step size converges to the global minimum at a linear rate.
As we discussed in Section~\ref{sec:technique}, the convergence rate depends on least eigenvalue of the Gram matrix $\mat{K}^{(H)}$.
\begin{defn}\label{defn:gram_mlp}
The Gram matrix $\mat{K}^{(H)}$ is recursively defined as follows,
for $(i,j) \in [n] \times [n]$, and $h=1,\ldots,H-1$ \begin{align}
	\mat{K}^{(0)}_{ij} = &\langle \vect{x}_i, \vect{x}_j \rangle , \nonumber \\
	\mat{A}_{ij}^{(h)} = &\begin{pmatrix}
	\mat{K}^{(h-1)}_{ii} & \mat{K}^{(h-1)}_{ij} \\
	\mat{K}^{(h-1)}_{ji} & \mat{K}^{(h-1)}_{jj}
	\end{pmatrix} ,\label{eqn:kernel_mlp}\\
	\mat{K}^{(h)}_{ij} =  &c_{\sigma} \expect_{\left(u,v\right)^\top \sim N\left(\vect{0},\mat{A}_{ij}^{(h)}\right)}\left[\relu{u}\relu{v}\right],  \nonumber \\
	\mat{K}^{(H)}_{ij} = & c_{\sigma}\mat{K}^{(H-1)}_{ij} \expect_{\left(u,v\right)^\top \sim N\left(\vect{0},\mat{A}_{ij}^{(H-1)}\right) }\left[\sigma'(u)\sigma'(v)\right] .\nonumber
\end{align}
\end{defn}
The derivation of this Gram matrix is deferred to Section~\ref{sec:general_formulation}.
The convergence rate and the amount of over-parameterization depends on the least eigenvalue of this Gram matrix.
In Section~\ref{sec:fullrank_fc} we show as long as the input training data is not degenerate, then $\lambdamin$ is strictly positive.
We remark that if $H=1$, then $\mat{K}^{(H)}$ is the same the Gram matrix defined in \cite{du2018gradient}.

Now we are ready to state our main convergence result of gradient descent for deep fully-connected neural networks.
\begin{thm}[Convergence Rate of Gradient Descent for Deep Fully-connected Neural Networks]\label{thm:main_gd}
Assume for all $i \in [n]$, $\norm{\vect{x}_i}_2 = 1$, $\abs{y_i} = O(1)$  and the number of hidden nodes per layer 
\begin{align*}
m=\Omega\left(2^{O(H)}\max\left\{
\frac{n^4}{\lambda_{\min}^4\left(\mat{K}^{(H)}\right)},\frac{n}{\delta}, \frac{n^2\log(\frac{Hn}{\delta})}{\lambda_{\min}^2\left(\mat{K}^{(H)}\right)}
\right\}\right)
\end{align*} where $\mat{K}^{(H)}$ is defined in Equation~\eqref{eqn:kernel_mlp}.
If we set the step size 
\[\eta = O\left(\frac{\lambda_{\min}\left(\mat{K}^{(H)}\right)}{n^22^{O(H)}}\right),\] 
then with probability at least $1-\delta$ over the random initialization the loss, for $k=1,2,\ldots$, the loss at each iteration satisfies
\begin{align*}
L(\params(k))\le \left(1-\frac{\eta \lambda _{\min}\left(\mat{K}^{(H)}\right)}{2}\right)^{k}L(\params(0)).
\end{align*}
\label{thm:gd-mlp}
\end{thm}

This theorem states that if the width $m$ is large enough and we set step size appropriately then gradient descent converges to the global minimum with zero loss at linear rate.
The main assumption of the theorem is that we need a large enough width of each layer.
The width $m$ depends on  $n$, $H$ and  $1/\lambdamin$.
The dependency on $n$ is only polynomial, which is the same as previous work on shallow neural networks~\citep{du2018gradient,li2018learning}.
Similar to \cite{du2018gradient}, $m$ also polynomially depends on $1/\lambdamin$.
However, the dependency on the number of layers $H$ is exponential.
As we discussed in Section~\ref{sec:technical_proofs}, this exponential comes from the instability of the fully-connected architecture (c.f. Equation~\eqref{eqn:why_not_fc}).
In the next section, we show with ResNet architecture, we can reduce the dependency on $H$ from $2^{(H)}$ to $\poly(H)$.

Note the requirement of $m$ has three terms.
The first term is used to show  the Gram matrix is stable during training. 
The second term is used to guarantee the output in each layer is approximately normalized at the initialization phase.
The third term is used to show  the perturbation  of Gram matrix at the initialization phase is small.
See Section~\ref{sec:deep_discrete_proof} for proofs.

The convergence rate depends step size $\eta$ and $\lambdamin$, similar to \cite{du2018gradient}.
Here we require $\eta = O\left(\frac{\lambda_{\min}\left(\mat{K}^{(H)}\right)}{n^22^{O(H)}}\right)$.
When $H=1$, this requirement is the same as the one used in \cite{du2018gradient}.
However, for deep fully-connected neural network, we require $\eta$ to be exponentially small in terms of number of layers.
The reason is similar to that we require $m$ to be exponentially large.
Again, this will be improved  in the next section.

%% file: resnet.tex
In this section we consider the convergence of gradient descent for training a ResNet.
We will focus on how much over-parameterization is needed to ensure the global convergence of gradient descent and compare it with fully-connected neural networks.
Again we first define the key Gram matrix whose least eigenvalue will determine the convergence rate.
\begin{defn}\label{defn:gram_resnet}
The Gram matrix $\mat{K}^{(H)}$ is recursively defined as follows, for $(i,j) \in [n] \times [n]$ and $h=2,\ldots,H-1$:
\begin{align}
\mat{K}_{ij}^{(0)} =& \langle \vect{x}_i, \vect{x}_j \rangle, \nonumber\\
\mat{K}^{(1)}_{ij} =  & \expect_{\left(u,v\right)^\top \sim N\left(\vect{0},\begin{pmatrix}
	\mat{K}^{(0)}_{ii} & \mat{K}^{(0)}_{ij} \\
	\mat{K}^{(0)}_{ji} & \mat{K}^{(0)}_{jj}
	\end{pmatrix}\right)}c_{\sigma}\relu{u}\relu{v},\nonumber\\
\vect{b}_i^{(1)} =  &\sqrt{c_{\sigma}}\expect_{u \sim N(0, \mat{K}_{ii}^{(0)})}\left[\relu{u}\right],  \nonumber \\
\mat{A}_{ij}^{(h)} = &\begin{pmatrix}
\mat{K}^{(h-1)}_{ii} & \mat{K}^{(h-1)}_{ij} \\
\mat{K}^{(h-1)}_{ji} & \mat{K}^{(h-1)}_{jj}
\end{pmatrix} \label{eqn:kernel_resnet}\\
	\mat{K}^{(h)}_{ij} =  & \mat{K}_{ij}^{(h-1)}+\nonumber\\& \expect_{\left(u,v\right)^\top \sim N\left(\vect{0},\mat{A}^{(h)}_{ij}\right)}\left[\frac{c_{res}\vect{b}_i^{(h-1)}\relu{u}}{H}\right. \nonumber\\
&	+\left.\frac{c_{res}\vect{b}_j^{(h-1)}\relu{v}}{H}+\frac{c_{res}^2\relu{u}\relu{v}}{H^2}\right], \nonumber \\
\vect{b}_i^{(h)} =  &\vect{b}_i^{(h-1)} + \frac{c_{res}}{H}\expect_{u \sim N(0, \mat{K}_{ii}^{(h-1)})}\left[\relu{u}\right],\nonumber\\
\mat{K}^{(H)}_{ij} =  &\frac{c_{res}^2}{H^2}\mat{K}^{(H-1)}_{ij}  \expect_{\left(u,v\right)^\top \sim N\left(\vect{0},\mat{A}_{ij}^{(H-1)}\right) }\left[\sigma'(u)\sigma'(v)\right].\nonumber
\end{align}
\end{defn}
Comparing $\mat{K}^{(H)}$ of the ResNet and the one of the fully-connect neural network, the definition of $\mat{K}^{(H)}$ also depends on a series of $\{\vect{b}^{(h)}\}_{h=1}^{H-1}$.
This dependency is comes from the skip connection block in the ResNet architecture.
See Section~\ref{sec:general_formulation}.
In Section~\ref{sec:fullrank_resnet}, we show as long as the input training data is not degenerate, then $\lambdamin$ is strictly positive.
Furthermore, $\lambdamin$ does not depend inversely exponentially in $H$.

Now we are ready to state our main theorem for ResNet.
\begin{thm}[Convergence Rate of Gradient Descent for ResNet]\label{thm:resnet_gd}
Assume for all $i \in [n]$, $\norm{\vect{x}_i}_2 = 1$, $\abs{y_i} = O(1)$ and the number of hidden nodes per layer \begin{align}m=&\Omega\left(\max\left\{\frac{n^4 }{\lambda_{\min}^4\left(\mat{K}^{(H)}\right)H^6},\frac{n^2 }{\lambda_{\min}^2(\mat{K}^{(H)})H^2},\right.\right.\\
&~~~~~~~~~~~~~~~~\left.\left.\frac{n}{\delta}, \frac{n^2\log\left(\frac{Hn}{\delta}\right)}{\lambda_{\min}^2\left(\mat{K}^{(H)}\right)} \right\}\right).\nonumber
\end{align}
If we set the step size $\eta = O\left(\frac{\lambdamin H^2 }{n^2}\right)$, then with probability at least $1-\delta$ over the random initialization we have for $k=1,2,\ldots$\begin{align*}
L(\params(k))\le \left(1-\frac{\eta \lambdamin}{2}\right)^{k}L(\params(0)).
\end{align*}
\end{thm}

In sharp contrast to Theorem~\ref{thm:main_gd}, this theorem is fully polynomial in the sense that both the number of neurons and the convergence rate is polynomially in $n$ and $H$.
Note the amount of over-parameterization depends on $\lambdamin$ which is the smallest eigenvalue of the $H$-th layer's Gram matrix.
The main reason that we do not have any exponential factor here is that the skip connection block makes the overall architecture more stable in both the initialization phase and the training phase.

Note the requirement on $m$ has $4$ terms.
The first two terms are used to show  the Gram matrix stable during training. 
The third term is used to guarantee the output in each layer is approximately normalized at the initialization phase.
The fourth term is used to show bound the size of the perturbation of the Gram matrix at the initialization phase.
See Section~\ref{sec:resnet_proof} for details.

%% file: conv_resnet.tex
In this section we generalize the convergence result of  gradient descent for ResNet to convolutional ResNet.
Again, we focus on how much over-parameterization is needed to ensure the global convergence of gradient descent.
Similar to previous sections, we first define the $\mat{K}^{(H)}$ for this architecture.
\begin{defn}\label{defn:gram_convresnet}
The Gram matrix $\mat{K}^{(H)}$ is recursively defined as follows, for $(i,j) \in [n] \times [n]$, $(l,r) \in [p] \times [p]$ and $h=2,\ldots,H-1$, 
\begin{align}
\mat{K}_{ij}^{(0)} = & \phi_1\left(\vect{x}_{i}\right)^\top \phi_1\left(\vect{x}_{j}\right) \in \mathbb{R}^{p\times p},  \nonumber\\
\mat{K}^{(1)}_{ij} =  &  \expect_{\left(\mat{u},\mat{v}\right)\sim N\left(\vect{0},\begin{pmatrix}
	\mat{K}^{(0)}_{ii}& \mat{K}^{(0)}_{ij} \\
	\mat{K}^{(0)}_{ji} & \mat{K}^{(0)}_{jj}
	\end{pmatrix}\right)}c_{\sigma}\relu{\mat{u}}^{\top}\relu{\mat{v}}, \nonumber\\
\vect{b}_i^{(1)} = &\sqrt{c_{\sigma}}\expect_{\mat{u}\sim N\left(\vect{0},\mat{K}^{(0)}_{ii}\right)}\left[\relu{\mat{u}}\right],\nonumber \\
\mat{A}_{ij}^{(h)} = &\begin{pmatrix}
\mat{K}^{(h-1)}_{ii}& \mat{K}^{(h-1)}_{ij} \\
\mat{K}^{(h-1)}_{ji} & \mat{K}^{(h-1)}_{jj}
\end{pmatrix} \nonumber \\
\mat{H}^{(h)}_{ij} =   &\mat{K}_{ij}^{(h-1)}+\nonumber\\
&  \expect_{\left(\mat{u},\mat{v}\right)\sim N\left(\vect{0},\mat{A}_{ij}^{(h-1)}\right)}\left[\frac{c_{res}\vect{b}_i^{(h-1)\top} \relu{\mat{u}}}{H} \right. \label{eqn:kernel_conv_resnet}\\
&\left.+\frac{c_{res}\vect{b}_j^{(h-1)\top}\relu{\mat{v}}}{H}+\frac{c_{res}^2\relu{\mat{u}}^\top\relu{\mat{v}}}{H^2}\right],\nonumber \\
	\mat{K}^{(h)}_{ij,lr} =&\tr\left(\mat{H}^{(h)}_{ij,D_l^{(h)}D_r^{(h)}}\right), \nonumber\\
\vect{b}_i^{(h)} = & \vect{b}_i^{(h-1)} + \frac{c_{res}}{H}\expect_{\mat{u}\sim N\left(\vect{0},\mat{K}^{(h-1)}_{ii}\right)}\left[\relu{\mat{u}}\right]
 \nonumber\\
 \mat{M}^{(H)}_{ij,lr}=&\mat{K}^{(H-1)}_{ij,lr}\expect_{\left(\mat{u},\mat{v}\right)\sim N\left(\vect{0},\mat{A}_{ij}^{(H-1)}\right) }\left[\sigma'(u_l)\sigma'(v_r)\right]\nonumber\\
\mat{K}^{(H)}_{ij} = & \tr(\mat{M}^{(H)}_{ij})
\nonumber
\end{align}
where $\vect{u}$ and $\vect{v}$ are both random row vectors and $D_l^{(h)}\triangleq\{s:\vect{x}^{(h-1)}_{:,s} \in \text{the $l^{th}$ patch}\}$.
\end{defn}
Note here $\mat{K}_{ij}^{(h)}$ has dimension $p \times p$ for $h=0,\ldots,H-1$  and $\mat{K}_{ij,lr}$ denotes the $(l,r)$-th entry.

Now we state our main convergence theorem for the convolutional ResNet.
\begin{thm}[Convergence Rate of Gradient Descent for Convolutional ResNet]\label{thm:convresnet_gd}
Assume for all $i \in [n]$, $\norm{\vect{x}_i}_F = 1$, $\abs{y_i} = O(1)$ and the number of hidden nodes per layer 
\begin{align}
m=&\Omega\left(\max\left\{\frac{n^4 }{\lambda_0^4H^6},\frac{n^4 }{\lambda_0^4H^2},\right.\right. \nonumber\\
&\left.\left.
~~~~~~~~~~~~~~~~\frac{n}{\delta}, \frac{n^2\log\left(\frac{Hn}{\delta}\right)}{\lambda_0^2}\right\}\poly(p)\right).
\end{align}
If we set the step size $\eta = O\left(\frac{\lambda_0H^2}{n^2\poly\left(p\right)}\right) $, then with probability at least $1-\delta$ over the random initialization we have for $k=1,2,\ldots$\begin{align*}
L(\params(k))\le \left(1-\frac{\eta \lambdamin}{2}\right)^{k}L(\params(0)).
\end{align*}
\end{thm}
This theorem is similar to that of ResNet.
The number of neurons required per layer is only polynomial in the depth and the number of data points and step size is only polynomially small.
The only extra term is $\poly(p)$ in the requirement of $m$ and $\eta$.
The analysis is also similar to ResNet and we refer readers to Section~\ref{sec:conv_resnet_proof} for details.

%% file: conclusion.tex
In this paper, we show that gradient descent on deep overparametrized networks can obtain zero training loss. 
Our proof builds on a careful analysis of the random initialization scheme and a perturbation analysis which shows that the Gram matrix is increasingly  stable under overparametrization.
These techniques allow us to show that every step of gradient descent decreases the loss at a geometric rate. 

We list some directions for future research:
\begin{enumerate}
	\item The current paper focuses on the training loss, but does not address the test loss. It would be an important problem to show that gradient descent can also find solutions of low test loss. In particular, existing work only demonstrate that gradient descent works under the same situations as kernel methods and random feature methods \citep{daniely2017sgd,li2018learning,allen2018learning,arora2019fine}. To further investigate of generalization behavior, we believe some algorithm-dependent analyses may be useful~\citep{pmlr-v48-hardt16,pmlr-v75-mou18a,2018arXiv180401619C}.
	\item The width of the layers $m$ is polynomial in all the parameters for the ResNet architecture, but still very large. Realistic networks have number of parameters, not width, a large constant multiple of $n$. We consider improving the analysis to cover commonly utilized networks an important open problem.
	\item The current analysis is for gradient descent, instead of stochastic gradient descent. We believe the analysis can be extended to stochastic gradient, while maintaining the linear convergence rate. 
\item The convergence rate can be potentially improved if the minimum eigenvalue takes into account the contribution of all Gram matrices, but this would considerably complicate the initialization and perturbation analysis.
	\end{enumerate}

%% file: ack.tex
We thank Lijie Chen and Ruosong Wang for useful discussions. 
SSD acknowledges support from AFRL grant FA8750-17-2-0212 and DARPA D17AP00001.
JDL acknowledges support of the ARO under MURI Award W911NF-11-1-0303.  This is part of the collaboration between US DOD, UK MOD and UK Engineering and Physical Research Council (EPSRC) under the Multidisciplinary University Research Initiative. HL and LW acknowlege support from National Basic Research Program of China (973 Program) (grant no.
2015CB352502), NSFC (61573026) and BJNSF (L172037).
Part of the work is done while SSD was visiting Simons Institute.

%% file: general_appendix.tex
\noindent\LARGE{\textbf{Appendix}}

\normalsize
In the proof we will use the  geometric series function $g_{\alpha}(n)=\sum_{i=0}^{n-1}\alpha^i$  extensively. Some constants we will define below may be different for different network structures, such as $c_x$, $c_{w,0}$ and $c_{x,0}$. We will also use $c$ to denote a small enough constant, which may be different in different lemmas.
For simplicity, we use $\lambda_0$ to denote $\lambdamin$ in the proofs.

%% file: proof_sketch.tex
Note  we can write the loss as \[
L(\params(k)) = \frac{1}{2}\norm{\vect{y}-\vect{u}(k)}_2^2.
\]
Our proof is by induction. 
Our induction hypothesis is just the following convergence rate of empirical loss.
\begin{condition}\label{cond:linear_converge}
	At the $k$-th iteration, we have \begin{align*}
	\norm{\vect{y}-\vect{u}(k)}_2^2 \le (1-\frac{\eta \lambda_0}{2})^{k} \norm{\vect{y}-\vect{u}(0)}_2^2.
	\end{align*}
\end{condition}
Note this condition implies the conclusions we want to prove.
To prove Condition~\ref{cond:linear_converge}, we consider one iteration on the loss function.
	\begin{align}
	&\norm{\vect{y}-\vect{u}(k+1)}_2^2 \nonumber\\
	= &\norm{\vect{y}-\vect{u}(k) - (\vect{u}(k+1)-\vect{u}(k))}_2^2 \nonumber\\
	= & \norm{\vect{y}-\vect{u}(k)}_2^2 - 2 \left(\vect{y}-\vect{u}(k)\right)^\top \left(\vect{u}(k+1)-\vect{u}(k)\right) + \norm{\vect{u}(k+1)-\vect{u}(k)}_2^2. \label{eqn:loss_expansion}
	\end{align}
This equation shows if $ 2 \left(\vect{y}-\vect{u}(k)\right)^\top \left(\vect{u}(k+1)-\vect{u}(k)\right) >  \norm{\vect{u}(k+1)-\vect{u}(k)}_2^2$, the loss decreases.
Note both terms involves $\vect{u}(k+1)-\vect{u}(k)$, which we will carefully analyze.
To simplify notations, we define \[
u_i'(\params)\triangleq \frac{\partial u_i}{\partial \params}, 
\quad u_i'^{(h)}(\params) \triangleq \frac{\partial u_i}{\partial \mat{W}^{(h)}}
, \quad   u_i'^{(a)}(\params) \triangleq \frac{\partial u_i}{\partial \vect{a}} \quad \text{ and } \quad
L'(\params) = \frac{\partial L(\params)}{\partial \params}, 
 \quad L'^{(h)}(\vect{W}^{(h)}) = \frac{\partial L(\params)}{\partial \mat{W}^{(h)}} , \quad   L'^{(a)}(\params) \triangleq \frac{\partial L}{\partial \vect{a}} 
 .\] 
We look one coordinate of $\vect{u}(k+1)-\vect{u}(k)$.

Using Taylor expansion, we have
	\begin{align*}
	&u_i(k+1) - u_i(k)\\
	=&u_i\left(\params(k)-\eta L'(\params(k))\right)-u_i\left(\params(k)\right)\\
	=&-\int_{s=0}^{\eta} \langle L'(\params(k)), u'_i\left(\params(k)-s L'(\params(k))\right)\rangle ds\\
	=&-\int_{s=0}^{\eta} \langle L'(\params(k)), u'_i\left(\params(k)\right) \rangle ds+\int_{s=0}^{\eta} \langle L'(\params(k)), u'_i\left(\params(k)\right) -u'_i\left(\params(k)-s L'(\params(k))\right) \rangle ds\\
	\triangleq& I^i_1(k)+I^i_2(k).
	\end{align*}
Denote $\vect{I}_1(k) = \left(I_1^1(k),\ldots,I_1^n(k)\right)^\top$ and $\vect{I}_2(k) = \left(I_2^1(k),\ldots,I_2^n(k)\right)^\top$ and so $\vect{u}(k+1)-\vect{u}(k) = \vect{I}_1(k) + \vect{I}_2(k)$.
We will show the $\vect{I}_1(k)$ term, which is proportional to $\eta$, drives the loss function to decrease and the $\vect{I}_2(k)$ term, which is a perturbation term but it is proportional to $\eta^2$ so it is small.
We further unpack the $I_1^i(k)$ term, 
	\begin{align*}
	I^i_1=&-\eta \langle L'(\params(k)), u'_i\left(\params(k)\right) \rangle \\
	=&-\eta \sum_{j=1}^{n}(u_j-y_j) \langle u'_j(\params(k)), u'_i\left(\params(k)\right) \rangle\\
	\triangleq& -\eta \sum_{j=1}^{n}(u_j-y_j) \sum_{h=1}^{H+1}\mat{G}^{(h)}_{ij}(k)
	\end{align*}
According to Section~\ref{sec:technique}, we will only look at $\mat{G}^{(H)}$ matrix which has the following form
\[
\mat{G}^{(H)}_{i,j}(k) = (\vect{x}_i^{(H-1)}(k))^\top \vect{x}_j^{(H-1)}(k)\cdot \frac{c_{\sigma}}{m}\sum_{r=1}^m a_r^2 \sigma'((\params_r^{(H)}(k))^\top\vect{x}_i^{(H-1)}(k))\sigma'((\params_r^{(H)}(k))^\top\vect{x}_j^{(H-1)}(k)).
\]
Now we analyze $\mat{I}_1(k)$.
We can write $\vect{I}_1$ in a more compact form with $\mat{G}(k)$.
\begin{align*}
\vect{I}_1(k) = -\eta \mat{G}(k)\left(\vect{u}(k)-\vect{y}\right).
\end{align*}
Now observe that \begin{align*}
	(\vect{y}-\vect{u}(k))^\top \vect{I}_1(k) =  &\eta \left(\vect{y}-\vect{u}(k)\right)^\top \mat{G}(k)(\vect{y}-\vect{u}(k))  \\
	\ge &\lambda_{\min}\left(\mat{G}(k)\right)\norm{\vect{y}-\vect{u}(k)}_2^2\\
	\ge & \lambda_{\min}\left(\mat{G}^{(H)}(k)\right)\norm{\vect{y}-\vect{u}(k)}_2^2
\end{align*}
Now recall the progress of loss function in Equation~\eqref{eqn:loss_expansion}:\begin{align*}
&\norm{\vect{y}-\vect{u}(k+1)}_2^2  \\
= &\norm{\vect{y}-\vect{u}(k)}_2^2 - 2 \left(\vect{y}-\vect{u}(k)\right)^\top\vect{I}_1(k) -2\left(\vect{y}-\vect{u}(k)\right)^\top\vect{I}_2(k)+  \norm{\vect{u}(k+1)-\vect{u}(k)}_2^2 \\
\le & \left(1-\eta \lambda_{\min}\left(\mat{G}^{(H)}(k)\right)\right)\norm{\vect{y}-\vect{u}(k)}_2^2 -2\left(\vect{y}-\vect{u}(k)\right)^\top\vect{I}_2(k)+  \norm{\vect{u}(k+1)-\vect{u}(k)}_2^2.
\end{align*}
For the perturbation terms, through standard calculations, we can show both $
-2\left(\vect{y}-\vect{u}(k)\right)^\top\vect{I}_2(k)$ and  $\norm{\vect{u}(k+1)-\vect{u}(k)}_2
$  are proportional to $\eta^2 \norm{\vect{y}-\vect{u}(k)}_2^2$
so if we set $\eta$ sufficiently small, this term is smaller than $\eta \lambda_{\min}\left(\mat{G}^{(H)}(k)\right)\norm{\vect{y}-\vect{u}(k)}_2^2$ and thus the loss function decreases with a linear rate.

Therefore, to prove the induction hypothesis, it suffices to prove $\lambda_{\min}\left(\mat{G}^{(H)}(k)\right) \ge \frac{\lambda_0}{2}$ for $k'=0,\ldots,k$, where $\lambda_0$ is independent of $m$.
To analyze the least eigenvalue, we first look at the initialization.
Using assumptions of the population Gram matrix and concentration inequalities, we can show  at the beginning $\norm{\mat{G}^{(H)}(0)-\mat{K}^{(H)}(0)}_2 \le \frac{1}{4}\lambda_0$, which implies
\begin{align*}
	\lambda_{\min}\left(\mat{G}^{(H)}(0)\right) \ge \frac{3}{4}\lambda_0.
\end{align*}

Now for the $k$-th iteration, by matrix perturbation analysis, we know it is sufficient to show $\norm{\mat{G}^{(H)}(k)-\mat{G}^{(H)}(0)}_2 \le \frac{1}{4}\lambda_0$.
To do this, we use a similar approach as in \cite{du2018gradient}. We show as long as $m$ is large enough, every weight matrix is close its initialization in a relative error sense. Ignoring all other parameters except $m$, $\norm{\mat{W}^{(h)}(k) - \mat{W}^{(h)}(0)}_F \lesssim 1$, and thus the average per-neuron distance from initialization is $\frac{\norm{\mat{W}^{(h)}(k) - \mat{W}^{(h)}(0)}_F}{\sqrt{m}} \lesssim \frac{1}{\sqrt{m}}$ which tends to zero as $m$ increases. See Lemma \ref{lem:dist_from_init} for precise statements with all the dependencies.

This fact in turn shows $\norm{\mat{G}^{(H)}(k)-\mat{G}^{(H)}(0)}_2$ is small.
The main difference from \cite{du2018gradient} is that we are considering deep neural networks, and when translating the  small deviation, $\norm{\mat{W}^{(h)}(k) - \mat{W}^{(h)}(0)}_F$ to $\norm{\mat{G}^{(H)}(k)-\mat{G}^{(H)}(0)}_2$, there is an amplification factor which depends on the neural network architecture.

For deep fully connected neural networks, we show this amplification factor is exponential in $H$.
On the other hand, for ResNet and convolutional ResNet we show this amplification factor is only polynomial in $H$.
We further show the width $m$ required is proportional to this amplification factor.

%% file: deep_discrete_proof.tex
We first derive the formula of  the gradient for the multilayer fully connected neural network \begin{align*}
\frac{\partial L(\params)}{\partial \mat{W}^{(h)}}
= & \left(\frac{c_{\sigma}}{m}\right)^{\frac{H-h+1}{2}}\sum_{i=1}^{n}\left(f(\vect{x}_i,\params)-y_i\right) \vect{x}_i^{(h-1)}\vect{a}^\top \left(\prod_{k=h+1}^{H}\mat{J}_i^{(k)}\mat{W}^{(k)}\right) \mat{J}_i^{(h)}
\end{align*} where \[
\mat{J}^{(h')} \triangleq \diag\left(
\sigma' \left((\vect{w}_{1}^{(h')})^\top \vect{x}^{(h'-1)}\right) , \ldots,
\sigma' \left((\vect{w}_{m}^{(h')})^\top \vect{x}^{(h'-1)}\right)
\right) \in \mathbb{R}^{m \times m}
\] are the derivative matrices induced by the activation function and \begin{align*}
\vect{x}^{(h')} =\sqrt{\frac{c_{\sigma}}{m}} \relu{\mat{W}^{(h')}\vect{x}^{(h'-1)}}.
\end{align*} is the output of the $h'$-th layer.

Through standard calculation, we can get the expression of $\mat{G}^{(H)}_{i,j}$ of the following form
\begin{align}
\mat{G}^{(H)}_{i,j} = (\vect{x}_i^{(H-1)})^\top \vect{x}_j^{(H-1)}\cdot \frac{c_{\sigma}}{m}\sum_{r=1}^m a_r^2 \sigma'((\vect{w}_r^{(H)})^\top\vect{x}_i^{(H-1)})\sigma'((\vect{w}_r^{(H)})^\top\vect{x}_j^{(H-1)}). \label{eqn:H_mlp}
\end{align}

We first present a lemma which shows with high probability the feature of each layer is approximately normalized.
\begin{lem}[Lemma on Initialization Norms]
	\label{lem:lem:init_norm}
	If $\sigma(\cdot)$ is $L-$Lipschitz and $m = \Omega\left(\frac{nHg_C(H)^2}{\delta}\right)$, where $C\triangleq c_{\sigma}L\left(2\abs{\sigma(0)}\sqrt{\frac{2}{\pi}}+2L\right)$,
	then with probability at least $1-\delta$ over random initialization, for every $h\in[H]$ and $i \in [n]$, we have 
	\[ \frac{1}{c_{x,0}}\le  \norm{\vect{x}_i^{(h)}(0)}_2 \le c_{x,0} \] where $c_{x,0} = 2$.
\end{lem}

We follow the proof sketch described in Section~\ref{sec:proof_sketch}.
We first analyze the spectral property of $\mat{G}^{(H)}(0)$ at the initialization phase.
The following lemma lower bounds its least eigenvalue.
This lemma is a direct consequence of results in Section~\ref{sec:general_formulation}.
\begin{lem}[Least Eigenvalue at the Initialization]\label{lem:mlp_least_eigen}
If $m = \Omega\left(\frac{n^2\log(Hn/\delta)2^{O(H)}}{\lambda_0^2}\right)$, we have \begin{align*}
	\lambda_{\min}(\mat{G}^{(H)}(0)) \ge \frac{3}{4}\lambda_0.
\end{align*}
\end{lem}

Now we proceed to analyze the training process.
We prove the following lemma which characterizes how the perturbation from weight matrices propagates to the input of each layer. This Lemma is used to prove the subsequent lemmas.
\begin{lem}\label{lem:pertubation_of_neuron}
	Suppose for every $h\in[H]$, $\norm{\mat{W}^{(h)}(0)}_2 \le c_{w,0}\sqrt{m}$, $\norm{\vect{x}^{(h)}(0)}_2 \le c_{x,0}$ and $\norm{\mat{W}^{(h)}(k)-\mat{W}^{(h)}(0)}_F \le \sqrt{m} R$ for some constant $c_{w,0},c_{x,0} > 0$ and $R \le c_{w,0}$.
	If $\sigma(\cdot)$ is $L-$Lipschitz, we have \begin{align*}
	\norm{\vect{x}^{(h)}(k)-\vect{x}^{(h)}(0)}_2 \le \sqrt{c_{\sigma}}Lc_{x,0}g_{c_x}(h)R
	\end{align*} where $c_x=2\sqrt{c_{\sigma}}Lc_{w,0}$.

\end{lem}

Here the assumption of $\norm{\mat{W}^{(h)}(0)}_2 \le c_{w,0}\sqrt{m}$ can be shown using Lemma~\ref{lem:operator_norm_of_random_matrix} and taking union bound over $h\in[H]$, where $c_{w,0}$ is a universal constant. Next, we show with high probability over random initialization,  perturbation in weight matrices leads to small perturbation in the Gram matrix. 
\begin{lem}\label{lem:close_to_init_small_perturbation_smooth}
	Suppose $\sigma(\cdot)$ is $L-$Lipschitz and $\beta-$smooth. Suppose for $h\in[H]$, $\norm{\mat{W}^{(h)}(0)}_2\le c_{w,0}\sqrt{m}$, $\norm{\vect{a}(0)}_2\le a_{2,0}\sqrt{m}$, $\norm{\vect{a}(0)}_4\le a_{4,0}m^{1/4}$ , $\frac{1}{c_{x,0}}\le\norm{\vect{x}^{(h)}(0)}_2 \le c_{x,0}$,  if $\norm{\mat{W}^{(h)}(k)-\mat{W}^{(h)}(0)}_F, \norm{\vect{a}(k)-\vect{a}(0)}_2 \le \sqrt{m}R$ where $R \le c g_{c_x}(H)^{-1}\lambda_0n^{-1}$ and $R\le c g_{c_x}(H)^{-1}$ for some small constant $c$ and $c_x = 2\sqrt{c_{\sigma}}Lc_{w,0}$, we have \begin{align*}
	\norm{\mat{G}^{(H)}(k) - \mat{G}^{(H)}(0)}_2 \le \frac{\lambda_0}{4}.
	\end{align*}

\end{lem}

Here the assumption of $\norm{\vect{a}(0)}_2\le a_{2,0}\sqrt{m}$, $\norm{\vect{a}(0)}_4\le a_{4,0}m^{1/4}$ can be easily obtained using standard concentration inequalities, where $a_{2,0}$ and $a_{4,0}$ are both universal constants. The following lemma shows if the induction holds, we have every weight matrix close to its initialization.
\begin{lem}\label{lem:dist_from_init}
	If Condition~\ref{cond:linear_converge} holds for $k'=1,\ldots,k$, we have for any $s =1,\ldots,k+1$
	\begin{align*}
	&\norm{\mat{W}^{(h)}(s)-\mat{W}^{(h)}(0)}_F, \norm{\vect{a}(s)-\vect{a}(0)}_2 \le  R'\sqrt{m}\\
	&\norm{\mat{W}^{(h)}(s)-\mat{W}^{(h)}(s-1)}_F,  \norm{\vect{a}(s)-\vect{a}(s-1)}_2\le \eta Q'(s-1)
	\end{align*}where $R'=\frac{16c_{x,0}a_{2,0}\left(c_x\right)^H \sqrt{n} \norm{\vect{y}-\vect{u}(0)}_2}{\lambda_0\sqrt{m}}  \le cg_{c_x}(H)^{-1}$ for some small constant $c$ with $c_x=\max\{2\sqrt{c_{\sigma}}Lc_{w,0},1\}$ and $ Q'(s)= 4c_{x,0}a_{2,0}\left(c_x\right)^{H}\sqrt{n} \norm{\vect{y}-\vect{u}(s)}_2$
\end{lem}

Now we proceed to analyze the perturbation terms.

\begin{lem}\label{lem:i2_mlp}
If Condition~\ref{cond:linear_converge} holds for $k'=1,\ldots,k$, suppose $\eta\le c\lambda_0\left(n^{2}H^2(c_x)^{3H}g_{2c_x}(H)\right)^{-1}$ for some small constant $c$, we have \begin{align*}
	\norm{\vect{I}_2(k)}_2 \le \frac{1}{8}\eta \lambda_0 \norm{\vect{y}-\vect{u}(k)}_2.
\end{align*}
\end{lem}

\begin{lem}\label{lem:quadratic_term}
If Condition~\ref{cond:linear_converge} holds for $k'=1,\ldots,k$, suppose $\eta\le c\lambda_0\left(n^{2}H^2(c_x)^{2H}g_{2c_x}(H)\right)^{-1}$ for some small constant $c$, then we have
$\norm{\vect{u}(k+1)-\vect{u}(k)}_2^2\le \frac{1}{8}\eta \lambda_0 \norm{\vect{y}-\vect{u}(k)}_2^2$.
\end{lem}

We now proceed with the proof of Theorem \ref{thm:main_gd}.	By induction, we assume Condition~\ref{cond:linear_converge} for all $k'<k$. Using Lemma \ref{lem:dist_from_init}, this establishes 
\begin{align*}
\norm{\mat{W}^{(h)}(k)-\mat{W}^{(h)}(0)}_F& \le  R'\sqrt{m}\\
& \le R \sqrt{m} \tag{ using the choice of $m$ in the theorem.}
\end{align*}
By Lemma \ref{lem:close_to_init_small_perturbation_smooth}, this establishes $\lambda_{\min} (\mat{G}^{(H)} (k)) \ge \frac{\lambda_0}{2}$.

With these estimates in hand, we are ready to prove the induction hypothesis of Condition~\ref{cond:linear_converge}.
	\begin{align*}
	&\norm{\vect{y}-\vect{u}(k+1)}_2^2 \\
	&= \norm{\vect{y}-\vect{u}(k)}_2^2 -2\eta \left(\vect{y}-\vect{u}(k)\right)^\top  \mat{G}(k) \left(\vect{y}-\vect{u}(k)\right)  -2  \left(\vect{y}-\vect{u}(k)\right)^\top\vect{I}_2 +\norm{\vect{u}(k+1)-\vect{u}(k)}_2^2\\
	 &\le \norm{\vect{y}-\vect{u}(k)}_2^2 -2\eta \left(\vect{y}-\vect{u}(k)\right)^\top \mat{G}^{(H)}(k) \left(\vect{y}-\vect{u}(k)\right)  -2  \left(\vect{y}-\vect{u}(k)\right)^\top\vect{I}_2 +\norm{\vect{u}(k+1)-\vect{u}(k)}_2^2\\
	 &\le (1-\eta \lambda_0) \norm{\vect{y}-\vect{u}(k)}_2^2 -2  \left(\vect{y}-\vect{u}(k)\right)^\top\vect{I}_2 +\norm{\vect{u}(k+1)-\vect{u}(k)}_2^2\\
	&\le  (1-\frac{\eta\lambda_0}{2})\norm{\vect{y}-\vect{u}(k)}_2^2 .
	\end{align*}
The first inequality drops the positive terms $\left(\vect{y}-\vect{u}(k)\right)^\top \sum_{h\in[H+1],h\neq H} \mat{G}^{(h)}(k) \left(\vect{y}-\vect{u}(k)\right)$. The second inequality uses the argument above that establishes $\lambda_{\min} (\mat{G}^{(H)} (k)) \ge \frac{\lambda_0}{2}$. The third inequality uses Lemmas \ref{lem:i2_mlp} and \ref{lem:quadratic_term}.

%% file: technical_proofs.tex
\begin{proof}[Proof of Lemma~\ref{lem:lem:init_norm}]
	We will bound $\norm{\vect{x}_i^{(h)}(0)}_2$ by induction on layers. The induction hypothesis is that with probability at least $1-(h-1)\frac{\delta}{nH}$ over $\mat{W}^{(1)}(0),\ldots,\mat{W}^{(h-1)}(0)$, for every $1\le h' \le h-1$, $\frac{1}{2}\le 1-\frac{g_{C}(h')}{2g_{C}(H)}\le\norm{\vect{x}_i^{(h')}(0)}_2 \le 1+\frac{g_{C}(h')}{2g_{C}(H)}\le 2$. Note that it is true for $h=1$.
	We calculate the expectation  of $\norm{\vect{x}_i^{(h)}(0)}_2^2$ over the randomness from $\mat{W}^{(h)}(0)$.
Recall
	\begin{align*}
	\norm{\vect{x}_i^{(h)}(0)}_2^2=\frac{c_{\sigma}}{m}\sum_{r=1}^{m} \sigma\left(\vect{w}^{(h)}_{r}(0)^\top \vect{x}_i^{(h-1)}(0) \right)^2.
	\end{align*}
	Therefore we have 
	\begin{align*}
		\expect\left[\norm{\vect{x}_i^{(h)}(0)}_2^2\right] = &c_{\sigma} \expect\left[
		\relu{\vect{w}_{r}^{(h)}(0)^\top \vect{x}_i^{(h-1)}(0)}^2
		\right] \\
		= & c_{\sigma}\expect_{X\sim N(0,1)} \sigma(\norm{\vect{x}_i^{(h-1)}(0)}_2 X)^2.
	\end{align*}
	Note that $\sigma(\cdot)$ is $L-$Lipschitz, for any $\frac{1}{2}\le \alpha \le 2$, we have
	\begin{align*}
&\abs{\expect_{X\sim N(0,1)} \sigma(\alpha X)^2-\expect_{X\sim N(0,1)}\sigma(X)^2}\\
\le&\expect_{X\sim N(0,1)}\abs{\sigma(\alpha X)^2-\sigma(X)^2}\\
\le&L\abs{\alpha -1} \expect_{X\sim N(0,1)}\abs{X\left(\sigma(\alpha X)+\sigma( X)\right)}\\
\le &L\abs{\alpha -1}\expect_{X\sim N(0,1)}\abs{X}\left(\abs{2\sigma(0)}+L\abs{(\alpha +1)X}\right)\\
\le & L\abs{\alpha -1}\left(2\abs{\sigma(0)}\expect_{X\sim N(0,1)}\abs{X}+L\abs{\alpha +1}\expect_{X\sim N(0,1)}X^2 \right)\\
=&L\abs{\alpha -1}\left(2\abs{\sigma(0)}\sqrt{\frac{2}{\pi}}+L\abs{\alpha +1} \right)\\
\le& \frac{C}{c_{\sigma}}\abs{\alpha-1},
	\end{align*}
where $C\triangleq c_{\sigma}L\left(2\abs{\sigma(0)}\sqrt{\frac{2}{\pi}}+2L\right)$, which implies
\begin{align*}
	1-\frac{Cg_{C}(h-1)}{2g_{C}(H)}\le	\expect\left[\norm{\vect{x}_i^{(h)}(0)}_2^2\right]\le 1+\frac{Cg_{C}(h-1)}{2g_{C}(H)}.
\end{align*}
For the variance we have \begin{align*}
\variance\left[\norm{\vect{x}_i^{(h)}(0)}_2^2\right] = &\frac{c_{\sigma}^2}{m} \variance\left[\relu{\vect{w}_r^{(h)}(0)^\top \vect{x}_i^{(h-1)}(0)}^2\right] \\
\le & \frac{c_{\sigma}^2}{m} \expect\left[\relu{\vect{w}_r^{(h)}(0)^\top \vect{x}_i^{(h-1)}(0)}^4\right] \\
\le & \frac{c_{\sigma}^2}{m} \expect\left[\left(\abs{\sigma(0)}+
L\abs{\vect{w}_r^{(h)}(0)^\top \vect{x}_i^{(h-1)}(0)}\right)^4\right] \\
\le &\frac{C_2}{m} .
\end{align*}
where $C_2\triangleq \sigma(0)^4+8\abs{\sigma(0)}^3L\sqrt{2/\pi}+24\sigma(0)^2L^2+64\sigma(0)L^3\sqrt{2/\pi}+512L^4$ and the last inequality we used the formula for the first four absolute moments of Gaussian. 
	
Applying Chebyshev's inequality and plugging in our assumption on $m$, we have  with probability $1-\frac{\delta}{nH}$ over $\mat{W}^{(h)}$,
	\[ \abs{ \norm{\vect{x}_i^{(h)}(0)}_2^2-\expect \norm{\vect{x}_i^{(h)}(0)}_2^2 } \le \frac{1}{2g_C(H)}. \]
Thus with probability $1-h\frac{\delta}{nH}$ over $\mat{W}^{(1)},\ldots,\mat{W}^{(h)}$,
\begin{align*}
 \abs{ \norm{\vect{x}_i^{(h)}(0)}_2-1 } \le \abs{ \norm{\vect{x}_i^{(h)}(0)}_2^2-1 } \le \frac{Cg_{C}(h-1)}{2g_{C}(H)}+\frac{1}{2g(H)}=\frac{g_{C}(h)}{2g_{C}(H)}.
\end{align*}
Using union bounds over $[n]$, we prove the lemma.
\end{proof}

\begin{proof}[Proof of Lemma~\ref{lem:pertubation_of_neuron}]
We prove this lemma by induction.
Our induction hypothesis is \begin{align*}
\norm{\vect{x}^{(h)}(k)-\vect{x}^{(h)}(0)}_2 \le \sqrt{c_{\sigma}}LRc_{x,0}g_{c_x}(h),
\end{align*}	where \begin{align*}
c_x=2\sqrt{c_{\sigma}}Lc_{w,0}.
\end{align*}
For $h=0$, since the input data is fixed, we know the induction hypothesis holds.
Now suppose the induction hypothesis holds for $h'=0,\ldots,h-1$, we consider $h'=h$.
\begin{align*}
\norm{\vect{x}^{(h)}(k)-\vect{x}^{(h)}(0)}_2 = &\sqrt{\frac{c_{\sigma}}{m}} \norm{\relu{\mat{W}^{(h)}(k) \vect{x}^{(h-1)}(k)} -\relu{\mat{W}^{(h)}(0) \vect{x}^{(h-1)}(0)}}_2 \\
\le & \sqrt{\frac{c_{\sigma}}{m}} \norm{\relu{\mat{W}^{(h)}(k) \vect{x}^{(h-1)}(k)} -\relu{\mat{W}^{(h)}(k) \vect{x}^{(h-1)}(0)}}_2\\ 
& + \sqrt{\frac{c_{\sigma}}{m}} \norm{\relu{\mat{W}^{(h)}(k) \vect{x}^{(h-1)}(0)} -\relu{\mat{W}^{(h)}(0) \vect{x}^{(h-1)}(0)}}_2 \\
\le & \sqrt{\frac{c_{\sigma}}{m}}L\left(
\norm{\mat{W}^{(h)}(0)}_2 + \norm{\mat{W}^{(h)}(k)-\mat{W}^{(h)}(0)}_F
\right) \cdot \norm{\vect{x}^{(h-1)}(k)-\vect{x}^{(h-1)}(0)}_2 \\
& + \sqrt{\frac{c_{\sigma}}{m}}L\norm{\mat{W}^{(h)}(k)-\mat{W}^{(h)}(0)}_F \norm{\vect{x}^{h-1}(0)}_2 \\
\le &\sqrt{\frac{c_{\sigma}}{m}}L\left(c_{w,0}\sqrt{m}+R\sqrt{m}\right)\sqrt{c_{\sigma}}LRc_{x,0}g_{c_x}(h-1) + \sqrt{\frac{c_{\sigma}}{m}}L \sqrt{m} R c_{x,0} \\
\le&\sqrt{c_{\sigma}}LRc_{x,0}\left(c_xg_{c_x}(h-1)+1\right)\\
\le& \sqrt{c_{\sigma}}LRc_{x,0}g_{c_x}(h). 
\end{align*}
\end{proof}

\begin{proof}[Proof of Lemma~\ref{lem:close_to_init_small_perturbation_smooth}]
	Because Frobenius-norm of a matrix is bigger than the operator norm, it is sufficient to bound $\norm{\mat{G}^{(H)}(k) - \mat{G}^{(H)}(0)}_F$. 
	For simplicity define $z_{i,r}(k) = \vect{w}_r^{(H)}(k)^\top \vect{x}_i^{(H-1)}(k)$, we have
	\begin{align*}
	&\abs{\mat{G}_{i,j}^{(H)}(k)-\mat{G}_{i,j}^{(H)}(0)} \\
	= &\big{|}\vect{x}_i^{(H-1)}(k)^\top \vect{x}_j^{(H-1)}(k)
	\frac{c_{\sigma}}{m}\sum_{r=1}^{m}a_r(k)^2\sigma'\left(z_{i,r}(k)\right)\sigma'\left(z_{j,r}(k)\right)
	\\
	&-\vect{x}_i^{(H-1)}(0)^\top \vect{x}_j^{(H-1)}(0)
	\frac{c_{\sigma}}{m}\sum_{r=1}^{m}a_r(0)^2\sigma'\left(z_{i,r}(0)\right)\sigma'\left(z_{j,r}(0)\right)
	\big{|} \\
	\le & \abs{\vect{x}_i^{(H-1)}(k)^\top \vect{x}_j^{(H-1)}(k) - \vect{x}_i^{(H-1)}(0)^\top \vect{x}_j^{(H-1)}(0)} \frac{c_{\sigma}}{m}\sum_{r=1}^{m}a_r(0)^2\abs{\sigma'\left(z_{i,r}(k)\right)\sigma'\left(z_{j,r}(k)\right) }\\
	& + \abs{\vect{x}_i^{(H-1)}(0)^\top \vect{x}_j^{(H-1)}(0)} \frac{c_{\sigma}}{m} \abs{\sum_{r=1}^{m}a_r(0)^2\left(	\sigma'\left(z_{i,r}(k)\right)\sigma'\left(z_{j,r}(k)\right)
		- 
		\sigma'\left(z_{i,r}(0)\right)\sigma'\left(z_{j,r}(0)\right)\right)
	} \\
	&+\abs{\vect{x}_i^{(H-1)}(k)^\top \vect{x}_j^{(H-1)}(k)} \frac{c_{\sigma}}{m} \abs{\sum_{r=1}^{m} \left(a_r(k)^2-a_r(0)^2\right)
		\sigma'\left(z_{i,r}(k)\right)\sigma'\left(z_{j,r}(k)\right)
	} \\
	\le & L^2c_{\sigma}a_{2,0}^2\abs{\vect{x}_i^{(H-1)}(k)^\top \vect{x}_j^{(H-1)}(k) - \vect{x}_i^{(H-1)}(0)^\top \vect{x}_j^{(H-1)}(0)} \\
	&+ c_{x,0}^2 \frac{c_{\sigma}}{m} \abs{\sum_{r=1}^{m}a_r(0)^2\left(	\sigma'\left(z_{i,r}(k)\right)\sigma'\left(z_{j,r}(k)\right)
		- 
		\sigma'\left(z_{i,r}(0)\right)\sigma'\left(z_{j,r}(0)\right)\right)
	}\\
&+4L^2c_{x,0}^2 \frac{c_{\sigma}}{m}\sum_{r=1}^{m}\abs{a_r(k)^2-a_r(0)^2}\\
	\triangleq& I_1^{i,j} + I_2^{i,j}+I_3^{i,j}.
	\end{align*}
	For $I_1^{i,j}$, using Lemma~\ref{lem:pertubation_of_neuron}, we have \begin{align*}
	I_1^{i,j} = &L^2c_{\sigma}a_{2,0}^2\abs{\vect{x}_i^{(H-1)}(k)^\top \vect{x}_j^{(H-1)}(k) - \vect{x}_i^{(H-1)}(0)^\top \vect{x}_j^{(H-1)}(0)}  \\
	\le & L^2c_{\sigma}a_{2,0}^2\abs{
		(\vect{x}_i^{(H-1)}(k)-\vect{x}_i^{(H-1)}(0))^\top \vect{x}_j^{(H-1)}(k)} + L^2c_{\sigma}a_{2,0}^2\abs{
		\vect{x}_i^{(H-1)}(0)^\top(\vect{x}_j^{(H-1)}(k)-\vect{x}_j^{(H-1)}(0))}  \\
	\le &c_{\sigma}a_{2,0}^2\sqrt{c_{\sigma}}L^3c_{x,0}g_{c_x}(H)R \cdot (c_{x,0} + \sqrt{c_{\sigma}}Lc_{x,0}g_{c_x}(H)R) +  c_{\sigma}\sqrt{c_{\sigma}}a_{2,0}^2L^3c_{x,0}g_{c_x}(H)R c_{x,0} \\
	\le &3c_{\sigma}a_{2,0}^2 c_{x,0}^2 \sqrt{c_{\sigma}}L^3g_{c_x}(H)R .
	\end{align*}
	For $I_{2}^{i,j}$, we have \begin{align*}
	I_2^{i,j} =&c_{x,0}^2 \frac{c_{\sigma}}{m} \abs{\sum_{r=1}^{m}
		a_r(0)^2\sigma'\left(z_{i,r}(k)\right)\sigma'\left(z_{j,r}(k)\right)
		- 
			a_r(0)^2\sigma'\left(z_{i,r}(0)\right)\sigma'\left(z_{j,r}(0)\right)
	}\\
	\le &c_{x,0}^2 \frac{c_{\sigma}}{m}\sum_{r=1}^{m}		a_r(0)^2\abs{\left( \sigma'\left(z_{i,r}(k)\right)-\sigma'\left(z_{i,r}(0)\right) \right)\sigma'\left(z_{j,r}(k)\right)} +	a_r(0)^2\abs{\left( \sigma'\left(z_{j,r}(k)\right)-\sigma'\left(z_{j,r}(0)\right) \right)\sigma'\left(z_{i,r}(0)\right)}\\
	\le & \frac{\beta L c_{\sigma}c_{x,0}^2}{m} \left(
	\sum_{r=1}^{m} 	a_r(0)^2\abs{z_{i,r}(k)-z_{i,r}(0)}+	a_r(0)^2\abs{z_{j,r}(k)-z_{j,r}(0)}
	\right) \\
	\le & \frac{\beta Lc_{\sigma}a_{4,0}^2 c_{x,0}^2}{\sqrt{m}}\left(\sqrt{\sum_{r=1}^{m}\abs{z_{i,r}(k)-z_{i,r}(0)}^2}+\sqrt{\sum_{r=1}^{m}\abs{z_{j,r}(k)-z_{j,r}(0)}^2}\right) .
	\end{align*}
	Using the same proof for Lemma~\ref{lem:pertubation_of_neuron}, it is easy to see \begin{align*}
	\sum_{r=1}^{m}\abs{z_{i,r}(t)-z_{i,r}(0)}^2 \le c_{x,0}^2g_{c_x}(H)^2mR^2 .
	\end{align*}
	Thus
	\begin{align*}
	I_2^{i,j}\le 2\beta c_{\sigma}a_{4,0}^2c_{x,0}^3 Lg_{c_x}(H)R .
	\end{align*}
	For $I_3^{i,j}$, 
	\begin{align*}
	I_3^{i,j}&=4L^2c_{x,0}^2 \frac{c_{\sigma}}{m}\sum_{r=1}^{m}\abs{a_r(k)^2-a_r(0)^2}\\
	&\le4L^2c_{x,0}^2 \frac{c_{\sigma}}{m}\sum_{r=1}^{m} \abs{a_r(k)-a_r(0)}\abs{a_r(k)}+\abs{a_r(k)-a_r(0)}\abs{a_r(0)}\\
	&\le 12L^2c_{x,0}^2 c_{\sigma}a_{2,0}R.
	\end{align*}
	Therefore we can bound the perturbation\begin{align*}
	\norm{\mat{G}^{(H)}(t) - \mat{G}^{(H)}(0)}_F=&\sqrt{\sum_{(i,j)}^{{n,n}} \abs{\mat{G}_{i,j}^{(H)}(t)-\mat{G}_{i,j}^{(H)}(0)}^2} \\
	\le &\left[\left(2\beta c_{x,0}a_{4,0}^2+3\sqrt{c_{\sigma}}L^2\right)L c_{\sigma}c_{x,0}^2a_{2,0}^2g_{c_x}(H)+12L^2c_{x,0}^2 c_{\sigma}a_{2,0}\right]nR.  \\
	\end{align*}
	Plugging in the bound on $R$, we have the desired result.
	
\end{proof}

\begin{proof}[Proof of Lemma~\ref{lem:dist_from_init}]
	We will prove this corollary by induction. The induction hypothesis is
	\begin{align*}
\norm{\mat{W}^{(h)}(s)-\mat{W}^{(h)}(0)}_F &\le \sum_{s'=0}^{s-1} (1-\frac{\eta \lambda_0}{2})^{s'/2}\frac{1}{4}\eta \lambda_0 R'\sqrt{m}\le R'\sqrt{m}, s\in [k+1],\\
\norm{\vect{a}(s)-\vect{a}(0)}_2 &\le \sum_{s'=0}^{s-1} (1-\frac{\eta \lambda_0}{2})^{s'/2}\frac{1}{4}\eta \lambda_0 R'\sqrt{m}\le R'\sqrt{m}, s\in [k+1].
	\end{align*}
	First it is easy to see it holds for $s'=0$. Now suppose it holds for $s'=0,\ldots,s$, we consider $s'=s+1$. 
	We have
	\begin{align*} 
	&\norm{\mat{W}^{(h)}(s+1)-\mat{W}^{(h)}(s)}_F\\
	= &\eta \norm{
		\left(\frac{c_{\sigma}}{m}\right)^{\frac{H-h+1}{2}} \sum_{i=1}^{n}(y_i-u_i(s))\vect{x}_i^{(h-1)}(s)\left(\vect{a}(s)^\top \left(\prod_{k=h+1}^{H}\mat{J}_i^{(k)}(s)\mat{W}^{(k)}(s) \right)\mat{J}_i^{(h)}(s)\right)
	}_F\\
	\le & \eta \left(\frac{c_{\sigma}}{m}\right)^{\frac{H-h+1}{2}} \norm{\vect{a}(s)}_2 \sum_{i=1}^{n}\abs{y_i-u_i(s)}\norm{\vect{x}_i^{(h-1)}(s)}_2 \prod_{k=h+1}^H \norm{\mat{W}^{(k)}(s)}_2\prod_{k=h}^H \norm{\mat{J}^{(k)}(s)}_2,\\
	&\norm{\vect{a}(s+1)-\vect{a}(s)}_2
	= \eta \norm{
	 \sum_{i=1}^{n}(y_i-u_i(s))\vect{x}_i^{(H)}(s)
	}_2.
	\end{align*}
	To bound $\norm{\vect{x}_i^{(h-1)}(s)}_2$, we can just apply Lemma~\ref{lem:pertubation_of_neuron} and get \begin{align*}
	\norm{\vect{x}_i^{(h-1)}(s)}_2 \le \sqrt{c_{\sigma}}Lc_{x,0}g_{c_x}(h)R' + c_{x,0}\le 2c_{x,0}.
	\end{align*}
	To bound $\norm{\mat{W}^{(k)}(s)}_2$, we use our assumption \begin{align*}
	\prod_{k=h+1}^H \norm{\mat{W}^{(k)}(s)}_2 \le &\prod_{k=h+1}^H \left(
	\norm{\mat{W}^{(k)}(0)}_2 + \norm{\mat{W}^{(k)}(s)-\mat{W}^{(k)}(0)}_2\right)\\
	\le & \prod_{k=h+1}^{H}(c_{w,0}\sqrt{m}+R'\sqrt{m}) \\
	= & \left(c_{w,0}+R'\right)^{H-h}m^{\frac{H-h}{2}}\\
	\le& \left(2c_{w,0}\right)^{H-h}m^{\frac{H-h}{2}}.
	\end{align*}
	Note that $\norm{\mat{J}^{(k)}(s)}_2\le L$. Plugging in these two bounds back, we obtain \begin{align*}
	\norm{\mat{W}^{(h)}(s+1)-\mat{W}^{(h)}(s)}_F \le& 4\eta c_{x,0}a_{2,0} c_x^{H}\sum_{i=1}^{n}\abs{y_i-u(s)}\\
	\le& 4\eta c_{x,0}a_{2,0} c_x^{H}\sqrt{n} \norm{\vect{y}-\vect{u}(s)}_2\\
	=& \eta Q'(s)\\
	\le& (1-\frac{\eta \lambda_0}{2})^{s/2}\frac{1}{4}\eta \lambda_0 R'\sqrt{m}.
	\end{align*}
	Similarly, we have \begin{align*}
	\norm{\vect{a}(s+1)-\vect{a}(s)}_2 \le& 2\eta c_{x,0} \sum_{i=1}^{n}\abs{y_i-u(s)}\\
	\le& \eta Q'(s)\\
	\le& (1-\frac{\eta \lambda_0}{2})^{s/2}\frac{1}{4}\eta \lambda_0 R'\sqrt{m}.
	\end{align*}
	
	Thus
	\begin{align*}
	&\norm{\mat{W}^{(h)}(s+1)-\mat{W}^{(h)}(0)}_F\\
	\le& \norm{\mat{W}^{(h)}(s+1)-\mat{W}^{(h)}(s)}_F+\norm{\mat{W}^{(h)}(s)-\mat{W}^{(h)}(0)}_F\\
	\le&\sum_{s'=0}^{s} \eta (1-\frac{\eta \lambda_0}{2})^{s'/2}\frac{1}{4}\eta \lambda_0 R'\sqrt{m}.\\
	\end{align*}
	Similarly,
		\begin{align*}
	&\norm{\vect{a}(s+1)-\vect{a}(0)}_2\\
	\le&\sum_{s'=0}^{s} \eta (1-\frac{\eta \lambda_0}{2})^{s'/2}\frac{1}{4}\eta \lambda_0 R'\sqrt{m}.\\
	\end{align*}
\end{proof}

\begin{proof}[Proof of Lemma~\ref{lem:i2_mlp}]
Fix $i \in [n]$, we bound
	\begin{align*}
	\abs{I_2^i(k)} \le &\eta \max_{0\le s\le \eta} \sum_{h=1}^{H}  \norm{L'^{(h)}(\params(k))}_F \norm{  u'^{(h)}_i\left(\params(k)\right) -u'^{(h)}_i\left(\params(k)-s L'^{(h)}(\params(k))\right) }_F .
	\end{align*}
For the gradient norm, we have
	\begin{align*}
	& \norm{L'^{(h)}(\params(k))}_F \\
	= &\norm{
		\left(\frac{c_{\sigma}}{m}\right)^{\frac{H-h+1}{2}} \sum_{i=1}^{n}(y_i-u_i(k))\vect{x}_i^{(h-1)}(k)\left(\vect{a}(k)^\top \left(\prod_{l=h+1}^{H}\mat{J}_i^{(l)}(k)\mat{W}^{(l)}(k)\right)\mat{J}_i^{(h)}(k)\right)
	}_F.
	\end{align*}
	Similar to the proof for  Lemma~\ref{lem:dist_from_init}, we have  
	\begin{align*}
	\norm{L'^{(h)}(\params(k))}_F \le Q'(k).
	\end{align*}
	Let $\params(k,s)=\params(k)-s L'(\params(k))$,
	\begin{align*}
	&\norm{  u'^{(h)}_i\left(\params(k)\right) - u'^{(h)}_i\left(\params(k,s)\right)}_F\\
	=&\left(\frac{c_{\sigma}}{m}\right)^{\frac{H-h+1}{2}} \norm{ \vect{x}_i^{(h-1)}(k)\left(\vect{a}(k)^\top \left(\prod_{l=h+1}^{H}\mat{J}_i^{(l)}(k)\mat{W}^{(l)}(k)\right)\mat{J}_i^{(h)}(k)\right) \right.\\&\left.-\vect{x}_i^{(h-1)}(k,s)\left(\vect{a}(k,s)^\top \left(\prod_{l=h+1}^{H}\mat{J}_i^{(l)}(k,s)\mat{W}^{(l)}(k,s)\right)\mat{J}_i^{(h)}(k,s)\right)}_F
	\end{align*}
Through standard calculations, we have
	\begin{align*}
	\norm{	\mat{W}^{(l)}(k)-\mat{W}^{(l)}(k,s)}_F \le &\eta Q'(k),\\
	\norm{	\vect{a}(k)-\vect{a}(k,s)}_2 \le &\eta Q'(k),\\
	\norm{	\vect{x}_i^{(h-1)}(k)-\vect{x}_i^{(h-1)}(k,s)}_F \le &2\eta  \sqrt{c_{\sigma}}Lc_{x,0} g_{2c_x}(H) \frac{Q'(k)}{\sqrt{m}},\\
	\norm{	\mat{J}_i^{(l)}(k)-\mat{J}_i^{(l)}(k,s)}_F \le &2\eta \beta  \sqrt{c_{\sigma}}Lc_{x,0} g_{2c_x}(H) Q'(k).
	\end{align*}
	According to Lemma~\ref{lem:difference_norm}, we have
	\begin{align*}
	&\norm{  u'^{(h)}_i\left(\vect{w}(k)\right) - u'^{(h)}_i\left(\vect{w}(k,s)\right)}_F \\
	\le &4c_{x,0}a_{2,0}c_x^H\eta\frac{Q'(k)}{\sqrt{m}}\left(\frac{H}{2}+\left[\frac{1}{2c_{x,0}}+\frac{H\beta \sqrt{m}}{L}\right]2\sqrt{c_{\sigma}}Lc_{x,0} g_{2c_x}(H) \right)\\
	\le &16H\sqrt{c_{\sigma}}c_{x,0}^2a_{2,0}c_x^Hg_{2c_x}(H) \beta \eta Q'(k).
	\end{align*}
	Thus we have
	\begin{align*}
	\abs{I^i_2}\le 16H^2\sqrt{c_{\sigma}}c_{x,0}^2a_{2,0}c_x^Hg_{2c_x}(H) \beta \eta^2  Q'(k)^2.
	\end{align*}
	Since this holds for all $i \in [n]$, plugging in $\eta$ and noting that $\norm{\vect{y}-\vect{u}(0)}_2=O(\sqrt{n})$, we have \begin{align*}
	\norm{\vect{I}_2(k)}_2	\le \frac{1}{8}\eta \lambda_0 \norm{\vect{y}-\vect{u}(k)}_2.
	\end{align*} 
\end{proof}

\begin{proof}[Proof of Lemma~\ref{lem:quadratic_term}]
	\begin{align*}
	\norm{\vect{u}(k+1)-\vect{u}(k)}_2^2 = & \sum_{i=1}^{n}\left(\vect{a}(k+1)^\top \vect{x}_i^{(H)}(k+1)-\vect{a}(k)^\top \vect{x}_i^{(H)}(k)\right)^2 \\
	= & \sum_{i=1}^{n}\left(\left[\vect{a}(k+1)-\vect{a}(k)\right]^\top \vect{x}_i^{(H)}(k+1)+\vect{a}(k)^\top \left[\vect{x}_i^{(H)}(k+1)-\vect{x}_i^{(H)}(k)\right] \right)^2 \\
	\le &2\norm{\vect{a}(k+1)-\vect{a}(k)}_2^2\sum_{i=1}^{n}\norm{\vect{x}_i^{(H)}(k+1)}_2^2+2\norm{\vect{a}(k)}_2^2\sum_{i=1}^{n}\norm{\vect{x}_i^{(H)}(k+1)-\vect{x}_i^{(H)}(k)}_2^2\\
	\le &8n\eta^2c_{x,0}^2Q'(k)^2+4 n \left(2\eta  \sqrt{c_{\sigma}}Lc_{x,0}a_{2,0}^2 g_{2c_x}(H)Q'(k)\right)^2\\
	\le &\frac{1}{8}\eta \lambda_0 \norm{\vect{y}-\vect{u}(k)}_2^2.
	\end{align*}
\end{proof}

%% file: resnet_proof.tex
The gradient for ResNet is 
\begin{align*}
\frac{\partial L}{\partial \mat{W}^{(h)}}=& \frac{c_{res}}{H\sqrt{m}}
\sum_{i=1}^{n}(y_i-u_i)\vect{x}_i^{(h-1)} \cdot
\left[\vect{a}^\top \prod_{l=h+1}^{H}\left(\mat{I}+\frac{c_{res}}{H\sqrt{m}}\mat{J}_i^{(l)}\mat{W}^{(l)} \right) \mat{J}_i^{(h)}\right]
\end{align*}
For ResNets, $\mat{G}^{(H)}$ has the following form:
 \begin{align}
\mat{G}_{ij}^{(H)}=\frac{c_{res}^2}{H^2m}(\vect{x}_i^{(H-1)})^\top \vect{x}_j^{(H-1)}\sum_{r=1}^{m}a_r^2\sigma'((\vect{w}_r^{(H)})^\top \vect{x}_i^{(H-1)})\sigma'((\vect{w}_r^{(H)})^\top \vect{x}_j^{(H-1)}).\label{eqn:H_resnet}
 \end{align}

Similar to Lemma \autoref{lem:lem:init_norm}, we can show with high probability the feature of each layer is approximately normalized.
\begin{lem}[Lemma on Initialization Norms]
	\label{lem:lem:init_norm_res}
	If $\sigma(\cdot)$ is $L-$Lipschitz and $m = \Omega\left(\frac{n}{\delta}\right)$, assuming $\norm{\mat{W}^{(h)}(0)}_2\le c_{w,0}\sqrt{m}$ for $h\in[2,H]$ and $c_{w,0} \approx 2$ for Gaussian initialization. We have with probability at least $1-\delta$ over random initialization, for every $h\in[H]$ and $i \in [n]$, 
	\[ \frac{1}{c_{x,0}}\le  \norm{\vect{x}_i^{(h)}(0)}_2 \le c_{x,0} \]  for some universal constant $c_{x,0} > 1$ (only depends on $\sigma$).
\end{lem}
The following lemma lower bounds $\mat{G}^{(H)}(0)$'s least eigenvalue.
This lemma is a direct consequence of results in Section~\ref{sec:general_formulation}.
\begin{lem}[Least Eigenvalue at the Initialization]\label{lem:resnet_least_eigen}
	If $m = \Omega\left(\frac{n^2\log(Hn/\delta)}{\lambda_0^2}\right)$, we have \begin{align*}
	\lambda_{\min}(\mat{G}^{(H)}(0)) \ge \frac{3}{4}\lambda_0.
	\end{align*}
\end{lem}

Next, we characterize how the perturbation on the weight matrices affects the input of each layer.
\begin{lem}\label{lem:pertubation_of_neuron_res}
	Suppose $\sigma(\cdot)$ is $L$-Lipschitz and for $h\in[H]$, $\norm{\mat{W}^{(h)}(0)}_2 \le c_{w,0}\sqrt{m}$, $\norm{\vect{x}^{(h)}(0)}_2 \le c_{x,0}$ and $\norm{\mat{W}^{(h)}(k)-\mat{W}^{(h)}(0)}_F \le \sqrt{m} R$ for some constant $c_{w,0},c_{x,0} > 0$ and $R\le c_{w,0}$ .
	Then we have \begin{align*}
	\norm{\vect{x}^{(h)}(k)-\vect{x}^{(h)}(0)}_2 \le \left(\sqrt{c_{\sigma}}L+\frac{c_{x,0}}{c_{w,0}}\right)e^{2c_{res}c_{w,0}L} R .
	\end{align*} 
\end{lem}

Next, we characterize how the perturbation on the weight matrices affect $\mat{G}^{(H)}$.
\begin{lem}\label{lem:close_to_init_small_perturbation_res_smooth}
	Suppose $\sigma(\cdot)$ is differentiable, $L-$Lipschitz and $\beta-$smooth. Using the same notations in Lemma~\ref{lem:close_to_init_small_perturbation_smooth},  if $\norm{\mat{W}^{(h)}(k)-\mat{W}^{(h)}(0)}_F, \norm{\vect{a}(k)-\vect{a}(0)}_2\le \sqrt{m}R$ where $R \le c \lambda_0H^2n^{-1}$ and $R\le c$ for some small constant $c$, we have  \begin{align*}
	\norm{\mat{G}^{(H)}(k) - \mat{G}^{(H)}(0)}_2 \le \frac{\lambda_0}{2}.
	\end{align*}
\end{lem}

We prove Theorem~\ref{thm:resnet_gd} by induction.
Our induction hypothesis is just the following convergence rate of empirical loss.

A directly corollary of this condition is the following bound of deviation from the initialization.
The proof only involves standard calculations so we defer it to appendix.
\begin{lem}\label{lem:dist_from_init_resnet}
	If Condition~\ref{cond:linear_converge} holds for $k'=1,\ldots,k$, we have for any $s \in [k+1]$
	\begin{align*}
	&\norm{\mat{W}^{(h)}(s)-\mat{W}^{(h)}(0)}_F, \norm{\vect{a}(s)-\vect{a}(0)}_2\le  R'\sqrt{m},\\
	&\norm{\mat{W}^{(h)}(s)-\mat{W}^{(h)}(s-1)}_F,\norm{\vect{a}(s)-\vect{a}(s-1)}_2\le \eta Q'(s-1),
	\end{align*}where $R'=\frac{16 c_{res}c_{x,0}a_{2,0}Le^{2c_{res}c_{w,0}L} \sqrt{n} \norm{\vect{y}-\vect{u}(0)}_2}{H\lambda_0\sqrt{m}} <c$ for some small constant $c$ and \\$ Q'(s)= 4c_{res}c_{x,0}a_{2,0}Le^{2c_{res}c_{w,0}L}\sqrt{n} \norm{\vect{y}-\vect{u}(s)}_2/H$.
\end{lem}

The next lemma bounds the $\vect{I}_2$ term.
\begin{lem}\label{lem:resnet_I2}
If Condition~\ref{cond:linear_converge} holds for $k'=1,\ldots,k$ and $\eta\le c\lambda_0H^2n^{-2}$ for some small constant $c$, we have \begin{align*}
\norm{\vect{I}_2(k)}_2 \le \frac{1}{8}\eta \lambda_0 \norm{\vect{y}-\vect{u}(k)}_2.
\end{align*}
\end{lem}

Next we bound the quadratic term.
\begin{lem}\label{lem:quadratic_resnet}
If Condition~\ref{cond:linear_converge} holds for $k'=1,\ldots,k$ and $\eta\le c\lambda_0H^2n^{-2}$ for some small constant $c$, we have
$\norm{\vect{u}(k+1)-\vect{u}(k)}_2^2\le \frac{1}{8}\eta \lambda_0 \norm{\vect{y}-\vect{u}(k)}_2^2$.
\end{lem}

Now using the same argument as in the proof for multilayer fully connected neural network, we finish our proof for ResNet.

%% file: technical_proofs_resnet.tex
\begin{proof}[Proof of Lemma~\ref{lem:lem:init_norm_res}]
	We will bound $\norm{\vect{x}_i^{(h)}(0)}_2$ layer by layer. For the first layer, we can calculate
\begin{align*}
\expect\left[\norm{\vect{x}_i^{(1)}(0)}_2^2\right] = &c_{\sigma} \expect\left[
\relu{\vect{w}_{r}^{(1)}(0)^\top \vect{x}_i}^2
\right] \\
= & c_{\sigma}\expect_{X\sim N(0,1)} \sigma(X)^2\\
=& 1.
\end{align*}
\begin{align*}
\variance\left[\norm{\vect{x}_i^{(1)}(0)}_2^2\right] = &\frac{c_{\sigma}^2}{m} \variance\left[\relu{\vect{w}_r^{(1)}(0)^\top \vect{x}_i(0)}^2\right] \\
\le & \frac{c_{\sigma}^2}{m} \expect_{X\sim N(0,1)} \sigma(X)^4 \\
\le & \frac{c_{\sigma}^2}{m} \expect\left[\left(\abs{\sigma(0)}+
L\abs{\vect{w}_r^{(1)}(0)^\top \vect{x}_i}\right)^4\right] \\
\le &\frac{C_2}{m} ,
\end{align*}
where $C_2\triangleq \sigma(0)^4+4\abs{\sigma(0)}^3L\sqrt{2/\pi}+6\sigma(0)^2L^2+8\abs{\sigma(0)}L^3\sqrt{2/\pi}+32L^4$. We have with probability at least $1-\frac{\delta}{n}$,
\begin{align*}
 \frac{1}{2}\le\norm{\vect{x}_i^{(1)}(0)}_2 \le 2.
\end{align*}
	 By definition we have for $2\le h\le H$,
	\begin{align*}
\norm{\vect{x}_i^{(h-1)}(0)}_2&-\norm{\frac{c_{res}}{H\sqrt{m}}\relu{\mat{W}^{(h)}(0)\vect{x}_i^{(h-1)}(0) }}_2 \le \norm{\vect{x}^{(h)}(0)}_2 \\ &\le \norm{\vect{x}_i^{(h-1)}(0)}_2+\norm{\frac{c_{res}}{H\sqrt{m}}\relu{\mat{W}^{(h)}(0)\vect{x}^{(h-1)}(0) }}_2,
	\end{align*}
	where
	\[\norm{\frac{c_{res}}{H\sqrt{m}}\relu{\mat{W}^{(h)}(0)\vect{x}_i^{(h-1)}(0) }}_2 \le \frac{c_{res}c_{w,0}L}{H}\norm{\vect{x}_i^{(h-1)}(0)}_2. \]
	Thus 
		\[\norm{\vect{x}_i^{(h-1)}(0)}_2\left(1-\frac{c_{res}c_{w,0}L}{H}\right)  \le \norm{\vect{x}^{(h)}(0)}_2 \le \norm{\vect{x}_i^{(h-1)}(0)}_2\left(1+\frac{c_{res}c_{w,0}L}{H}\right) ,\]
		which implies
		\[\frac{1}{2}e^{-c_{res}c_{w,0}L}  \le \norm{\vect{x}^{(h)}(0)}_2 \le 2e^{c_{res}c_{w,0}L}.\]
Choosing $c_{x,0}=2e^{c_{res}c_{w,0}L}$ and using union bounds over $[n]$, we prove the lemma.

\end{proof}

\begin{proof}[Proof of Lemma~\ref{lem:pertubation_of_neuron_res}]
	We prove this lemma by induction.
	Our induction hypothesis is \begin{align*}
	\norm{\vect{x}^{(h)}(k)-\vect{x}^{(h)}(0)}_2 \le g(h)  ,
	\end{align*}	where \begin{align*}
	g(h) = g(h-1)\left[1+\frac{2c_{res}c_{w,0}L}{H}\right] + \frac{L}{H}R c_{x,0}.
	\end{align*}
	For $h=1$, we have
	\begin{align*}
	\norm{\vect{x}^{(1)}(k)-\vect{x}^{(1)}(0)}_2&\le \sqrt{\frac{c_{\sigma}}{m}} \norm{\relu{\mat{W}^{(1)}(k) \vect{x}} -\relu{\mat{W}^{(1)}(0) \vect{x}}}_2\\
&	\le \sqrt{\frac{c_{\sigma}}{m}}\norm{\mat{W}^{(1)}(k)-\mat{W}^{(1)}(0)}_F \le \sqrt{c_{\sigma}}LR,
	\end{align*}
	which implies $g(1)=\sqrt{c_{\sigma}}LR$, for $2\le h\le H$, we have
	\begin{align*}
&	\norm{\vect{x}^{(h)}(k)-\vect{x}^{(h)}(0)}_2 \le \frac{c_{res}}{H\sqrt{m}} \norm{\relu{\mat{W}^{(h)}(k) \vect{x}^{(h-1)}(k)} -\relu{\mat{W}^{(h)}(0) \vect{x}^{(h-1)}(0)}}_2\\&+\norm{\vect{x}^{(h-1)}(k)-\vect{x}^{(h-1)}(0)}_2 \\
& 	\le \frac{c_{res}}{H\sqrt{m}} \norm{\relu{\mat{W}^{(h)}(k) \vect{x}^{(h-1)}(k)} -\relu{\mat{W}^{(h)}(k) \vect{x}^{(h-1)}(0)}}_2\\ 
	& + \frac{c_{res}}{H\sqrt{m}} \norm{\relu{\mat{W}^{(h)}(k) \vect{x}^{(h-1)}(0)} -\relu{\mat{W}^{(h)}(0) \vect{x}^{(h-1)}(0)}}_2\\ &+\norm{\vect{x}^{(h-1)}(k)-\vect{x}^{(h-1)}(0)}_2\\
&	\le  \frac{c_{res}L}{H\sqrt{m}}\left(
	\norm{\mat{W}^{(h)}(0)}_2 + \norm{\mat{W}^{(h)}(k)-\mat{W}^{(h)}(0)}_F
	\right) \cdot \norm{\vect{x}^{(h-1)}(k)-\vect{x}^{(h-1)}(0)}_2 \\
	& + \frac{c_{res}L}{H\sqrt{m}}\norm{\mat{W}^{(h)}(k)-\mat{W}^{(h)}(0)}_F \norm{\vect{x}^{h-1}(0)}_2 +\norm{\vect{x}^{(h-1)}(k)-\vect{x}^{(h-1)}(0)}_2\\
&	\le \left[1+\frac{c_{res}L}{H\sqrt{m}}\left(c_{w,0}\sqrt{m}+R\sqrt{m}\right)\right]g(h-1) + \frac{c_{res}L}{H\sqrt{m}} \sqrt{m} R c_{x,0}\\
&	\le \left(1+\frac{2c_{res}c_{w,0}L}{H}\right)g(h-1)+\frac{c_{res}}{H} Lc_{x,0}R . 
	\end{align*}
	Lastly, simple calculations show $g(h) \le \left(\sqrt{c_{\sigma}}L+\frac{c_{x,0}}{c_{w,0}}\right)e^{2c_{res}c_{w,0}L} R$. 
	
\end{proof}

\begin{proof}[Proof of Lemma~\ref{lem:close_to_init_small_perturbation_res_smooth}]
	Similar to the proof of Lemma~\ref{lem:close_to_init_small_perturbation_smooth}, we can obtain
	\begin{align*}
	&\abs{\mat{G}_{i,j}^{(H)}(k)-\mat{G}_{i,j}^{(H)}(0)} \le \frac{c_{res}^2}{H^2}\left(I_1^{i,j} + I_2^{i,j}+I_3^{i,j}\right) .
	\end{align*}
	For $I_1^{i,j}$, using Lemma~\ref{lem:pertubation_of_neuron_res}, we have \begin{align*}
	I_1^{i,j} = &L^2a_{2,0}^2\abs{\vect{x}_i^{(H-1)}(k)^\top \vect{x}_j^{(H-1)}(k) - \vect{x}_i^{(H-1)}(0)^\top \vect{x}_j^{(H-1)}(0)}  \\
	\le & L^2a_{2,0}^2\abs{
		(\vect{x}_i^{(H-1)}(k)-\vect{x}_i^{(H-1)}(0))^\top \vect{x}_j^{(H-1)}(k)} + L^2a_{2,0}^2\abs{
		\vect{x}_i^{(H-1)}(0)^\top(\vect{x}_i^{(H-1)}(k)-\vect{x}_i^{(H-1)}(0))}  \\
	\le & c_{x}L^2 a_{2,0}^2R \cdot (c_{x,0} + c_{x} R) + c_{x,0} c_x L^2a_{2,0}^2R \\
	\le &3 c_{x,0} c_xL^2a_{2,0}^2 R,
	\end{align*}
	where $c_x \triangleq \left(\sqrt{c_{\sigma}}L+\frac{c_{x,0}}{c_{w,0}}\right)e^{2c_{res}c_{w,0}L} $.
	To bound $I_{2}^{i,j}$, we have 
	\begin{align*}
	I_2^{i,j} =&c_{x,0}^2 \frac{1}{m} \abs{\sum_{r=1}^{m}
		a_r(0)^2\sigma'\left(z_{i,r}(k)\right)\sigma'\left(z_{j,r}(k)\right)
		- 
		a_r(0)^2\sigma'\left(z_{i,r}(0)\right)\sigma'\left(z_{j,r}(0)\right)
	}\\
	\le &c_{x,0}^2 \frac{1}{m}\sum_{r=1}^{m}		a_r(0)^2\abs{\left( \sigma'\left(z_{i,r}(k)\right)-\sigma'\left(z_{i,r}(0)\right) \right)\sigma'\left(z_{j,r}(k)\right)} +	a_r(0)^2\abs{\left( \sigma'\left(z_{j,r}(k)\right)-\sigma'\left(z_{j,r}(0)\right) \right)\sigma'\left(z_{i,r}(0)\right)}\\
	\le & \frac{\beta L c_{x,0}^2}{m} \left(
	\sum_{r=1}^{m} 	a_r(0)^2\abs{z_{i,r}(k)-z_{i,r}(0)}+	a_r(0)^2\abs{z_{j,r}(k)-z_{j,r}(0)}
	\right) \\
	\le & \frac{\beta La_{4,0}^2 c_{x,0}^2}{\sqrt{m}}\left(\sqrt{\sum_{r=1}^{m}\abs{z_{i,r}(k)-z_{i,r}(0)}^2}+\sqrt{\sum_{r=1}^{m}\abs{z_{j,r}(k)-z_{j,r}(0)}^2}\right) .
	\end{align*}
	Using the same proof for Lemma~\ref{lem:pertubation_of_neuron_res}, it is easy to see \begin{align*}
	\sum_{r=1}^{m}\abs{z_{i,r}(k)-z_{i,r}(0)}^2 \le \left(2c_xc_{w,0}+c_{x,0}\right)^2L^2mR^2 .
	\end{align*}
	Thus
	\begin{align*}
	I_2^{i,j}\le 2\beta c_{x,0}^2 \left(2c_xc_{w,0}+c_{x,0}\right)L^2 R .
	\end{align*}
	The bound of $I_3^{i,j}$ is similar to that in Lemma~\ref{lem:close_to_init_small_perturbation_smooth},
		\begin{align*}
	I_3^{i,j}
	&\le 12L^2c_{x,0}^2 a_{2,0}R.
	\end{align*}
	Therefore we can bound the perturbation\begin{align*}
	\norm{\mat{G}^{(H)}(k) - \mat{G}^{(H)}(0)}_F=&\sqrt{\sum_{(i,j)}^{{n,n}} \abs{\mat{G}_{i,j}^{(H)}(k)-\mat{G}_{i,j}^{(H)}(0)}^2} \\
	\le &\frac{c_{res}^2L^2nR}{H^2}\left[3 c_{x,0} c_xa_{2,0}^2+2\beta c_{x,0}^2 \left(2c_xc_{w,0}+c_{x,0}\right)a_{4,0}^2+12c_{x,0}^2 a_{2,0}\right]. \\
	\end{align*}
	Plugging in the bound on $R$, we have the desired result.
	
\end{proof}

\begin{proof}[Proof of Lemma~\ref{lem:dist_from_init_resnet}]
	We will prove this corollary by induction. The induction hypothesis is
		\begin{align*}
	\norm{\mat{W}^{(h)}(s)-\mat{W}^{(h)}(0)}_F &\le \sum_{s'=0}^{s-1} (1-\frac{\eta \lambda_0}{2})^{s'/2}\frac{1}{4}\eta \lambda_0 R'\sqrt{m}\le R'\sqrt{m}, s\in [k+1],\\
	\norm{\vect{a}(s)-\vect{a}(0)}_2 &\le \sum_{s'=0}^{s-1} (1-\frac{\eta \lambda_0}{2})^{s'/2}\frac{1}{4}\eta \lambda_0 R'\sqrt{m}\le R'\sqrt{m}, s\in [k+1].
	\end{align*}
	First it is easy to see it holds for $s'=0$. Now suppose it holds for $s'=0,\ldots,s$, we consider $s'=s+1$. Similar to Lemma~\ref{lem:dist_from_init}, we have
	\begin{align*} 
	&\norm{\mat{W}^{(h)}(s+1)-\mat{W}^{(h)}(s)}_F\\
	\le &\eta \frac{Lc_{res}}{H\sqrt{m}} \norm{\vect{a}}_2 \sum_{i=1}^{n}\abs{y_i-u_i(s)}\norm{\vect{x}^{(h-1)}_i(s)}_2 \prod_{k=h+1}^H \norm{\mat{I}+\frac{c_{res}\lambda^{3/2}}{H\sqrt{m}}\mat{J}_i^{(k)}(s)\mat{W}^{(k)}(s) }_2\\
	\le&  2\eta c_{res}c_{x,0}La_{2,0}e^{2c_{res}c_{w,0}L}\sqrt{n} \norm{\vect{y}-\vect{u}(s)}_2/H\\
	=& \eta Q'(s)\\
	\le& (1-\frac{\eta \lambda_0}{2})^{s/2}\frac{1}{4}\eta \lambda_0 R'\sqrt{m},\\
	\end{align*}
		Similarly, we have \begin{align*}
	\norm{\vect{a}(s+1)-\vect{a}(s)}_2 \le& 2\eta c_{x,0} \sum_{i=1}^{n}\abs{y_i-u(s)}\\
	\le& \eta Q'(s)\\
	\le& (1-\frac{\eta \lambda_0}{2})^{s/2}\frac{1}{4}\eta \lambda_0 R'\sqrt{m}.
	\end{align*}
	Thus
	\begin{align*}
	&\norm{\mat{W}^{(h)}(s+1)-\mat{W}^{(h)}(0)}_F\\
	\le& \norm{\mat{W}^{(h)}(s+1)-\mat{W}^{(h)}(s)}_F+\norm{\mat{W}^{(h)}(s)-\mat{W}^{(h)}(0)}_F\\
	\le&\sum_{s'=0}^{s} \eta (1-\frac{\eta \lambda_0}{2})^{s'/2}\frac{1}{4}\eta \lambda_0 R'\sqrt{m}.\\
	\end{align*}
		Similarly,
	\begin{align*}
	&\norm{\vect{a}(s+1)-\vect{a}(0)}_2\\
	\le&\sum_{s'=0}^{s} \eta (1-\frac{\eta \lambda_0}{2})^{s'/2}\frac{1}{4}\eta \lambda_0 R'\sqrt{m}.\\
	\end{align*}
\end{proof}

\begin{proof}[Proof of Lemma~\ref{lem:resnet_I2}]
Similar to Lemma~\ref{lem:i2_mlp}, we first bound the gradient norm.
\begin{align*}
&\norm{L'^{(h)}(\vect{w}(k))}_F \\
= &\big{\|}\frac{c_{res}}{H\sqrt{m}}
	\sum_{i=1}^{n}(y_i-u_i(k))\vect{x}_i^{(h-1)}(k)  \cdot
	\left[\vect{a}(k)^\top \prod_{l=h+1}^{H}\left(\mat{I}+\frac{c_{res}}{H\sqrt{m}}\mat{J}_i^{(l)}(k)\mat{W}^{(l)}(k) \right)\mat{J}_i^{(h)}(k) \right]
\big{\|}_F\\
\le &\frac{c_{res}L}{H\sqrt{m}} \norm{\vect{a}(k)}_2 \sum_{i=1}^{n}\abs{y_i-u_i(k)}\norm{\vect{x}^{(h-1)}(k)}_2 \prod_{k=h+1}^H \norm{\mat{I}+\frac{c_{res}}{H\sqrt{m}}\mat{J}_i^{(k)}(k)\mat{W}^{(k)}(k) }_2.
\end{align*}
We have bounded the RHS in the proof for Lemma~\ref{lem:dist_from_init_resnet}, thus
	\begin{align*}
	\norm{L'^{(h)}(\params(k))}_F \le \lambda_0 Q'(k).
	\end{align*}
	Let $\params(k,s)=\params(k)-s L'(\params(k))$, we have
	\begin{align*}
	&\norm{  u'^{(h)}_i\left(\params(k)\right) - u'^{(h)}_i\left(\params(k,s)\right)}_F=\\
	&\frac{c_{res}}{H\sqrt{m}}\norm{ \vect{x}_i^{(h-1)}(k)\vect{a}(k)^\top \prod_{l=h+1}^{H}\left(\mat{I}+\frac{c_{res}}{H\sqrt{m}}\mat{J}_i^{(l)}(k)\mat{W}^{(l)}(k) \right)\mat{J}_i^{(h)}(k) \right.\\&\left.
-\vect{x}_i^{(h-1)}(k,s)\vect{a}(k,s)^\top	\prod_{l=h+1}^{H}\left(\mat{I}+\frac{c_{res}}{H\sqrt{m}}\mat{J}_i^{(l)}(k,s)\mat{W}^{(l)}(k,s) \right)\mat{J}_i^{(h)}(k,s)}_F.
	\end{align*}
Through standard calculations, we have
	\begin{align*}
	\norm{	\mat{W}^{(l)}(k)-\mat{W}^{(l)}(k,s)}_F \le &\eta Q'(k),\\
		\norm{	\vect{a}(k)-\vect{a}(k,s)}_F \le &\eta Q'(k),\\
	\norm{	\vect{x}_i^{(h-1)}(k)-\vect{x}_i^{(h-1)}(k,s)}_F \le &\eta c_x \frac{Q'(k)}{\sqrt{m}},\\
	\norm{	\mat{J}^{(l)}(k)-\mat{J}^{(l)}(k,s)}_F \le &2\left(c_{x,0}+c_{w,0}c_x\right)\eta \beta   Q'(k),
	\end{align*}
	where $c_x\triangleq\left(\sqrt{c_{\sigma}}L+\frac{c_{x,0}}{c_{w,0}}\right)e^{3c_{res}c_{w,0}L}$.
	According to Lemma~\ref{lem:difference_norm}, we have
	\begin{align*}
	&\norm{  u'^{(h)}_i\left(\params(k)\right) - u'^{(h)}_i\left(\params(k,s)\right)}_F \\
	\le &\frac{4}{H}c_{res}c_{x,0}La_{2,0}e^{2Lc_{w,0}}\eta\frac{Q'(k)}{\sqrt{m}}\left(\frac{c_x}{c_{x,0}}+\frac{2}{L}\left(c_{x,0}+c_{w,0}c_x\right)\beta \sqrt{m}+4c_{w,0}\left(c_{x,0}+c_{w,0}c_x\right)\beta+L+1 \right)\\
	\le &\frac{32}{H}c_{res}c_{x,0}a_{2,0}e^{2Lc_{w,0}} \left(c_{x,0}+c_{w,0}c_x\right) \beta \eta Q'(k).
	\end{align*}
	Thus we have 
	\begin{align*}
	\abs{I^i_2}\le 32c_{res}c_{x,0}a_{2,0}e^{2Lc_{w,0}} \left(c_{x,0}+c_{w,0}c_x\right) \beta \eta^2  Q'(k)^2 \le \frac{1}{8}\eta \lambda_0 \norm{\vect{y}-\vect{u}(k)}_2,
	\end{align*}
	where we used the bound of $\eta$ and that $\norm{\vect{y}-\vect{u}(0)}_2=O(\sqrt{n})$,.
\end{proof}

\begin{proof}[Proof of Lemma~\ref{lem:quadratic_resnet}]
	\begin{align*}
\norm{\vect{u}(k+1)-\vect{u}(k)}_2^2 = & \sum_{i=1}^{n}\left(\vect{a}(k+1)^\top \vect{x}_i^{(H)}(k+1)-\vect{a}(k)^\top \vect{x}_i^{(H)}(k)\right)^2 \\
= & \sum_{i=1}^{n}\left(\left[\vect{a}(k+1)-\vect{a}(k)\right]^\top \vect{x}_i^{(H)}(k+1)+\vect{a}(k)^\top \left[\vect{x}_i^{(H)}(k+1)-\vect{x}_i^{(H)}(k)\right] \right)^2 \\
\le &2\norm{\vect{a}(k+1)-\vect{a}(k)}_2^2\sum_{i=1}^{n}\norm{\vect{x}_i^{(H)}(k+1)}_2^2+2\norm{\vect{a}(k)}_2^2\sum_{i=1}^{n}\norm{\vect{x}_i^{(H)}(k+1)-\vect{x}_i^{(H)}(k)}_2^2\\
\le &8n\eta^2c_{x,0}^2Q'(k)^2+4 n \left(\eta  a_{2,0}c_xQ'(k)\right)^2\\
\le &\frac{1}{8}\eta \lambda_0 \norm{\vect{y}-\vect{u}(k)}_2^2.
\end{align*}
\end{proof}

%% file: conv_resnet_proof.tex
For CNN, denote $\vect{x}_{i,l}=\phi\left(\vect{x}_{i,l}\right)_{:,l}$, $\mat{G}^{(H)}$ has the following form:
 \begin{align}
 \mat{G}_{ij}^{(H)}=\frac{c_{res}^2}{H^2m}\sum_{r=1}^m \left[\sum_{l=1}^{p}a_{l,r} \vect{x}_{i,l}^{(H-1)} \sigma'\left( \left(\vect{w}_r^{(H)}\right)^\top \vect{x}_{i,l}^{(H-1)}\right)\right]^\top \left[ \sum_{k=1}^{p}a_{k,r} \vect{x}_{j,k}^{(H-1)} \sigma'\left(\left(\vect{w}_r^{(H)}\right)^\top \vect{x}_{j,k}^{(H-1)}\right) \right]. \label{eqn:H_convresnet}
 \end{align}

We define a constant $c_{\sigma,c_0}=\left(\min_{c_0\le\alpha \le 1}\expect_{X\sim N(0,1)}\sigma(\alpha X)^2\right)^{-1}>0$, where $0< c_0\le1$.
In particular, it is easy to see for smooth ReLU,  $c_{\sigma,\frac{1}{\sqrt{p}}}= \poly(p)$.

Similar to Lemma \autoref{lem:lem:init_norm}, we can show with high probability the feature of each layer is approximately normalized.
\begin{lem}[Lemma on Initialization Norms]
	\label{lem:lem:init_norm_cnn}
	If $\sigma(\cdot)$ is $L-$Lipschitz and $m = \Omega\left(\frac{p^2n}{c_{\sigma,\frac{1}{\sqrt{p}}}^2\delta}\right)$, assuming $\norm{\mat{W}^{(h)}(0)}_2\le c_{w,0}\sqrt{m}$ for $h\in[H]$,
	we have with probability at least $1-\delta$ over random initialization, for every $h\in[H]$ and $i \in [n]$, 
	\[ \frac{1}{c_{x,0}}\le  \norm{\vect{x}_i^{(h)}(0)}_F \le c_{x,0} \]  for some constant $c_{x,0} =poly(p) > 1$.
\end{lem}

The following lemma lower bounds $\mat{G}^{(H)}(0)$'s least eigenvalue.
This lemma is a direct consequence of results in Section~\ref{sec:general_formulation}.
\begin{lem}[Least Eigenvalue at the Initialization]\label{lem:conv_resnet_least_eigen}
	If $m = \Omega\left(\frac{n^2p^2\log(Hn/\delta)}{\lambda_0^2}\right)$, we have \begin{align*}
	\lambda_{\min}(\mat{G}^{(H)}(0)) \ge \frac{3}{4}\lambda_0.
	\end{align*}
\end{lem}

Next, we prove the following lemma which characterizes how the perturbation from weight matrices propagates to the input of each layer.
 
 \begin{lem}\label{lem:pertubation_of_neuron_cnn}
 	Suppose $\sigma(\cdot)$ is $L-$Lipschitz and for $h\in[H]$, $\norm{\mat{W}^{(h)}(0)}_2 \le c_{w,0}\sqrt{m}$, $\norm{\vect{x}^{(h)}(0)}_F \le c_{x,0}$ and $\norm{\mat{W}^{(h)}(k)-\mat{W}^{(h)}(0)}_F \le \sqrt{m} R$ for some constant $c_{w,0},c_{x,0} > 1$ and $R\le c_{w,0}$ .
 	Then we have \begin{align*}
 	\norm{\vect{x}^{(h)}(k)-\vect{x}^{(h)}(0)}_F \le \left(\sqrt{c_{\sigma}}L\sqrt{q}+\frac{c_{x,0}}{c_{w,0}}\right)e^{2c_{w,0}L\sqrt{q}c_{res}} R.
 	\end{align*} 
 \end{lem}
 
 Next, we show with high probability over random initialization,  perturbation in weight matrices leads to small perturbation in the Gram matrix.
\begin{lem}\label{lem:close_to_init_small_perturbation_cnn_smooth}
	Suppose $\sigma(\cdot)$ is differentaible, $L-$Lipschitz and $\beta-$smooth. Using the same notations in Lemma~\ref{lem:close_to_init_small_perturbation_smooth},  if $\norm{\vect{a}_{:,i}}_2\le a_{2,0}\sqrt{m}$ and $\norm{\vect{a}_{:,i}}_4\le a_{4,0}m^{1/4}$ for any $i\in [ p]$, $\norm{\mat{W}^{(h)}(k)-\mat{W}^{(h)}(0)}_F,\norm{ \vect{a}(k)-\vect{a}(0)}_F \le \sqrt{m}R$ where $R \le c \lambda_0H^2\left(n \right)^{-1}poly(p)^{-1}$ for some small constant $c$, we have  \begin{align*}
	\norm{\mat{G}^{(H)}(k) - \mat{G}^{(H)}(0)}_2 \le \frac{\lambda_0}{2}.
	\end{align*}
\end{lem}

\begin{lem}\label{lem:dist_from_init_cnn}
	If Condition~\ref{cond:linear_converge} holds for $k'=1,\ldots,k$, we have for any $s \in [k+1]$
	\begin{align*}
	&\norm{\mat{W}^{(h)}(s)-\mat{W}^{(h)}(0)}_F, \norm{\vect{a}(s)-\vect{a}(0)}_F\le  R'\sqrt{m},\\
	&\norm{\mat{W}^{(h)}(s)-\mat{W}^{(h)}(s-1)}_F,\norm{\vect{a}(s)-\vect{a}(s-1)}_F\le \eta Q'(s-1),
	\end{align*}where $R'=\frac{16c_{res} c_{x,0}L\sqrt{pq}e^{2c_{res}c_{w,0}La_{2,0}\sqrt{q}} \sqrt{n} \norm{\vect{y}-\vect{u}(0)}_2}{H\lambda_0\sqrt{m}} <c$ for some small constant $c$ and \[Q'(s) = 4c_{res}c_{x,0}La_{2,0}\sqrt{pq}e^{2c_{res}c_{w,0}L\sqrt{q}}\sqrt{n} \norm{\vect{y}-\vect{u}(s)}_2/H.\]
\end{lem}

The follow lemma bounds the norm of $\vect{I}_2$.
\begin{lem}\label{lem:cnn_I2}
If Condition~\ref{cond:linear_converge} holds for $k'=1,\ldots,k$ and $\eta\le c\lambda_0H^2n^{-2}poly(1/p)$ for some small constant $c$, we have \begin{align*}
\norm{\vect{I}_2(k)}_2 \le \frac{1}{8}\eta \lambda_0 \norm{\vect{y}-\vect{u}(k)}_2.
\end{align*}
\end{lem}

Next we also bound the quadratic term.
\begin{lem}\label{lem:cnn_quadratic}
If Condition~\ref{cond:linear_converge} holds for $k'=1,\ldots,k$ and $\eta\le c\lambda_0H^2n^{-2}poly(1/p)$ for some small constant $c$, we have
$\norm{\vect{u}(k+1)-\vect{u}(k)}_2^2\le \frac{1}{8}\eta \lambda_0 \norm{\vect{y}-\vect{u}(k)}_2^2$.
\end{lem}

Now using the same argument as in the proof for multilayer fully connected neural network, we finish our proof for CNN.

%% file: technical_proofs_conv_resnet.tex
\begin{proof}[Proof of Lemma~\ref{lem:lem:init_norm_cnn}]
	We will bound $\norm{\vect{x}_i^{(h)}(0)}_F$ layer by layer. For the first layer, we can calculate
	\begin{align*}
	\expect\left[\norm{\vect{x}_i^{(1)}(0)}_F^2\right] = &c_{\sigma}\sum_{l=1}^{p_1} \expect\left[
	\relu{\vect{w}_{r}^{(1)}(0)^\top \vect{x}_{i,l}}^2
	\right] \\
	\ge&\frac{c_{\sigma}}{c_{\sigma,\frac{1}{\sqrt{p}}}},
	\end{align*}
	where the inequality we use the definition of $c_{\sigma,\frac{1}{\sqrt{p}}}$ and the fact that there must exist $l'\in [p]$ such that $\norm{\vect{x}_{i,l'}}_2^2 \ge \frac{1}{p_1}\ge\frac{1}{p}$. For the variance,
	\begin{align*}
	\variance\left[\norm{\vect{x}_i^{(1)}(0)}_F^2\right] = &\frac{c_{\sigma}^2}{m} \variance\left[\sum_{l=1}^{p_1}\relu{\vect{w}_r^{(1)}(0)^\top \vect{x}_{i,l}}^2\right] \\
	\le & \frac{c_{\sigma}^2}{m} \expect\left[\left(\sum_{l=1}^{p_1} \left(\abs{\sigma(0)}+
	L\abs{\vect{w}_r^{(1)}(0)^\top \vect{x}_{i,l}}\right)^2\right)^2\right] \\
	\le &\frac{p^2C_2}{m} ,
	\end{align*}
	where $C_2\triangleq \sigma(0)^4+4\abs{\sigma(0)}^3L\sqrt{2/\pi}+6\sigma(0)^2L^2+8\abs{\sigma(0)}L^3\sqrt{2/\pi}+32L^4$. We have with probability at least $1-\frac{\delta}{n}$,
	\begin{align*}
\norm{\vect{x}_i^{(1)}(0)}_F^2 \ge \frac{c_{\sigma}}{2c_{\sigma,\frac{1}{\sqrt{p}}}}.
	\end{align*}
	It is easy to get its upper bound
	\begin{align*}
	 \norm{\vect{x}_i^{(1)}(0)}_F^2  = \frac{c_{\sigma}}{m}\norm{\sigma\left(\mat{W}^{(1)}\phi(\vect{x}_i)\right)}_F^2\le qL^2c_{\sigma}c_{w,0}^2.
	\end{align*}
	By defination we have for $2\le h\le H$
	\begin{align*}
	\norm{\vect{x}_i^{(h-1)}(0)}_F-\norm{\frac{c_{res}}{H\sqrt{m}}\relu{\mat{W}^{(h)}(0)\phi\left(\vect{x}_i^{(h-1)}(0)\right) }}_F \le \norm{\vect{x}_i^{(h)}(0)}_F \\\le \norm{\vect{x}_i^{(h-1)}(0)}_F+\norm{\frac{c_{res}}{H\sqrt{m}}\relu{\mat{W}^{(h)}(0)\phi\left(\vect{x}_i^{(h-1)}(0)\right)}}_F,
	\end{align*}
	where
	\[\norm{\frac{c_{res}}{H\sqrt{m}}\relu{\mat{W}^{(h)}(0)\phi\left(\vect{x}_i^{(h-1)}(0)\right) }}_F \le \frac{\sqrt{q}c_{res}c_{w,0}L}{H}\norm{\vect{x}_i^{(h-1)}(0)}_F .\]
	Thus 
	\[\norm{\vect{x}_i^{(h-1)}(0)}_F\left(1-\frac{\sqrt{q}c_{res}c_{w,0}L}{H}\right)  \le \norm{\vect{x}^{(h)}(0)}_F \le \norm{\vect{x}_i^{(h-1)}(0)}_F\left(1+\frac{\sqrt{q}c_{res}c_{w,0}L}{H}\right), \]
	which implies
	\[\sqrt{\frac{c_{\sigma}}{2c_{\sigma,\frac{1}{\sqrt{p}}}} }e^{-\sqrt{q}c_{res}c_{w,0}L}  \le \norm{\vect{x}^{(h)}(0)}_F \le \sqrt{qL^2c_{\sigma}c_{w,0}^2}e^{\sqrt{q}c_{res}c_{w,0}L}.\]
	Choosing $c_{x,0}=\max\{\sqrt{qL^2c_{\sigma}c_{w,0}^2},\sqrt{\frac{2c_{\sigma,\frac{1}{\sqrt{p}}}}{c_{\sigma}} }\}e^{\sqrt{q}c_{res}c_{w,0}L}$ and using union bounds over $[n]$, we prove the lemma.
	
\end{proof}

\begin{proof}[Proof of Lemma~\ref{lem:pertubation_of_neuron_cnn}]
	We prove this lemma by induction.
	Our induction hypothesis is \begin{align*}
	\norm{\vect{x}^{(h)}(k)-\vect{x}^{(h)}(0)}_F \le g(h)  ,
	\end{align*}	where \begin{align*}
	g(h) = g(h-1)\left[1+\frac{2c_{res}c_{w,0}L\sqrt{q}}{H}\right] + \frac{c_{res}L\sqrt{q}}{H}R c_{x,0}.
	\end{align*}
	For $h=1$, we have
	\begin{align*}
	\norm{\vect{x}^{(1)}(k)-\vect{x}^{(1)}(0)}_F\le &\sqrt{\frac{c_{\sigma}}{m}} \norm{\relu{\mat{W}^{(1)}(k) \phi_1(\vect{x})} -\relu{\mat{W}^{(1)}(0)\phi_1(\vect{x}) }}_F\\
	\le& \sqrt{\frac{c_{\sigma}}{m}}L\sqrt{q}\norm{\mat{W}^{(1)}(k)-\mat{W}^{(1)}(0)}_F \le \sqrt{c_{\sigma}}L\sqrt{q}R,
	\end{align*}
	which implies $g(1)=\sqrt{c_{\sigma}}L\sqrt{q}R$, for $2\le h\le H$, we have
	\begin{align*}
	&\norm{\vect{x}^{(h)}(k)-\vect{x}^{(h)}(0)}_F\\ \le &\frac{c_{res}}{H\sqrt{m}} \norm{\relu{\mat{W}^{(h)}(k)\phi_h\left(\vect{x}^{(h-1)}(k)\right) } -\relu{\mat{W}^{(h)}(0)\phi_h\left(\vect{x}^{(h-1)}(0)\right)}}_F+\norm{\vect{x}^{(h-1)}(k)-\vect{x}^{(h-1)}(0)}_F \\
	\le & \frac{c_{res}}{H\sqrt{m}} \norm{\relu{\mat{W}^{(h)}(k)\phi_h\left(\vect{x}^{(h-1)}(k)\right)} -\relu{\mat{W}^{(h)}(k)\phi_h\left(\vect{x}^{(h-1)}(0)\right)}}_F\\ 
	& + \frac{c_{res}}{H\sqrt{m}} \norm{\relu{\mat{W}^{(h)}(k)\phi_h\left(\vect{x}^{(h-1)}(0)\right)} -\relu{\mat{W}^{(h)}(0)\phi_h\left(\vect{x}^{(h-1)}(0)\right)}}_F +\norm{\vect{x}^{(h-1)}(k)-\vect{x}^{(h-1)}(0)}_F\\
	\le & \frac{L\sqrt{q}c_{res}}{H\sqrt{m}}\left(
	\norm{\mat{W}^{(h)}(0)}_2 + \norm{\mat{W}^{(h)}(k)-\mat{W}^{(h)}(0)}_F
	\right) \cdot \norm{\vect{x}^{(h-1)}(k)-\vect{x}^{(h-1)}(0)}_F \\
	& + \frac{L\sqrt{q}c_{res}}{H\sqrt{m}}\norm{\mat{W}^{(h)}(k)-\mat{W}^{(h)}(0)}_F \norm{\vect{x}^{h-1}(0)}_F +\norm{\vect{x}^{(h-1)}(k)-\vect{x}^{(h-1)}(0)}_F\\
	\le &\left[1+\frac{L\sqrt{q}c_{res}}{H\sqrt{m}}\left(c_{w,0}\sqrt{m}+R\sqrt{m}\right)\right]g(h-1) + \frac{L\sqrt{q}c_{res}}{H\sqrt{m}} \sqrt{m} R c_{x,0}\\
	\le &\left(1+\frac{2c_{w,0}L\sqrt{q}c_{res}}{H}\right)g(h-1)+\frac{1}{H} L\sqrt{q}c_{res}c_{x,0}R . 
	\end{align*}
	Lastly, simple calculations show $g(h) \le\left(\sqrt{c_{\sigma}}L\sqrt{q}+\frac{c_{x,0}}{c_{w,0}}\right)e^{2c_{w,0}L\sqrt{q}c_{res}} R$. 
	
\end{proof}

\begin{proof}[Proof of Lemma~\ref{lem:close_to_init_small_perturbation_cnn_smooth}]
Similar to Lemma~\ref{lem:close_to_init_small_perturbation_res_smooth}, define $z_{i,l,r}=\left(\vect{w}_r^{(H)}\right)^\top \vect{x}_{i,l}^{(H-1)}$, we have
	\begin{align*}
	&\abs{\mat{G}_{i,j}^{(H)}(k)-\mat{G}_{i,j}^{(H)}(0)} \\
	= &\frac{c_{res}^2}{H^2}\big{|}\sum_{l=1}^{p}\sum_{k=1}^{p}\vect{x}_{i,l}^{(H-1)}(k)^\top \vect{x}_{j,k}^{(H-1)}(k)
	\frac{1}{m}\sum_{r=1}^{m}a_{r,l}(k)a_{r,k}(k)\sigma'\left(z_{i,l,r}(k)\right)\sigma'\left(z_{j,k,r}(k)\right) 
	\\
	&-\sum_{l=1}^{p}\sum_{k=1}^{p}\vect{x}_{i,l}^{(H-1)}(0)^\top \vect{x}_{j,k}^{(H-1)}(0)
	\frac{1}{m}\sum_{r=1}^{m}a_{r,l}(0)a_{r,k}(0)\sigma'\left(z_{i,l,r}(0)\right)\sigma'\left(z_{j,k,r}(0)\right) 
	\big{|} \\
	\le & \frac{ c_{res}^2L^2a_{2,0}^2}{H^2}\abs{\sum_{l=1}^{p}\sum_{k=1}^{p}\vect{x}_{i,l}^{(H-1)}(k)^\top \vect{x}_{j,k}^{(H-1)}(k)- \vect{x}_{i,l}^{(H-1)}(0)^\top \vect{x}_{j,k}^{(H-1)}(0)} \\
	& + \frac{c_{res}^2}{H^2}\sum_{l=1}^{p}\sum_{k=1}^{p}\abs{\vect{x}_{i,l}^{(H-1)}(0)^\top \vect{x}_{j,k}^{(H-1)}(0) }\frac{1}{m} \sum_{r=1}^{m}\abs{a_{r,l}(0)a_{r,k}(0)}\abs{
		\sigma'\left(z_{i,l,r}(k)\right)\sigma'\left(z_{j,k,r}(k)\right) 
		-	\sigma'\left(z_{i,l,r}(0)\right)\sigma'\left(z_{j,k,r}(0)\right) 
	} \\
	&+\frac{c_{res}^2}{H^2}L^2\sum_{l=1}^{p}\sum_{k=1}^{p}\abs{\vect{x}_{i,l}^{(H-1)}(k)^\top \vect{x}_{j,k}^{(H-1)}(k) }\frac{1}{m} \sum_{r=1}^{m}\abs{a_{r,l}(k)a_{r,k}(k)-a_{r,l}(0)a_{r,k}(0)}
	 \\
	\triangleq& \frac{c_{res}^2}{H^2}\left(I_1^{i,j} + I_2^{i,j}+I_3^{i,j}\right).
	\end{align*}
	For $I_1^{i,j}$, using Lemma~\ref{lem:pertubation_of_neuron_cnn}, we have \begin{align*}
	I_1^{i,j} = &L^2a_{2,0}^2 \abs{\sum_{l=1}^{p}\sum_{k=1}^{p}\vect{x}_{i,l}^{(H-1)}(k)^\top \vect{x}_{j,k}^{(H-1)}(k)- \vect{x}_{i,l}^{(H-1)}(0)^\top \vect{x}_{j,k}^{(H-1)}(0)} \\
	\le & L^2a_{2,0}^2\sum_{l=1}^{p}\sum_{k=1}^{p}\abs{
		(\vect{x}_{i,l}^{(H-1)}(k)-\vect{x}_{i,l}^{(H-1)}(0))^\top \vect{x}_{j,k}^{(H-1)}(k)} + L^2a_{2,0}^2\sum_{l=1}^{p}\sum_{k=1}^{p}\abs{
		\vect{x}_{i,l}^{(H-1)}(0)^\top(\vect{x}_{j,k}^{(H-1)}(k)-\vect{x}_{j,k}^{(H-1)}(0))}  \\
	\le& L^2a_{2,0}^2\sqrt{\sum_{l=1}^{p}\sum_{k=1}^{p}\norm{\vect{x}_{i,l}^{(H-1)}(k)-\vect{x}_{i,l}^{(H-1)}(0)}_2^2}\sqrt{\sum_{l=1}^{p}\sum_{k=1}^{p}\norm{\vect{x}_{j,k}^{(H-1)}(k)}_2^2} \\
	&+L^2a_{2,0}^2\sqrt{\sum_{l=1}^{p}\sum_{k=1}^{p}\norm{\vect{x}_{i,l}^{(H-1)}(0)}_2^2}\sqrt{\sum_{l=1}^{p}\sum_{k=1}^{p}\norm{\vect{x}_{j,k}^{(H-1)}(k)-\vect{x}_{j,k}^{(H-1)}(0)}_2^2}\\
	\le & L^2a_{2,0}^2p \norm{\vect{x}_{i}^{(H-1)}(k)-\vect{x}_{i}^{(H-1)}(0)}_F\norm{\vect{x}_{j}^{(H-1)}(k)}_F + L^2a_{2,0}^2p \norm{\vect{x}_{i}^{(H-1)}(0)}_F\norm{\vect{x}_{j}^{(H-1)}(k)-\vect{x}_{j}^{(H-1)}(0)}_F \\
	\le &3 c_{x,0} c_xL^2 a_{2,0}^2p R,
	\end{align*}
	where $c_x \triangleq \left(\sqrt{c_{\sigma}}L\sqrt{q}+\frac{c_{x,0}}{c_{w,0}}\right)e^{2c_{res}c_{w,0}L\sqrt{q}} $.
	To bound $I_{2}^{i,j}$, we have \begin{align*}
	I_2^{i,j}= &
	\sum_{l=1}^{p}\sum_{k=1}^{p}\abs{\vect{x}_{i,l}^{(H-1)}(0)^\top \vect{x}_{j,k}^{(H-1)}(0) }\frac{1}{m}  \sum_{r=1}^{m}\abs{a_{r,l}(0)a_{r,k}(0)}\abs{
		\sigma'\left(z_{i,l,r}(k)\right)\sigma'\left(z_{j,k,r}(k)\right) 
		-	\sigma'\left(z_{i,l,r}(0)\right)\sigma'\left(z_{j,k,r}(0)\right) 
	}\\
	\le&\sum_{l=1}^{p}\sum_{k=1}^{p}\abs{\vect{x}_{i,l}^{(H-1)}(0)^\top \vect{x}_{j,k}^{(H-1)}(0) }\frac{\beta L}{m} \left(\sum_{r=1}^{m}\abs{a_{r,l}(0)a_{r,k}(0)}\left(\abs{z_{i,l,r}(k)-z_{i,l,r}(0)}+\abs{z_{j,k,r}(k)-z_{j,k,r}(0)}\right)\right)  \\
	\le&\frac{\beta L}{m} \sqrt{\sum_{l=1}^{p}\sum_{k=1}^{p}\norm{\vect{x}_{i,l}^{(H-1)}(0)}_2^2\norm{ \vect{x}_{j,k}^{(H-1)}(0)}_2^2 }\\ &\left(\sqrt{\sum_{l=1}^{p}\sum_{k=1}^{p}\left(\sum_{r=1}^{m}\abs{a_{r,l}(0)a_{r,k}(0)}\abs{z_{i,l,r}(k)-z_{i,l,r}(0)}\right)^2}+\sqrt{\sum_{l=1}^{p}\sum_{k=1}^{p}\left(\sum_{r=1}^{m}\abs{a_{r,l}(0)a_{r,k}(0)}\abs{z_{j,k,r}(k)-z_{j,k,r}(0)}\right)^2}\right)\\
	\le & \frac{\beta Lc_{x,0}^2a_{4,0}^2}{m} \left(
	\sqrt{m\sum_{l=1}^{p}\sum_{k=1}^{p}\sum_{r=1}^{m}\abs{z_{i,l,r}(k)-z_{i,l,r}(0)}^2}+\sqrt{m\sum_{l=1}^{p}\sum_{k=1}^{p}\sum_{r=1}^{m}\abs{z_{j,k,r}(k)-z_{j,k,r}(0)}^2}
	\right) \\
	\le & \frac{\beta La_{4,0}^2\sqrt{p} c_{x,0}^2}{\sqrt{m}}\left(\norm{\vect{z}_i}_F+\norm{\vect{z}_j}_F\right) .
	\end{align*}

	Using the same proof for Lemma~\ref{lem:pertubation_of_neuron_cnn}, it is easy to see \begin{align*}
	\norm{\vect{z}_i}_F\le \left(2c_xc_{w,0}\sqrt{q}+c_{x,0}\right) R \sqrt{m}.
	\end{align*}
	Thus
	\begin{align*}
	I_2^{i,j}\le 2\beta La_{4,0}^2\sqrt{p} c_{x,0}^2\left(2c_xc_{w,0}\sqrt{q}+c_{x,0}\right) R .
	\end{align*}
		Similarly for $I_3^{i,j}$, we have
	\begin{align*}
	I_3^{i,j}=&\frac{c_{res}^2}{H^2}L^2\sum_{l=1}^{p}\sum_{k=1}^{p}\abs{\vect{x}_{i,l}^{(H-1)}(k)^\top \vect{x}_{j,k}^{(H-1)}(k) }\frac{1}{m} \sum_{r=1}^{m}\abs{a_{r,l}(k)a_{r,k}(k)-a_{r,l}(0)a_{r,k}(0)}\\
	\le &\frac{c_{res}^2}{H^2}L^2\sum_{l=1}^{p}\sum_{k=1}^{p}\abs{\vect{x}_{i,l}^{(H-1)}(k)^\top \vect{x}_{j,k}^{(H-1)}(k) }\frac{1}{m} \sum_{r=1}^{m} \left(\abs{a_{r,l}(k)-a_{r,l}(0)}\abs{a_{r,k}(k)} +\abs{a_{r,k}(k)-a_{r,k}(0)}\abs{a_{r,l}(0)} \right)\\
	\le &\frac{c_{res}^2}{H^2}L^2\sum_{l=1}^{p}\sum_{k=1}^{p}\abs{\vect{x}_{i,l}^{(H-1)}(k)^\top \vect{x}_{j,k}^{(H-1)}(k) }\frac{1}{m}  \left(\norm{\vect{a}_{:,l}(k)-\vect{a}_{:,l}(0)}_2\norm{\vect{a}_{:,k}(k)}_2 + \norm{\vect{a}_{:,k}(k)-\vect{a}_{:,k}(0)}_2\norm{\vect{a}_{:,l}(0)}_2 \right)\\
	\le &\frac{c_{res}^2}{H^2{m}}L^2\sqrt{\sum_{l=1}^{p}\sum_{k=1}^{p}\norm{\vect{x}_{i,l}^{(H-1)}(k)}_2^2\norm{ \vect{x}_{j,k}^{(H-1)}(k)}_2^2 }\\ &\left(\sqrt{\sum_{l=1}^{p}\sum_{k=1}^{p}\norm{\vect{a}_{:,l}(k)-\vect{a}_{:,l}(0)}_2^2\norm{\vect{a}_{:,k}(k)}_2^2 } +\sqrt{\sum_{l=1}^{p}\sum_{k=1}^{p}\norm{\vect{a}_{:,k}(k)-\vect{a}_{:,k}(0)}_2^2\norm{\vect{a}_{:,l}(0)}_2^2}\right)\\
	\le&\frac{c_{res}^2}{H^2{m}}L^2\norm{\vect{x}_{i}^{(H-1)}(k)}_F\norm{ \vect{x}_{j}^{(H-1)}(k)}_F\left(\norm{\vect{a}(k)-\vect{a}(0)}_F\norm{\vect{a}(k)}_F+(\norm{\vect{a}(k)-\vect{a}(0)}_F\norm{\vect{a}(0)}_F\right)\\
	\le& \frac{12a_{2,0}c_{res}^2c_{x,0}^2L^2\sqrt{p}R}{H^2}.
	\end{align*}
	Therefore we can bound the perturbation\begin{align*}
	\norm{\mat{G}^{(H)}(k) - \mat{G}^{(H)}(0)}_2\le&\norm{\mat{G}^{(H)}(k) - \mat{G}^{(H)}(0)}_F\\
	=&\sqrt{\sum_{(i,j)}^{{n,n}} \abs{\mat{G}_{i,j}^{(H)}(k)-\mat{G}_{i,j}^{(H)}(0)}^2} \\
	\le &\frac{c_{res}^2}{H^2}\left[3 c_{x,0} c_xLa_{2,0}^2p+2\beta c_{x,0}^2a_{4,0}^2\sqrt{p} \left(2c_xc_{w,0}\sqrt{q}+c_{x,0}\right)+12c_{x,0}^2La_{2,0}\sqrt{p}\right]LnR  .
	\end{align*}
	Plugging in the bound on $R$, we have the desired result.	
\end{proof}

\begin{proof}[Proof of Lemma~\ref{lem:dist_from_init_cnn}]
	We will prove this corollary by induction. The induction hypothesis is
		\begin{align*}
	\norm{\mat{W}^{(h)}(s)-\mat{W}^{(h)}(0)}_F &\le \sum_{s'=0}^{s-1} (1-\frac{\eta \lambda_0}{2})^{s'/2}\frac{1}{4}\eta \lambda_0 R'\sqrt{m}\le R'\sqrt{m}, s\in [k+1],\\
	\norm{\vect{a}(s)-\vect{a}(0)}_F &\le \sum_{s'=0}^{s-1} (1-\frac{\eta \lambda_0}{2})^{s'/2}\frac{1}{4}\eta \lambda_0 R'\sqrt{m}\le R'\sqrt{m}, s\in [k+1].
	\end{align*}
	First it is easy to see it holds for $s'=0$. Now suppose it holds for $s'=0,\ldots,s$, we consider $s'=s+1$. Similar to Lemma~\ref{lem:dist_from_init}, we have
	\begin{align*} 
	&\norm{\mat{W}^{(h)}(s+1)-\mat{W}^{(h)}(s)}_F\\
	\le &\eta \frac{c_{res}L}{H\sqrt{m}} \norm{\vect{a}}_F \sum_{i=1}^{n}\abs{y_i-u(s)}\norm{\phi_h\left(\vect{x}^{(h-1)}(s)\right)}_F \prod_{k=h+1}^H \norm{\mat{I}+\frac{c_{res}}{H\sqrt{m}}\mat{W}^{(k)}(s)\phi_k }_{op}\\
	\le&  2\eta c_{res} c_{x,0}La_{2,0}\sqrt{pq}e^{2c_{res}c_{w,0}L\sqrt{q}}\sqrt{n} \norm{\vect{y}-\vect{u}(s)}_2/H\\
	=& \eta Q'(s)\\
	\le& (1-\frac{\eta \lambda_0}{2})^{s/2}\frac{1}{4}\eta \lambda_0 R'\sqrt{m},
	\end{align*}
	where $\norm{\cdot}_{op}$ denotes the operator norm. 	Similarly, we have \begin{align*}
	\norm{\vect{a}(s+1)-\vect{a}(s)}_2 \le& 2\eta c_{x,0} \sum_{i=1}^{n}\abs{y_i-u(s)}\\
	\le& \eta Q'(s)\\
	\le& (1-\frac{\eta \lambda_0}{2})^{s/2}\frac{1}{4}\eta \lambda_0 R'\sqrt{m}.
	\end{align*}
	Thus
	\begin{align*}
	&\norm{\mat{W}^{(h)}(s+1)-\mat{W}^{(h)}(0)}_F\\
	\le& \norm{\mat{W}^{(h)}(s+1)-\mat{W}^{(h)}(s)}_F+\norm{\mat{W}^{(h)}(s)-\mat{W}^{(h)}(0)}_F\\
	\le&\sum_{s'=0}^{s} \eta (1-\frac{\eta \lambda_0}{2})^{s'/2}\frac{1}{4}\eta \lambda_0 R'\sqrt{m}.
	\end{align*}
		Similarly,
	\begin{align*}
	&\norm{\vect{a}(s+1)-\vect{a}(0)}_2\\
	\le&\sum_{s'=0}^{s} \eta (1-\frac{\eta \lambda_0}{2})^{s'/2}\frac{1}{4}\eta \lambda_0 R'\sqrt{m}.\\
	\end{align*}
\end{proof}

\begin{proof}[Proof of Lemma~\ref{lem:cnn_I2}]
\begin{align*}
\abs{I_2^i} \le &\eta \max_{0\le s\le \eta} \sum_{h=1}^{H}  \norm{L'^{(h)}(\params(k))}_F \norm{  u'^{(h)}_i\left(\params(k)\right) -u'^{(h)}_i\left(\params(k)-s L'^{(h)}(\params(k))\right) }_F .
\end{align*}
For the gradient norm, we have 
\begin{align*}
&\norm{L'^{(h)}(\params(k))}_F\\
\le &\frac{Lc_{res}}{H\sqrt{m}} \norm{\vect{a}(k)}_F \sum_{i=1}^{n}\abs{y_i-u_i(k)}\norm{\phi_h\left(\vect{x}_i^{(h-1)}(k)\right)}_F \prod_{k=h+1}^H \norm{\mat{I}+\frac{c_{res}}{H\sqrt{m}}\mat{J}_i^{(k)}(k)\mat{W}^{(k)}(k)\phi_k }_{op},
\end{align*}
which we have bounded in Lemma~\ref{lem:dist_from_init_cnn}, thus
\begin{align*}
\norm{L'^{(h)}(\params(k))}_F \le  Q'(k).
\end{align*}
Let $\params(k,s)=\params(k)-s L'(\params(k))$
,similar to the proof of Lemma~\ref{lem:i2_mlp}, we have
\begin{align*}
&\norm{  u'^{(h)}_i\left(\params(k)\right) - u'^{(h)}_i\left(\params(k,s)\right)}_F \\
\le &\frac{2}{H}c_{res}c_{x,0}La_{2,0}\sqrt{q}e^{2c_{res}Lc_{w,0}\sqrt{q}}\eta\frac{Q'(k)}{\sqrt{m}}\left(\frac{c_x}{c_{x,0}}+\frac{2}{L}\left(c_{x,0}+c_{w,0}c_x\right)\beta \sqrt{m}+4\sqrt{q}c_{w,0}\left(c_{x,0}+c_{w,0}c_x\right)\beta\sqrt{m}+(L+1)\sqrt{q} \right)\\
\le &\frac{24}{H}c_{res}c_{x,0}La_{2,0}\sqrt{q}c_{w,0}e^{2c_{res}Lc_{w,0}\sqrt{q}} \left(c_{x,0}+c_{w,0}c_x\right) \beta \eta Q'(k).
\end{align*}
Thus
\begin{align*}
\abs{I^i_2}\le 24c_{res}c_{x,0}La_{2,0}\sqrt{q}c_{w,0}e^{2c_{res}Lc_{w,0}} \left(c_{x,0}+c_{w,0}c_x\right) \beta \eta^2 \lambda_0\sqrt{m} Q'(k) R' \le \frac{1}{8}\eta \lambda_0 \norm{\vect{y}-\vect{u}(k)}_2.
\end{align*}
where we used the bound of $\eta$ and that $\norm{\vect{y}-\vect{u}(0)}_2=O(\sqrt{n})$.
\end{proof}

\begin{proof}[Proof of Lemma~\ref{lem:cnn_quadratic}]
\begin{align*}
\norm{\vect{u}(k+1)-\vect{u}(k)}_2^2 = & \sum_{i=1}^{n}\left(\langle \vect{a}(k+1), \vect{x}_i^{(H)}(k+1)\rangle-\langle \vect{a}(k), \vect{x}_i^{(H)}(k+1)\rangle\right)^2 \\
\le& \sum_{i=1}^{n}\left(\langle \vect{a}(k+1)-\vect{a}(k), \vect{x}_i^{(H)}(k+1) \rangle+\langle \vect{a}(k), \vect{x}_i^{(H)}(k+1)-\vect{x}_i^{(H)}(k)\rangle \right)^2\\
\le &2\norm{\vect{a}(k+1)-\vect{a}(k)}_F^2\sum_{i=1}^{n}\norm{\vect{x}_i^{(H)}(k+1)}_F^2+2\norm{\vect{a}(k)}_F^2\sum_{i=1}^{n}\norm{\vect{x}_i^{(H)}(k+1)-\vect{x}_i^{(H)}(k)}_F^2\\
\le & 8n\eta^2c_{x,0}^2Q'(k)^2+4 np \left(\eta a_{2,0} c_xQ'(k)\right)^2\\
\le &\frac{1}{8}\eta \lambda_0 \norm{\vect{y}-\vect{u}(k)}_2^2.
\end{align*}
\end{proof}

%% file: general_formulation.tex
\subsection{A General Framework for Analyzing Random Initialization in First $(H-1)$ Layers}
In this section we provide a self-contained framework to analyze the Gram matrix at the initialization phase.
There are two main objectives.
First, we provide the expression of the Gram matrix as $m \rightarrow \infty$, i.e., the population Gram matrix.
Second, we quantitatively study how much over-parameterization is needed to ensure the Gram matrix generated by the random initialization.
The bound will depend on number of samples $n$ and properties of the activation function.
This analysis framework is fully general that it can explain fully connected neural network, ResNet, convolutional neural considered in this paper and other neural network architectures that satisfy the general setup defined below.

We begin with some notations.
Suppose that we have a sequence of real vector spaces
\begin{equation*}
	\mathbb{R}^{p^{(0)}}\rightarrow \mathbb{R}^{p^{(1)}}\rightarrow \cdots\rightarrow \mathbb{R}^{p^{(H)}}.
\end{equation*}
\begin{rem}
For fully-connected neural network and ResNet, $p^{(0)} = p^{(1)} =\ldots =p^{(H)} = 1$.
For convolutional neural network, $p^{(h)}$ is the number of patches of the $h$-th layer.
\end{rem}

For each pair $(\mathbb{R}^{p^{(h-1)}},\mathbb{R}^{p^{(h)}})$, let $\linsub
\subset \linfunc(\mathbb{R}^{p^{(h-1)}},\mathbb{R}^{p^{(h)}})= \mathbb{R}^{p^{(h)}\times p^{(h-1)}}$ be a linear subspace.
\begin{rem}
For convolutional neural network, the dimension of $\linsub$ is the filter size.
\end{rem}
In this section, by Gaussian distribution $\gaussian$ over a $q$-dimensional subspace $\linsub$, we mean that for a basis $\{\vect{e}_1,\ldots,\vect{e}_q\}$ of $\linsub$ and $(v_1,\ldots,v_q) \sim N(\vect{0},\mat{I})$ such that $\sum_{i=1}^{q}v_i\vect{e}_i \sim \gaussian$.
In this section, we equip one Gaussian distribution $\gaussian^{(h)}$ with each linear subspace $\linsub^{(h)}$.
By an abuse of notation, we also use $\linsub$ to denote a transformation. For $\mat{K} \in \mathbb{R}^{p^{(h-1)}\times p^{(h-1)}}$, we let \[
\linsub^{(h)}\left(\mat{K}\right) = \expect_{\mat{W}\sim\gaussian^{(h)}}\left[
\mat{W}\mat{K}\mat{W}^\top
\right].
\]

We also consider a deterministic linear mapping $\detmap^{(h)}:\mathbb{R}^{n^{(h-1)}}\rightarrow \mathbb{R}^{n^{(h)}}$.
For this section, we denote $\detmap^{(1)} = \mat{0}$, i.e., the zero mapping.
\begin{rem}
	For full-connected neural networks, we take $\detmap^{(h)}$ to be the zero mapping.
	For ResNet and convolutional ResNet, we take $\detmap^{(h)}$ to be the identity mapping.
\end{rem}
Let $\activate^{(1)},\cdots,\activate^{(H)}$ be a sequence of  activation functions over $\mathbb{R}$.
Note here we use $\activate$ instead of $\sigma$ to denote the activation function because we will incorporate the scaling in $\activate$ for the ease of presentation and the full generality.

Now we recursively define the output of each layer in this setup.
In the following, we use $h \in [H]$ to index layers, $i \in [n]$ to index data points, $\alpha, \beta, \gamma \in [m]$ or $[d]$ to index channels (for CNN) or weight vectors (for fully connected neural networks or ResNet).
\begin{rem}
$d=1$ for fully connected neural network and ResNet and $d \ge 1$ for convolutional neural network because $d$ represents the number of input channels.
\end{rem}
We denote $\vect{X}^{(h),[\alpha]}_i$ an $p^{(h)}$-dimensional vector which is the output at $(h-1)$-th layer.
We have the following recursive formula
	\begin{align*}
	\vect{X}^{(1),(\alpha)}_i =&\rho^{(h)}\left(\sum_ \beta \mat{W}^{(h),(\alpha)}_{(\beta)}\vect{X}^{(h-1),(\beta)}_i\right)\\
	\vect{X}^{(h),(\alpha)}_i =&\detmap^{(h)}(\vect{X}^{(h-1),(\alpha)}_i)+\rho^{(h)}\left(\frac{\sum_ \beta \mat{W}^{(h),(\alpha)}_{(\beta)}\vect{X}^{(h-1),(\beta)}_i}{\sqrt{m}}\right)
	\end{align*}
	where $\mat{W}^{(h),(\alpha)}_{(\beta)}$  is $p^{(h)}\times p^{(h-1)}$ matrix generated according to the following rule
	\begin{itemize}
		\item for $h=1$, $\mat{W}^{(h),(\alpha)}_{[\beta]}$ is defined for $1\le \alpha\le m$ and $1\le \beta\le d$; for $h>1$, $\mat{W}^{(h),(\alpha)}_{(\beta)}$ is defined for $1\le \alpha\le m$ and $1\le \beta\le m$;
		\item the set of random variables $\{\mat{W}^{(h),(\alpha)}_{(\beta)}\}_{h,\alpha,\beta}$ are independently generated;
		\item for fixed $h, \alpha, \beta$, $\mat{W}^{(h),(\alpha)}_{(\beta)}\sim \gaussian^{(h)}$.
	\end{itemize}
\begin{rem}
Choosing $\rho^{(h)}(z)$ to be $\relu{z}$ and $\detmap^{(h)}$ to be the zero mapping, we recover the fully-connected architecture.
Choosing $\rho^{(h)}(z)$ to be $\frac{c_{res}}{H}\relu{z}$ and $\detmap^{(h)}$ to be the identity mapping, we recover ResNet architecture.
\end{rem}
\begin{rem}
	Note here $\vect{X}_i^{(h)} = \vect{x}_i^{(h)}\sqrt{m}$ for $h \ge 1$ and 
	$\vect{X}_i^{(h)} = \vect{x}_i^{(h)}$ for $h = 0$ in the main text.
	We change the scaling here to simplify the calculation of expectation and the covariance in this section.
\end{rem}

With these notations, we first define the population Gram matrices recursively.

\begin{defn}\label{defn:pop_kernel}
We fix $(i,j) \in [n] \times [n]$, for $h=1,\ldots,H$.
The population Gram matrices are defined according to the following formula
\begin{align}
\mat{K}^{(0)}_{ij}=&\sum_\gamma (\vect{X}^{(0),[\gamma]}_i)^\top \vect{X}^{(0),[\gamma]}_j, \nonumber  \\
\vect{b}_i^{(0)} =&\vect{0},\nonumber\\
\mat{K}^{(h)}_{ij}=&\detmap^{(h)}\mat{K}^{(h-1)}_{ij}\detmap^{(h)\top}+\expect_{(\vect{U},\vect{V})}\left(\rho(\vect{U}) \detmap^{(h)}(\vect{\bias}^{(h-1)}_j)^\top+(\detmap^{(h)}(\vect{\bias}^{(h-1)}_i)) \rho(\vect{V})^\top+\rho(\vect{U})\rho(\vect{V})^\top)\right), \nonumber \\
\vect{b}_i^{(h)}=& \detmap^{(h)}(\vect{b}_i^{(h-1)})+\expect_{\mat{U}}\rho^{(h)}(\mat{U}),\label{eqn:kernel_defn}
\end{align}
where
\begin{equation}
(\mat{U},\mat{V})\sim N\left(\vect{0},\left(\begin{array}{ccc}
\linsub\left(\mat{K}^{(h-1)}_{ii}\right)  & \linsub\left(\mat{K}^{(h-1)}_{ij}\right) \\
\linsub\left(\mat{K}^{(h-1)}_{ji}\right) & \linsub\left(\mat{K}^{(h-1)}_{jj}\right)\\
\end{array}\right)\right). \label{eqn:uv_dist}
\end{equation}
\end{defn}
Notice that the Gram matrix of the next layer $\mat{K}^{(h)}$ not only depends on the previous layer's Gram matrix $\mat{K}^{(h-1)}$ but also depends on a ``bias" term $\vect{\bias}^{(h-1)}$.

Given the population Gram matrices defined in Equation~\eqref{eqn:kernel_defn} and~\eqref{eqn:uv_dist}, we derive the following quantitative bounds which characterizes how much over-parameterization, i.e., how large $m$ is needed to ensure the randomly generated Gram matrices is close to the population Gram matrices.

\begin{thm}\label{thm:main_general_framework}
With probability $1- \delta$ over the $\left\{\mat{W}_{(\beta)}^{(h),(\alpha)}\right\}_{h,\alpha,\beta}$, for any $ 1\le h\le H-1,1\le i,j\le n,$
	\begin{equation}
	\label{eqn:K_bound}
	\norm{\frac{1}{m}\sum\limits_{\alpha=1}^m(\vect{X}^{(h),(\alpha)}_i)^\top \vect{X}^{(h),(\alpha)}_j-\mat{K}^{(h)}_{ij}}_\infty\le \mathcal{E}\sqrt{\frac{\log (Hn\max_h p^{(h)}/\delta)}{m}}
	\end{equation}
	and any $h\in[H-1], \forall 1\le i\le n,$
\begin{equation}
	\label{eqn:b_bound}
\norm{\frac{1}{m}\sum_{\alpha=1}^m \vect{X}^{(h),(\alpha)}_i-\vect{b}^{(h)}_i}_\infty\le \mathcal{E} \sqrt{\frac{\log (Hn\max_h p^{(h)}/\delta)}{m}}
\end{equation}
The error constant $\error$ satisfies there exists an absolute constant $C>0$ such that
\begin{align*}
\error\le C \left(\prod_{h=2}^{H-1} \left(A_{(h)}+\Lambda_{(h)}\mathfrak{W}+ C_{(h)}A_{(h)}B\mathfrak{W}+C_{(h)} A_{(h)}\sqrt{\mathfrak{W}_{(h)}M}\right)\right)\times
		 \max\{\mathfrak{W}\sqrt{(1+ C^2_{(1)})M^2},\sqrt{C^2_{(1)}M}\}
\end{align*}
	where $M,B,\Lambda_{(h)},C_{(h)},A_{(h)},\wbound_{(h)}$ are defined by:
	\begin{itemize}
		\item $M=1+100\max_{i,j,p,q,h}|\linsub^{(h)}(\mat{K}^{(h-1)}_{ij})_{pq}|$,
		\item $A_{(h)}=1+\max\left\{\|\detmap^{(h)}\|_{L^\infty\rightarrow L^\infty},\|\detmap^{(h)}(\cdot)\detmap^{(h)\top}\|_{L^\infty\rightarrow L^\infty}\right\}$,
		\item $B=1+100\max_{i,p,h}|\vect{b}^{(h)}_{ip}|$,
		\item $C_{(h)}=\abs{\activate(0)}+\sup_{x\in \mathbb{R}}\abs{\activate'(x)}$,
		\item $\rhobound_{(h)}$ is a constant that only depends on $\rho^{(h)}$,
		\item $\wbound_{(h)}=1+\|\mathcal{W}^{(h)}\|_{L^\infty\rightarrow L^\infty}$.
	\end{itemize}
\end{thm}
\begin{rem}
	\label{rem:constants}
For fully-connected neural networks, we have $M = O(1),A_{(h)} =0, B = O(1), C_{(h)} = O(1), \Lambda_{(h)} = O(1), \wbound_{(h)} = O(1)$, so we need $m = \Omega\left(\frac{n^2\log(Hn/\delta)2^{O(H)}}{\lambda_0^2}\right)$.
For ResNet, we have $M = O(1),A_{(h)} =1, B = O(1), C_{(h)} = O(\frac{1}{H}), \Lambda_{(h)} = O(\frac{1}{H}), \wbound_{(h)} = O(1)$, so we need $m = \Omega\left(\frac{n^2\log(Hn/\delta)}{\lambda_0^2}\right)$.
The convolutional ResNet has the same parameters as ResNet but because the Gram matrix is $np \times np$, so we need $m = \Omega\left(\frac{n^2p^2\log(Hnp/\delta)}{\lambda_0^2}\right)$.
\end{rem}

\begin{proof}[Proof of Theorem~\ref{thm:main_general_framework}]
The proof is by induction.
For the base case, $h=1$, recall
\begin{equation*}
\vect{X}^{(1),[\alpha]}_i=\activate^{(1)}(\sum_ \beta \mat{W}^{(1),(\alpha)}_{(\beta)}\vect{X}^{(0),(\beta)}_i).
\end{equation*}
	
We define
\begin{equation*}
\vect{U}^{(1),(\alpha)}_i=\sum_\beta W^{(1),(\alpha)}_{(\beta)}\vect{X}^{(0),(\beta)}_i.
\end{equation*}

By our generating process of $\left\{\mat{W}_{(\beta)}^{(h),(\alpha)}\right\}_{h,\alpha,\beta}$,
the collection $\{\vect{U}_i^{(1),(\beta)}\}_{1\le i\le n,1\le \beta\le m}$ is a mean-zero Gaussian variable with covariance matrix:
\begin{align*}
	&\expect \vect{U}_i^{(1),(\alpha)} \left(\vect{U}_j^{(1),(\beta)}\right)^\top\\
	=&\expect\sum_ {\gamma, \gamma'} \mat{W}_{(\gamma)}^{(1),(\alpha)} \vect{X}^{(0),(\gamma)}_i  \left(\vect{X}^{(0),(\gamma')\top}_i\right)^\top \left(\mat{W}_{(\gamma')}^{(1),(\beta)}\right)^\top\\
	=&\delta_{\alpha \beta}\linsub^{(1)}\left(\sum_\gamma \left(\vect{X}^{(0),(\gamma)}_i\vect{X}_j^{(0),(\gamma)}\right)^\top\right)\\
	=&\delta_{\alpha \beta}\mathcal{W}^{(1)}(\mat{K}^{(0)}_{ij})
\end{align*}
Therefore, we have
\begin{align*}
\expect\left[\frac{1}{m}\sum_{i=1}^m\vect{X}_i^{(1),(\alpha)}\vect{X}_j^{(1),(\alpha)\top}\right]=&\mat{K}^{(1)}_{ij}\\
\mathbb{E}\left[\frac{1}{m}\sum_{i=1}^m\vect{X}_i^{(1),(\alpha)}\right]=&\vect{b}^{(1)}_i.
\end{align*}
Now we have calculated the expectation.
Note since inside the expectation is an average, we can apply standard  standard Bernstein bounds and Hoeffding bound and obtain the following concentration inequalities.
With probability at least $1- \frac{\delta}{H}$, we have
\begin{align*}
\max_{i,j}\norm{\frac{1}{m}\sum_{i=1}^m\vect{X}_i^{(1),(\alpha)}\vect{X}_j^{(1),(\alpha)\top}-\mat{K}^{(1)}_{ij}}_\infty\le&\sqrt{\frac{16(1+2 C^2_{(1)}/\sqrt{\pi})M^2\log(4Hn^2(p^{(1)})^2/\delta)}{m}},\\
\max_{i,p}\abs{\frac{1}{m}\sum_{\alpha=1}^m \vect{X}_{ip}^{(1),(\alpha)}- \vect{b}_{ip}^{(1)}}
\le&\sqrt{\frac{2C^2_{(1)}M\log(2np^{(1)}H/\delta)}{m}}
\end{align*}

Now we prove the induction step.
 Define for $1\le h\le H$
\begin{align*}
	\hat{\mat{K}}_{ij}^{(h)}=&\frac{1}{m}\sum_\gamma \vect{X}^{(h),(\gamma)}_i\left(\vect{X}^{(h),(\gamma)}_j\right)^\top\\
	\hat{\vect{b}}_{i}^{(h)}=&\frac{1}{m}\sum_\gamma \vect{X}^{(1),(\gamma)}_i
\end{align*}
In the following, by $\mathbb{E}^{(h)}$ we mean taking expectation conditioned on first $(h-1)$ layers.

Now suppose that Equation~\eqref{eqn:K_bound} and~\eqref{eqn:b_bound} hold for $1\le l\le h$ with probability at least $1- \frac{h}{H}\delta$, now we want to show the equations holds for $h+1$ with probability at least $1- \delta/H$ conditioned on previous layers satisfying Equation~\eqref{eqn:K_bound} and~\eqref{eqn:b_bound}.
Let $l=h+1$.
recall	
\begin{align*}
\vect{X}^{(l),(\alpha)}_i=\detmap^{(l)}(\vect{X}^{(l-1)})+\activate^{(l)}\left(\frac{\sum_\beta \mat{W}^{(l),(\alpha)}_{(\beta)}\vect{X}^{(l-1),(\beta)}_i}{\sqrt{m}}\right).
\end{align*}
Similar to the base case, denote
\begin{align*}
\mat{U}^{(l),(\alpha)}_i=\frac{\sum_ \beta \mat{W}^{(l),(\alpha)}_{(\beta)}\vect{X}^{(l-1),(\beta)}_i}{\sqrt{m}}.
\end{align*}
Again note that $\{\vect{U}_i^{(l),(\beta)}\}_{1\le i\le n,1\le \beta\le m}$ is a collection of mean-zero Gaussian variables with covariance matrix:
\begin{align*}
\expect \left[\vect{U}_i^{(1),(\alpha)} \left(\vect{U}_j^{(1),(\beta)}\right)^\top\right]=
\delta_{\alpha \beta}\linsub^{(l)}(\hat{\mat{K}}_{ij}^{(l-1)})
\end{align*}

Now we get the following formula for the expectation:
\begin{align*}
\expect^{(l)}[\hat{\mat{K}}^{(l)}_{ij}]=&\detmap^{(l)}\hat{\mat{K}}^{(l-1)}_{ij}\left(\detmap^{(l)}\right)^\top+\expect_{(\vect{U},\vect{V})}\left(\rho^{(l)}(\vect{U})^\top \detmap^{(l)}(\hat{\vect{b}}^{(l-1)}_j)+(\detmap^{(l)}(\hat{\vect{b}}^{(l-1)}_i))^\top \rho^{(l)}(\vect{V})+\rho^{(l)}(\vect{U})^\top\rho^{(l)}(\vect{V}))\right)\\
\expect^{(l)}\hat{\vect{b}}_i^{(l)}=&\detmap^{(l)}(\hat{\vect{b}}_i^{(l-1)})+\expect_{\vect{U}}\rho^{(l)}(\vect{U})
\end{align*}
with
\begin{align*}
		(\vect{U},\vect{V})\sim N\left(\vect{0},\left(\begin{array}{ccc}
		\linsub^{(l)}(\hat{\mat{K}}_{ii}^{(l-1)}) & \linsub^{(l)}( \hat{\mat{K}}_{ij}^{(l-1)})\\
		\linsub^{(l)}( \hat{\mat{K}}_{ji}^{(l-1)}) & \linsub^{(l)}( \hat{\mat{K}}_{jj}^{(l-1)})\\
		\end{array}\right)\right)
\end{align*}
Same as the base case, applying concentration inequalities, we have with probability at least $1- \delta/H$,
\begin{align*}
\max_{ij}\|\expect^{(l)}\hat{\mat{K}}^{(l)}_{ij}- \hat{\mat{K}}^{(l)}_{ij}\|_\infty\le &\sqrt{\frac{16(1+2 C^2_{(l)}/\sqrt{\pi})M^2\log(4Hn^2(p^{(l)})^2/\delta)}{m}},\\
\max_{i}\|\expect^{(l)}\hat{\vect{b}}^{(l)}_{i}-\hat{\vect{b}}^{(l)}_{i}\|_\infty\le &\sqrt{\frac{2C^2_{(l)}M\log(2np^{(1)}H/\delta)}{m}}
\end{align*}

Now it remains to bound the differences
\begin{align*}
\max_{ij}\norm{\expect^{(l)}\hat{\mat{K}}_{ij}^{(l)}- \mat{K}_{ij}^{(l)}}_{\infty} \text{ and }
\max_{i}\norm{\expect^{(l)}\hat{\vect{b}}^{(l)}_{ij}- \vect{b}^{(l)}_{ij}}_{\infty}
\end{align*}
which determine how the error propagates through layers.

We analyze the error directly. 
\begin{align*}
&\norm{\expect^{(l)}\hat{\mat K}_{ij}^{(l)}- \mat K_{ij}^{(l)}}_{\infty}\\
\le& \norm{\detmap^{(l)}\hat{\mat K}^{(l-1)}_{ij}\detmap^{(l)\top}-\detmap^{(l)}\mat{K}^{(l-1)}_{ij}\detmap^{(l)\top}}_{\infty}\\
&+\norm{\expect_{(\vect{U},\vect{V})\sim \hat{\mat{A}}}\rho^{(l)}(\vect{U})^\top \detmap^{(l)}(\hat{\vect{b}}^{(l-1)}_j)-\expect_{(\vect{U},\vect{V})\sim \mat{A}}\rho^{(l)}(\vect{U})^\top \detmap^{(l)}(\vect{b}^{(l-1)}_j)}_{\infty}\\
&+\norm{\expect_{(U,V)\sim \hat{\mat{A}}}(\detmap^{(l)}(\hat{\vect{b}}^{(l-1)}_i))^\top \rho^{(l)}(\vect{V})-\expect_{(\vect{U},\vect{V})\sim \mat A}(\detmap^{(l)}(\vect{b}^{(l-1)}_i))^\top \rho^{(l)}(\vect{V})}_{\infty}\\
&+\norm{\expect_{(\vect{U},\vect{V})\sim \hat{\mat{A}}}\activate^{(l)}(\vect{U})^\top\rho^{(l)}(\vect{V}))-\expect_{(\vect{U},\vect{V})\sim \mat{A}}\rho^{(l)}(\vect{U})^\top\rho^{(l)}(\vect{V}))}_{\infty}\\
\le&  \norm{\detmap^{(l)}\hat{\mat K}^{(l-1)}_{ij}\detmap^{(l)\top}-\detmap^{(l)}\mat{K}^{(l-1)}_{ij}\detmap^{(l)\top}}_{\infty}\\
&+\norm{\expect_{(\vect{U},\vect{V})\sim \hat{\mat A}}\rho^{(l)}(\vect{U})^\top \detmap^{(l)}(\hat{\vect{b}}^{(l-1)}_j)-\expect_{(\vect{U},\vect{V})\sim \hat{\mat A}}\rho^{(l)}(\vect{U})^\top \detmap^{(l)}(\vect{b}^{(l-1)}_j)}_{\infty}\\
&+\norm{\expect_{(\vect{U},\vect{V})\sim \hat{\mat A}}\rho^{(l)}(\vect{U})^\top \detmap^{(l)}(\vect{b}^{(l-1)}_j)-\expect_{(\vect{U},\vect{V})\sim \mat{A}}\rho^{(l)}(\vect{U})^\top \detmap^{(l)}(\vect{b}^{(l-1)}_j)}_{\infty}\\
&+\norm{\expect_{(\vect{U},\vect{V})\sim \hat{\mat A}}(\vect{a}^{(l)}(\hat{\vect{b}}^{(l-1)}_i))^\top \rho^{(l)}(\vect{V})-\expect_{(\vect{U},\vect{V})\sim \hat{\mat{A}}}(\detmap^{(l)}(\vect{b}^{(l-1)}_i))^\top \rho^{(l)}(\vect{V})}_{\infty}\\
&+\norm{\expect_{(\vect{U},\vect{V})\sim \hat{A}}(\detmap^{(l)}(\vect{b}^{(l-1)}_i))^\top \rho^{(l)}(\vect{V})-\expect_{(\vect{U},\vect{V})\sim \mat{A}}(\detmap^{(l)}(\vect{b}^{(l-1)}_i))^\top \rho^{(l)}(\vect{V})}_{\infty}\\
&+\norm{\expect_{(\vect{U},\vect{V})\sim \hat{\mat{A}}}\activate^{(l)}(\vect{U})^\top\rho^{(l)}(\vect{V}))-\expect_{(\vect{U},\vect{V})\sim \mat{A}}\rho^{(l)}(\vect{U})^\top\rho^{(l)}(\vect{V}))}_{\infty}
\end{align*}
where we define
\begin{align*}
\hat{\mat A}=\left(\begin{array}{ccc}
\linsub^{(l)}(\hat{\mat 
K}_{ii}^{(l-1)}) & \linsub^{(l)}( \hat{\mat K}_{ij}^{(l-1)})\\
\linsub^{(l)}( \hat{\mat K}_{ji}^{(l-1)}) & \linsub^{(l)}( \hat{\mat K}_{jj}^{(l-1)})\\
\end{array}\right)
\text{and }
\mat{A}=\left(\begin{array}{ccc}
\linsub^{(l)}(\mat 
K_{ii}^{(l-1)}) & \linsub^{(l)}(\mat K_{ij}^{(l-1)})\\
\linsub^{(l)}( \mat K_{ji}^{(l-1)}) & \linsub^{(l)}( \mat K_{jj}^{(l-1)})\\
\end{array}\right)
\end{align*}

By definition, we have
\begin{align*}
\|\mat A-\hat{\mat A}\|_\infty\le &\wbound\max\limits_{ij}\|\hat{\mat K}^{(l-1)}_{ij}-\mat K^{(l-1)}_{ij}\|_{\infty}\text{ and }\\
\norm{\detmap^{(l)}\hat{\mat K}^{(l-1)}_{ij}\detmap^{(l)\top}-\detmap^{(l)}\mat K^{(l-1)}_{ij}\detmap^{(l)\top}}_{\infty}\le & A_{(l)}\max\limits_{ij}\|\hat{\mat K}^{(l-1)}_{ij}-\mat K^{(l-1)}_{ij}\|_{\infty}.
\end{align*}
We can also estimate other terms
\begin{align*}
&\norm{\expect_{(\vect{U},\vect{V})\sim \hat{\mat A}}\rho^{(l)}(\vect{U})^\top \detmap^{(l)}(\hat{\vect{b}}^{(l-1)}_j)-\expect_{(\vect{U},\vect{V})\sim \hat{ \mat A}}\rho^{(l)}(U)^\top \detmap^{(l)}(\vect{b}^{(l-1)}_j)}_{\infty}\\
\le &\norm{\expect_{(\vect{U},\vect{V})\sim \hat{\mat A}}\rho^{(l)}(\vect{U})^\top \detmap^{(l)}\left(\hat{\vect{b}}^{(l-1)}_j-\vect{b}^{(l-1)}_j\right)}_{\infty}\\
\le &C_{(l)}A_{(l)}\sqrt{\wbound \max_{ij}\|\hat{\mat K}_{ij}^{(l)}\|_\infty}\max_{i}\norm{\hat{\vect{b}}^{(l-1)}_{ij}- \vect{b}^{(l-1)}_{ij}}_{\infty}\\
\le &C_{(l)} A_{(l)}\sqrt{\mathfrak{W}_{(l)}M}\max_{i}\norm{\hat{\vect{b}}^{(l-1)}_{ij} - \vect{b}^{(l-1)}_{ij}}_{\infty},
\end{align*}

	\begin{align*}
&\norm{\expect_{(\vect{U},\vect{V})\sim \hat{\mat A}}\rho^{(l)}(\vect{U})^\top \detmap^{(l)}(\vect b^{(l-1)}_j) - \expect_{(\vect U,\vect V)\sim \mat A}\rho^{(l)}(\vect U)^\top \detmap^{(l)}(\vect{b}^{(l-1)}_j)}_{\infty}\\
\le &A_{(l)}B C_{(l)}\|\mat A-\hat{\mat A}\|_\infty \\
\le  &A_{(l)}BC_{(l)} \wbound \max\limits_{ij}\|\hat{\mat K}^{(l-1)}_{ij}-\mat K^{(l-1)}_{ij}\|_{\infty},
	\end{align*}
and
	\begin{align*}
&\norm{\expect_{(\vect{U},\vect{V})\sim \hat{\mat A}}\rho^{(l)}(\vect{U})^\top\rho^{(l)}(\vect{V})-\expect_{(\vect{U},\vect{V})\sim \mat A}\rho^{(l)}(U)^\top\rho^{(l)}(V)}_{\infty}\\
\le &\rhobound_{(l)}\|\mat A-\hat{\mat A}\|_{\infty}\\
\le &\rhobound_{(l)}\wbound\max\limits_{ij}\|\hat{\mat K}^{(l-1)}_{ij}-\mat K^{(l-1)}_{ij}\|_{\infty}.
	\end{align*}
where we have used Lemma~\ref{lem:2p2pperb}.

Putting these estimates together, we have
	
\begin{align*}
&\max_{ij}\|\expect^{(l)}\hat{\mat K}_{ij}^{(l)}- \mat K_{ij}^{(l)}\|_\infty\\
\le& \left(A_{(l)}+\rhobound_{(l)}\wbound+2C_{(l)}A_{(l)}B\wbound\right)\max\limits_{ij}\|\hat{\mat K}^{(l-1)}_{ij}-\mat K^{(l-1)}_{ij}\|_{\infty}+2C_{(l)} A_{(l)}\sqrt{\wbound_{(l)}M}\max_{i}\|\hat{\vect b}^{(l-1)}_{ij}- \vect{b}^{(l-1)}_{ij}\|_\infty\\
\le& \left(A_{(l)}+\rhobound_{(l)}\wbound+2C_{(l)}A_{(l)}B\wbound+2C_{(l)} A_{(l)}\sqrt{\wbound_{(l)}M}\right)\left(\max\limits_{ij}\|\hat{\mat K}^{(l-1)}_{ij}-\mat K^{(l-1)}_{ij}\|_{\infty}\vee\max_{i}\|\hat{\vect b}^{(l-1)}_{ij}- \vect b^{(l-1)}_{ij}\|_\infty\right)
\end{align*}
	and
\begin{align*}
\max_{i}\norm{\expect^{(l)}\hat{\vect b}^{(l)}_{ij}- \vect b^{(l)}_{ij}}_{\infty}&\le 
\rhobound_{(l)}\wbound\max_{ij}\norm{\hat{\mat K}^{(l-1)}_{ij} - \mat K^{(l-1)}_{ij}}_\infty+
A_{(l)}\max_{i}\norm{\hat{\vect b}^{(l-1)}_{ij}- \vect b^{(l-1)}_{ij}}_\infty\\
&\le (A_{(l)}+\rhobound_{(l)}\wbound)\left(\max_{ij}\|\hat{\mat K}^{(l-1)}_{ij}-\mat K^{(l-1)}_{ij}\|_\infty\vee\max_{i}\|\hat{\vect b}^{(l-1)}_{ij}- \vect{b}^{(l-1)}_{ij}\|_\infty\right).
\end{align*}

These two bounds imply the theorem.
\end{proof}

\subsection{From $\mat{K}^{(H-1)}$ to $\mat{K}^{(H)}$}
Recall $\mat{K}^{(H)}$ defined in Equation~\eqref{eqn:kernel_mlp},~\eqref{eqn:kernel_resnet} and~\eqref{eqn:kernel_conv_resnet}.
Note the definition of $\mat{K}^{(H)}$ is qualitatively different from that of  $\mat{K}^{(h)}$ for $h = 1,\ldots,H-1$ because $\mat{K}^{(H)}$ depends on $\mat{K}^{(H)}$ and $\sigma'(\cdot)$ instead of $\sigma(\cdot)$.
Therefore, we take special care of $\mat{K}^{(H)}$.
Further note $\mat{K}^{(H)}$ for our three architectures have the same form and only differ in scaling and dimension, so we will only prove the bound for the fully-connected architecture.
The generalization to ResNet and convolutional ResNet is straightforward.

\begin{lem}\label{lem:H_concen_mlp}
For $(i,j) \in [n] \times [n]$, define 
\begin{align*}
\hat{\mat{K}}_{ij}^{(H-1)} = \hat{\mat{K}}_{ij}^{(H-1)}\expect_{\vect{w}\sim N(\vect{0},\mat{I})}\left[\sigma'(\vect{w}^\top \vect{x}_i^{(H-1)}(0))\sigma'(\vect{w}^\top \vect{x}_j^{(H-1)}(0))\right].
\end{align*}
and suppose $\abs{\hat{\mat{K}}_{ij}^{(H-1)}-\mat{K}_{ij}^{(H-1)}} \le \frac{c\lambda_0}{n^2}$ for some small constant $c>0$.
Then if $m = \Omega\left(\frac{n^2\log(n/\delta)}{\lambda_0^2}\right)$, we have with probability at least $1-\delta$ over $\{\vect{w}_r^{(H)}(0)\}_{r=1}^m$ and $\{a_r(0)\}_{r=1}^m$, we have $\norm{\mat{G}^{(H)}(0)-\mat{K}^{(H)}}_{op} \le \frac{\lambda_0}{4}$.
\end{lem}
\begin{proof}[Proof of Lemma~\ref{lem:H_concen_mlp}]
We decompose
\begin{align*}
\mat{G}^{(H)}(0) - \mat{K}^{(H)} 
=  \left(\mat{G}^{(H)}(0) - \hat{\mat{K}}^{(H)} \right)+ \left(\hat{\mat{K}}^{(H)} -\mat{K}^{(H)}\right).
\end{align*}
Recall $\mat{G}^{(H)}$ defined in Equation~\eqref{eqn:H_mlp}.
Based on its expression, it is straightforward to use concentration inequality to show if $m = \Omega\left(\frac{n^2\log(n/\delta)}{\lambda_0^2}\right)$, we have 
\begin{align*}
\norm{\mat{G}^{(H)}(0) - \hat{\mat{K}}^{(H)} }_{op} \le \frac{\lambda_0}{8}.
\end{align*}
For the other 

Recall $\mat{A}_{ij}^{(H)} = \begin{pmatrix}
\mat{K}^{(H-1)}_{ii} & \mat{K}^{(H-1)}_{ij} \\
\mat{K}^{(H-1)}_{ji}  & \mat{K}^{(H-1)}_{jj} 
\end{pmatrix}$ and let $\hat{\mat{A}}_{ij}^{(H)} = \begin{pmatrix}
\hat{\mat{K}}^{(H-1)}_{ii} & \hat{\mat{K}}^{(H-1)}_{ij} \\
\hat{\mat{K}}^{(H-1)}_{ji}  & \hat{\mat{K}}^{(H-1)}_{jj} 
\end{pmatrix}$.

According to Lemma~\ref{lem:22perb} (viewing $\sigma'(\cdot)$ as the $\sigma(\cdot)$ in Lemma~\ref{lem:22perb}), we know \begin{align*}
\abs{\expect_{(\vect{U})\sim \hat{\mat{A}}_{ij}} \left[\sigma'(u)\sigma'(v)\right] - \expect_{(u,v)\sim \mat{A}_{ij}} \left[\sigma'(u)\sigma'(v)\right]} \le C\abs{\hat{\mat{A}}_{ij}-\mat{\mat{A}}_{ij}}
\end{align*} for some constant $C>0$.
Since $c$ is small enough, we directly have \begin{align*}
\norm{\hat{\mat{K}}^{(H)} -\mat{K}^{(H)}}_{op} \le \frac{\lambda_0}{8}
\end{align*}
\end{proof}

\begin{rem}
Combing Theorem~\ref{thm:main_general_framework}, Lemma~\ref{lem:H_concen_mlp}  and standard matrix perturbation bound directly have Lemma~\ref{lem:mlp_least_eigen}.
Similarly we can prove Lemma~\ref{lem:resnet_least_eigen} and Lemma~\ref{lem:conv_resnet_least_eigen}.
\end{rem}

%% file: justification.tex
\subsection{Full Rankness of $\mat{K}^{(h)}$ for the Fully-connected Neural Network}
\label{sec:fullrank_fc}
In this section we show as long as no two input vectors are parallel, then $\mat{K}^{(H)}$ defined in Equation~\eqref{eqn:kernel_resnet} is strictly positive definite.
\begin{prop}\label{prop:fullrank_fc}
Assume $\sigma(\cdot)$ satisfies Condition~\ref{cond:analytic} and for any $i,j \in [n], i\neq j$, $\vect{x}_i \not\parallel \vect{x}_j$.
Then we have $\lambdamin > 0$ where $\lambdamin$ is defined in Equation~\eqref{eqn:kernel_mlp}.
\end{prop}
\begin{proof}[Proof of Proposition~\ref{prop:fullrank_fc}]
By our assumption on the data point and using Lemma~\ref{lem:kernel-positive} we know $\mat{K}^{(1)}$ is strictly positive definite.

By letting $\mat{Z}= \mat{D}^{1/2} \mat{U}^\top$ , where $\mat{U}\mat{D} \mat{U}^\top =\mat{K}^{h}$. We then use Lemma~\ref{lem:kernel-positive} inductively for $(H-2)$ times to conclude $\mat{K}^{(H-1)}$ is strictly positive definite.
Lastly we use Lemma~\ref{lem:gradient-kernel-positive} to finish the proof.
\end{proof}

\begin{lem}\label{lem:kernel-positive}
Assume $\sigma(\cdot)$ is analytic and not a polynomial function.
Consider data $Z = \{\vect z_i \}_{i\in [n]}$ of $n$ non-parallel points (meaning $\vect z_i \notin \text{span}(\vect z_j)$ for all $i \neq j$).
Define \[\mat G (Z)_{ij} = \E_{\vect{w}\sim N(\vect{0},\mat{I})}[\sigma(\vect w^\top \vect z_i) \sigma(\vect w^\top \vect z_j) ].\] Then $\lambda_{\min} (\mat G(Z)) >0$.
\end{lem}
\begin{proof}[Proof of Lemma~\ref{lem:kernel-positive}]
The feature map induced by the kernel $\mat G$ is given by $\phi_{\vect z} (\vect w) = \sigma(\vect w^\top \vect z) \vect z$. To show that $\mat G(Z)$ is strictly positive definite, we need to show $\phi_{\vect z_1}(\vect w), \ldots, \phi_{\vect z_n} (\vect w)$ are linearly independent functions.
Assume that there are $a_i$ such that
	\begin{align*}
	0 &= \sum_i a_i \phi_{\vect z_i} = \sum_i a_i \sigma( \vect w^\top  \vect z_i ) \vect z_i.
	\end{align*}
	We wish to show that $a_i=0$. Differentiating the above equation $(n-2)$ times with respect to $\vect w$, we have
	\begin{align*}
	0&= \sum_i \left( a_i \sigma^{(n-1)} ( \vect w^\top \vect z_i) \right) \vect  z_i ^{\otimes (n-1)}.
	\end{align*}
Using Lemma~\ref{lem:linear_ind}, we know $\left\{\vect{z}_i^{\otimes (n-1)}\right\}_{i=1}^n$ are linearly independent.
Therefore, we must have  $a_i \sigma^{(n-1)} (\vect  w^\top \vect z_i) = 0$ for all $i$.
Now choosing a $\vect{w}$ such that $\sigma^{(n-1)}\left(\vect{w}^\top \vect{z}_i\right) \neq 0$ for all $i \in [n]$ (such $\vect{w}$ exists because of our assumption on $\sigma$), we have $a_i = 0$ for all $i \in [n]$.
\end{proof}

\begin{lem}\label{lem:gradient-kernel-positive}
Assume $\sigma(\cdot)$ is analytic and not a polynomial function.
Consider data $Z = \{\vect z_i \}_{i\in [n]}$ of $n$ non-parallel points (meaning $\vect z_i \notin \text{span}(\vect z_j)$ for all $i \neq j$).
Define \[\mat G(Z)_{ij} = \E_{\vect{w}\sim N(\vect{0},\mat{I})}[\sigma'(\vect w^\top \vect z_i) \sigma'(\vect w^\top \vect z_j) (\vect z_i ^\top \vect z_j)].\] Then $\lambda_{\min} (\mat G(Z)) >0$.
\end{lem}
\begin{proof}[Proof of Lemma~\ref{lem:gradient-kernel-positive}]
The feature map induced by the kernel $\mat G$ is given by $\phi_{\vect z} (\vect w) = \sigma'(\vect w^\top \vect z) \vect z$. To show that $\mat G(Z)$ is strictly positive definite, we need to show $\phi_{\vect z_1}(\vect w), \ldots, \phi_{\vect z_n} (\vect w)$ are linearly independent functions.
Assume that there are $a_i$ such that
	\begin{align*}
	0 &= \sum_i a_i \phi_{\vect z_i} = \sum_i a_i \sigma'( \vect w^\top  \vect z_i ) \vect z_i.
	\end{align*}
	We wish to show that $a_i=0$. Differentiating the above equation $(n-2)$ times with respect to $\vect w$, we have
	\begin{align*}
	0&= \sum_i \left( a_i \sigma^{(n)} ( \vect w^\top \vect z_i) \right) \vect  z_i ^{\otimes (n-1)}.
	\end{align*}
Using Lemma~\ref{lem:linear_ind}, we know $\left\{\vect{z}_i^{\otimes (n-1)}\right\}_{i=1}^n$ are linearly independent.
Therefore, we must have  $a_i \sigma^n (\vect  w^\top \vect z_i) = 0$ for all $i$.
Now choosing a $\vect{w}$ such that $\sigma^{(n)}\left(\vect{w}^\top \vect{z}_i\right) \neq 0$ for all $i \in [n]$ (such $\vect{w}$ exists because of our assumption on $\sigma$), we have $a_i = 0$ for all $i \in [n]$.
\end{proof}

\subsection{Full Rankness of $\mat{K}^{(h)}$ for ResNet}
\label{sec:fullrank_resnet}

In this section we show as long as no two input vectors are parallel, then $\mat{K}^{(H)}$ defined in Equation~\eqref{eqn:kernel_resnet} is strictly positive definite.
Furthermore, $\lambdamin$ does not depend inverse exponentially in $H$.


\begin{prop}\label{prop:H_lambda_0}
	Assume $\sigma(\cdot)$ satisfies Condition~\ref{cond:analytic} and for any $i,j \in [n], i\neq j$, $\vect{x}_i \not\parallel \vect{x}_j$.
	Recall that in Equation~\eqref{eqn:kernel_resnet}, we define \[\mat{K}^{(H)}_{ij} =  c_H \mat{K}^{(H-1)}_{ij} \cdot \expect_{\left(u,v\right)^\top \sim N\left(\vect{0},\begin{pmatrix}
		\mat{K}^{(H-1)}_{ii} & \mat{K}^{(H-1)}_{ij} \\
		\mat{K}^{(H-1)}_{ji} & \mat{K}^{(H-1)}_{jj}
		\end{pmatrix}\right) }\left[\sigma'(u)\sigma'(v)\right],
	\]
	where $c_H \sim \frac{1}{H^2}$.
	Then we have $\lambda_{\min} (\mat K^{(H)}) \ge c_H \kappa$, where $\kappa$ is a constant that only depends on the activation $\sigma$ and the input data.
In particular, $\kappa$ does not depend on the depth.
\label{prop:resnet-depth-ind}
\end{prop}
\begin{proof}[Proof of Proposition~\ref{prop:H_lambda_0}]
	First note	$\mat K_{ii} ^{(H-1)} \in [ 1/c_{x,0}^2, c_{x,0}^2  ]$ for all $H$, so it is in a bounded range that does not depend on the depth (c.f. Lemma~\ref{lem:lem:init_norm_res}).
	Define a function $$\mat{G}: \mathbb{R}^{n \times n} \rightarrow \mathbb{R}^{n\times n}$$ such that 
$\mat{G}(\mat K)_{ij} = \mat K_{ij} \expect_{
	\left(u,v\right)^\top \sim N\left(\vect{0},\begin{pmatrix}
	\mat{K}_{ii} & \mat{K}_{ij} \\
	\mat{K}_{ji} & \mat{K}_{jj}
	\end{pmatrix}\right)
	} [ \sigma'(u) \sigma'(v)]$. 
	Now define a scalar function \[g(\lambda) = \min_{\mat K: \mat K \succ 0, \frac{1}{c_{x,0}^2 }\le \mat K_{ii}\le c_{x,0 }  , \lambda(\mat K) \ge \lambda} \lambda_{\min} (\mat G(\mat K))  \] with \[\lambda(\mat K) = \min_{ij} \begin{pmatrix}
	\mat{K}_{ii} & \mat{K}_{ij} \\
	\mat{K}_{ji} & \mat{K}_{jj}
	\end{pmatrix}.\]
	
	By Lemma~\ref{lem:lambda_lower_bound}, we know
	$\lambda (\mat K^{(H-1)}) \ge  c_H \lambda\left(\mat{K}^{(0)}\right)$. 
	
	Next, let $\mat U\mat D\mat U^\top =\mat K^{(H-1)}$ be the eigen-ecomposition of $\mat K$, and $\mat Z= \mat D^{1/2} \mat U^\top$ be the feature embedding into $\mathbb{R}^n$. Since $\begin{pmatrix}
	\vect z_i ^\top \vect z_i & \vect z_i ^\top \vect z_j\\
	\vect z_j ^\top \vect z_i & \vect z_j ^\top \vect z_j 
	\end{pmatrix}$  is full rank, then $\vect z_i \notin \text{span}(\vect z_j)$. Then using Lemma \ref{lem:gradient-kernel-positive} , we know $g(\lambda\left(\mat{K}^{(0)}\right)) >0$. Thus we have established that $\lambda_{\min} (\mat K^{(H)}) \ge c_H g(\lambda\left(\mat{K}^{(0)}\right))$ , where $g(\lambda\left(\mat{K}^{(0)}\right))$ only depends on the input data and activation $\sigma$. In particular, it is independent of the depth.
\end{proof}

\begin{lem}\label{lem:lambda_lower_bound}
If $\detmap^{(h)}$ is the identity mapping defined in Section~\ref{sec:general_formulation}, then $\lambda\left(\mat{K}^{(H)}\right) \ge \min_{(i,j) \in [n] \times [n]} \lambda_{\min}\left(\begin{array}{ccc}
\mat K^{(0)}_{ii} & \mat K^{(0)}_{ij}\\
\mat K^{(0)}_{ji} & \mat K^{(0)}_{jj}\\
\end{array}\right)$.
\end{lem}
\begin{proof}[Proof of Lemma~\ref{lem:lambda_lower_bound}]
First recall
	
	\begin{equation*}
	(\vect{U},\vect{V})\sim N\left(\vect{0},\left(\begin{array}{ccc}
	\linsub^{(h)}(\mat K_{ii}^{(h-1)}) & \linsub^{(h)}( \mat K_{ij}^{(h-1)})\\
	\linsub^{(h)}(\mat K_{ji}^{(h-1)}) & \linsub^{(h)}( \mat K_{jj}^{(h-1)})\\
	\end{array}\right)\right)
	\end{equation*}
Then we compute
	\begin{equation*}
	\begin{split}
	&\mat K^{(h)}_{ij}-\vect b^{(h)}_i\vect b^{(h)\top}_j\\
	&=\detmap^{(h)}\mat K^{(h-1)}_{ij}\detmap^{(h)\top}+\expect_{(\vect U,\vect V)}\Big(\rho(\vect U) \detmap^{(h)}(\vect b^{(h-1)}_j)^\top+(\detmap^{(h)}(\vect b^{(h-1)}_i)) \rho(\vect V)^\top+\rho(\vect U)\rho(\vect V)^\top)\Big)\\
	&-\left(\detmap^{(h)}(\vect b_i^{(h-1)})+
	\expect_{\vect U}\rho^{(h)}(\vect U)\right)\left(\detmap^{(h)}(\vect b_j^{(h-1)})+\expect_{\vect V}\rho^{(h)}(\vect V))\right)^\top\\
	&=\detmap^{(h)}\left(\mat K^{(h-1)}_{ij} - \vect b_i^{(h-1)}\vect b_j^{(h-1)\top}\right)\detmap^{(h)\top}+\expect_{(\vect U,\vect V)}\Big(\rho(\vect U)\rho(\vect V)^\top)\Big)-\left(\expect_{\vect U}\rho^{(h)}(\vect U)\right)\left(\expect_{\vect V}\rho^{(h)}(\vect V))\right)^\top\\
	\end{split}
	\end{equation*}
	
	For ResNet, $\detmap^{(h)}$ is the identity mapping so
we have	
	\begin{equation*}
	\begin{split}
	&\mat K^{(h)}_{ij}-\vect b^{(h)}_i\vect b^{(h)\top}_j\\
	&=\mat K^{(h-1)}_{ij}-\vect b_i^{(h-1)}\vect b_j^{(h-1)\top}+\expect_{(\vect U,\vect V)}\Big(\rho(\vect U)\rho(\vect V)^\top)\Big)-\left(\expect_{\vect U}\rho^{(h)}(\vect U)\right)\left(\expect_{\vect V}\rho^{(h)}(\vect V))\right)^\top.
	\end{split}
	\end{equation*}
	
To proceed, we calculate
	
	\begin{equation*}
	\begin{split}
	&\left(\begin{array}{ccc}
	\mat K^{(h)}_{ii} & \mat K^{(h)}_{ij} \\
	\mat K^{(h)}_{ji} & \mat K^{(h)}_{jj}\\
	\end{array}\right)-\left(\begin{array}{ccc}
	\vect b_i^{(h)}\\
	\vect b_j^{(h)} \\
	\end{array}\right)\left(\begin{array}{ccc}
	\vect b_i^{(h)\top},
	\vect b_j^{(h)\top} \\
	\end{array}\right)\\
	=&\left(\begin{array}{ccc}
	\mat K^{(h-1)}_{ii} & \mat K^{(h-1)}_{ij} \\
	\mat K^{(h-1)}_{ji} & \mat K^{(h-1)}_{jj}\\
	\end{array}\right)-\left(\begin{array}{ccc}
	\vect b_i^{(h-1)}\\
	\vect b_j^{(h-1)} \\
	\end{array}\right)\left(\begin{array}{ccc}
	\vect b_i^{(h-1)\top},
	\vect b_j^{(h-1)\top} \\
	\end{array}\right)\\
&+ \left(\expect_{\vect U,\vect V}\left(\begin{array}{ccc}
	\rho^{(h)}(\vect U)\rho^{(h)}(\vect U)^{\top} & \rho^{(h)}(\vect U)\rho^{(h)}(\vect V)^{\top} \\
	\rho^{(h)}(\vect V)\rho^{(h)}(\vect U)^{\top} & \rho^{(h)}(\vect V)\rho^{(h)}(\vect V)^{\top}\\
	\end{array}\right)-
	\expect_{\vect U,\vect V}\left(\begin{array}{ccc}
	\rho(\vect U)\\
	\rho(\vect V) \\
	\end{array}\right)\expect_{\vect U,\vect V}\left(\begin{array}{ccc}
	\rho(\vect U)^{\top},
	\rho(\vect V)^{\top} \\
	\end{array}\right)\right)\\
\ge&\left(\begin{array}{ccc}
	\mat K^{(h-1)}_{ii} & \mat K^{(h-1)}_{ij} \\
	\mat K^{(h-1)}_{ji}  & \mat K^{(h-1)}_{jj}\\
	\end{array}\right)-\left(\begin{array}{ccc}
	\vect b_i^{(h-1)}\\
	\vect b_j^{(h-1)} \\
	\end{array}\right)\left(\begin{array}{ccc}
	\vect b_i^{(h-1)\top},
	\vect b_j^{(h-1)\top} \\
	\end{array}\right)
	\end{split}
	\end{equation*}
	
	As a result, we have

	\begin{equation}
	\begin{split}
	&\lambda_{\min}\left(\begin{array}{ccc}
	\mat K^{(h)}_{ii} & \mat K^{(h)}_{ij} \\
	\mat K^{(h)}_{ji} & \mat K^{(h)}_{jj}\\
	\end{array}\right)\\
	\ge&\lambda_{\min}\left(\begin{array}{ccc}
	\mat K^{(h)}_{ii} & \mat K^{(h)}_{ij} \\
	\mat K^{(h)}_{ji} & \mat K^{(h)}_{jj}\\
	\end{array}\right)-\left(\begin{array}{ccc}
	\vect b_i^{(h)}\\
	\vect b_j^{(h)} \\
	\end{array}\right)\left(\begin{array}{ccc}
	\vect b_i^{(h)\top},
	\vect b_j^{(h)\top} \\
	\end{array}\right)\\
	\ge&\min \lambda_{\min}\left(\begin{array}{ccc}
	\mat K^{(h-1)}_{ii} & \mat K^{(h-1)}_{ij} \\
	\mat K^{(h-1)}_{ji} & \mat K^{(h-1)}_{jj}\\
	\end{array}\right)-\left(\begin{array}{ccc}
	\vect b_i^{(h-1)}\\
	\vect b_j^{(h-1)} \\
	\end{array}\right)\left(\begin{array}{ccc}
	\vect b_i^{(h-1)\top},
	\vect b_j^{(h-1)\top} \\
	\end{array}\right)\\
	\ge &\cdots\\
	\ge&\lambda_{\min}\left(\begin{array}{ccc}
	\mat K^{(0)}_{ii} & \mat K^{(0)}_{ij} \\
	\mat K^{(0)}_{ji} & \mat K^{(0)}_{jj}\\
	\end{array}\right)-\left(\begin{array}{ccc}
	\vect b_i^{(0)}\\
	\vect b_j^{(0)} \\
	\end{array}\right)\left(\begin{array}{ccc}
	\vect b_i^{(0)\top},
	\vect b_j^{(0)\top} \\
	\end{array}\right) \\
	= &\lambda_{\min}\left(\begin{array}{ccc}
	\mat K^{(0)}_{ii} & \mat K^{(0)}_{ij} \\
	\mat K^{(0)}_{ji} & \mat K^{(0)}_{jj}\\
	\end{array}\right).
	\end{split}
	\label{eq:non-decreasing}
	\end{equation}
We now prove the theorem.
\end{proof}

%% file: useful_lemmas.tex
\begin{lem}\label{lem:difference_norm}
	Given a set of matrices $\{\mat{A}_i,\mat{B}_i:i\in [n]\}$, if $\norm{\mat{A}_i}_2\le M_i$, $\norm{\mat{B}_i}_2\le M_i$ and $\norm{\mat{A}_i-\mat{B}_i}_F\le \alpha_i M_i$, we have
	\begin{align*}
	\norm{\prod_{i=1}^{n}\mat{A}_i-\prod_{i=1}^{n}\mat{B}_i}_F \le \left(\sum_{i=1}^{n}\alpha_i\right)\prod_{i=1}^{n}M_i .
	\end{align*}
\end{lem}
\begin{proof}[Proof of Lemma~\ref{lem:difference_norm}]
	\begin{align*}
	&\norm{\prod_{i=1}^{n}\mat{A}_i-\prod_{i=1}^{n}\mat{B}_i}_F \\
	=&\norm{\sum_{i=1}^{n} \left(\prod_{j=1}^{i-1}\mat{A}_j\right)\left(\mat{A}_i-\mat{B}_i\right) \left(\prod_{k=i+1}^{n}\mat{B}_k\right)}_F\\
	\le&\sum_{i=1}^{n}\norm{ \left(\prod_{j=1}^{i-1}\mat{A}_j\right)\left(\mat{A}_i-\mat{B}_i\right) \left(\prod_{k=i+1}^{n}\mat{B}_k\right)}_F\\
	\le&\left(\sum_{i=1}^{n}\alpha_i\right)\prod_{i=1}^{n}M_i.
	\end{align*}
\end{proof}

\begin{lem}\label{lem:operator_norm_of_random_matrix}
	Given a matrix $\mat{W}\in\mathbb{R}^{m\times cm}$ with $\mat{W}_{i,j}\sim N(0,1)$, where $c$ is a constant. We have with probability at least $1-\exp\left(-\frac{\left(c_{w,0}-\sqrt{c}-1\right)^2m}{2}\right)$
	\[\norm{\mat{W}}_2\le c_{w,0}\sqrt{m},\]
	where $c_{w,0}>\sqrt{c}+1$ is a constant.
\end{lem}
\begin{proof}[Proof of Lemma~\ref{lem:operator_norm_of_random_matrix}]
	The lemma is a consequence of well-known deviations bounds concerning the singular values of Gaussian random matrices \cite{vershynin2010introduction}
	\begin{align*}
		P\left(\lambda_{\max}\left(\mat{W}\right)>\sqrt{m}+\sqrt{cm}+t\right)\le e^{t^2/2}.
	\end{align*}
	Choosing $t=\left(c_{w,0}-\sqrt{c}-1\right)\sqrt{m}$, we prove the lemma.
\end{proof}

\begin{lem}
\label{lem:1perb}
Assume $\relu{\cdot}$ satisfies Condition~\ref{cond:lip_and_smooth}.
For $a,b \in \mathbb{R}$ with $\frac{1}{c}<\min(a,b)$, $\max(a,b) < c$ for some constant $c > 0$, we have \[
\left|\expect_{z \sim N(0,a)}[\sigma(z)] - \expect_{z \sim N(0,b)}[ \sigma(z)]\right| \le C\abs{a-b}.
\] for some constant $C>0$ that depends only on $c$ and the constants in Condition~\ref{cond:lip_and_smooth}.
\end{lem}
\begin{proof}[Proof of Lemma~\ref{lem:1perb}]
We compute for any $\min (a,b)\le \alpha \le  \max(a,b)$\begin{align*}
	\abs{\frac{d\expect_{z \sim N(0,\alpha)} [\relu{z}]}{d\alpha}} 
	= \abs{\frac{d\expect_{z \sim N(0,1)} [\relu{\alpha z}]}{d\alpha} }
= \abs{ \expect_{z \sim N(0,1)}[z\sigma'(\alpha z)] } \le C.
\end{align*}
Applying Taylor's Theorem we finish the proof.
\end{proof}

\begin{lem}
\label{lem:22perb}
Assume $\relu{\cdot}$ satisfies Condition~\ref{cond:lip_and_smooth}.
Suppose that there exists some constant $c>0$ such that $ \mat A= \begin{bmatrix} a_1^2 & \rho a_1 b_1\\ \rho_1 a_1b_1 & b_1^2 \end{bmatrix}$, $\frac{1}{c}\le \min(a_1,b_1)$, $\max(a_1,b_1) \le c$, $ \mat B= \begin{bmatrix} a_2^2 & \rho_2 a_2 b_2\\ \rho a_2b_2 & b_2^2 \end{bmatrix}$, $\frac{1}{c}\le\min(a_2,b_2)$, $\max(a_2,b_2) \le c$

 and $\mat A,\mat B \succ 0$. 	Define 	$F(\mat A)=\mathbb{E}_{(u,v)\sim N(\vect{0}, \mat A)}\sigma(u)\sigma (v)$.
	Then, we have
\begin{align*}
\abs{F(\mat A) -F(\mat B)} \le C \| \mat A- \mat B \|_F 
 \le 2C\|\mat A-\mat B \|_\infty.
\end{align*} for some constant $C>0$ that depends only on $c$ and the constants in Condition~\ref{cond:lip_and_smooth}.
\end{lem}

\begin{proof}	
Let $ \mat A'= \begin{bmatrix} a^2 & \rho a b\\ \rho ab & b^2 \end{bmatrix} \succ 0$ 
with $\min(a_1,a_2)\le a\le \max(a_1,a_2)$, $\min(b_1,b_2) \le b \le \max(b_1,b_2)$ and $\min(\rho_1,\rho_2) \le \rho \le \max(\rho_1,\rho2)$. We can express 
\begin{align*}
F(\mat A') = \E_{(z_1,z_2)\sim \cN(0,C) } \sigma(a z_1) \sigma(bz_2)
\text{ with }\mat C=\begin{pmatrix} 1 & \rho \\ \rho &1 \end{pmatrix}.
\end{align*}

Recall $\gaussianspace =\{f: \int f(z) e^{-z^2/2} dz <\infty\}$ is the Gaussian function space.
We compute
\begin{align*}
\frac{dF}{da} &= \E[ \sigma'(az_1) \sigma(bz_2) z_1]\\
\Big|\frac{dF}{da}\Big| &\le \|\sigma'(az_1) z_1 \|_{L^2} \| \sigma(b z_2)\|_{L^2} \tag{$\|f\|_{L^2} := (\E f(z)^2)^{1/2}$, Cauchy}\\
&<\infty \tag{by Condition~\ref{cond:lip_and_smooth}}
\end{align*}
By the same argument, we have
\begin{align*}
\Big|\frac{dF}{db}\Big| <\infty 
\end{align*}
Next, let $\sigma_a (z):= \sigma(az) $ with Hermite expansion $\sigma_a (z) = \sum_{i=0}^\infty \alpha_i h_i (z)$ and similarly $\sigma_b (z) = \sum_i \beta_i h_i(z)$. Using the orthonormality that $\E [h_i (z) h_j(z)  ]=1_{i=j}$,
\begin{align*}
F(A) = \sum_{i=0}^\infty \alpha_i \beta_i \rho^i.
\end{align*}
Differentiating, we have 
\begin{align*}
\Big| \frac{dF}{d\rho} \Big|&= \Big|\sum_{i=1}^\infty \alpha_i \beta_i i \rho^{i-1}\Big|\\
&<\big( \sum_{i=1}^\infty \alpha_i ^2 i \big)^{1/2} \big( \sum_{i=1}^\infty \beta_i ^2 i \big)^{1/2} \tag{ $\rho=1$ and Cauchy}\\
&<\infty \tag{ Condition~\ref{cond:lip_and_smooth}}
\end{align*}

Note by Condition~\ref{cond:lip_and_smooth} we know there exists $B_\rho$, $B_a$ and $B_b$ such that $\big| \frac{dF}{d \rho}\big| \le B_\rho$,$ \big| \frac{dF}{d a}\big|\le B_a$,\text{ and } $\big| \frac{dF}{d b}\big|\le B_b$. 

Next, we bound $\nabla_{\mat A'} F(\mat A')$.
We see that 
\begin{align*}
\Big|\frac{dF}{dA_{11}' }\Big| &\le \Big| \frac{dF}{da}\Big| \Big|\frac{da}{dA_{11}'}\Big|\\
&\le B_a \frac{1}{2 \sqrt{A_{11}'}} \tag{since $ a= \sqrt{A_{11}'}$}\\
	&\le\frac12 B_a /c\\
	\Big|\frac{dF}{dA_{11}' }\Big| &\le \frac12 B_b /c \tag{analogous argument asa bove.}
\end{align*}

Using the change of variables, let $$g(A_{11}', A_{22}',A_{12}') = [ \sqrt{A_{11}'}, \sqrt{A_{22}'}, A_{12}'/ \sqrt{A_{11}' A_{22}'}]=[a, b, \rho].$$
By chain rule, we know
\begin{align*}
\frac{\partial F}{\partial A_{12}' } = \frac{\partial F}{\partial a}\frac{\partial a}{\partial  A_{12}'} + \frac{\partial  F}{\partial b}\frac{\partial b}{\partial A_{12}'}+\frac{\partial F}{\partial \rho}\frac{\partial \rho}{\partial  A_{12}'}= \frac{\partial F}{\partial \rho} \frac{\partial \rho}{\partial A_{12}'}.
\end{align*}
We can easily verify that $|\frac{\partial \rho}{\partial A_{12}'} |\le 1/c^2$, and so we have 
\begin{align*}
\big|\frac{\partial F}{\partial A_{12}' }\big| &\le \frac{B_\rho}{c^2}.
\end{align*}
Similarly, we have 
\begin{align*}
\Big| \frac{\partial F}{\partial A_{11}'} \Big| \le \frac{B_a}{c^2}\\
\Big| \frac{\partial F}{\partial A_{22}'} \Big| \le \frac{B_b}{c^2}
\end{align*}
Define $ B_\sigma= \max(B_a,B_b,B_\rho)$.
This establishes $\|\nabla F(\mat A') \|_F \le 2 B_\sigma/ c^2 \le C$ for some constant $C>0$. Thus by Taylor's Theorem, we have 
\begin{align*}
| F(\mat A) -F(\mat B) | \le C \| \mat A- \mat B \|_F 
 \le 2C\|\mat A-\mat B \|_\infty .
\end{align*}

\end{proof}

With Lemma~\ref{lem:1perb} and~\ref{lem:22perb}, we can prove the following useful lemma.
\begin{lem}\label{lem:2p2pperb}
Suppose $\relu{\cdot}$ satisfies Condition~\ref{cond:lip_and_smooth}
For a positive definite matrix $\mat{A} \in \mathbb{R}^{2p \times 2p}$, define \begin{align*}
\mat{F}(\mat{A}) = \expect_{\vect{U}\sim N(0,\mat{A})}\left[\relu{\vect{U}}\relu{\vect{U}}^\top\right], \\
\mat{G}(\mat{A}) =\expect_{ \vect{U}\sim N(\vect{0},\mat{A})}  \left[\relu{\vect{U}}\right].
\end{align*}
Then for any two positive definite matrices $\mat{A},\mat{B}$ with $\frac{1}{c}\le \mat{A}_{ii},\mat{B}_{ii} \le c$ for some constant $c>0$, we have \begin{align*}
\norm{\mat{G}(\mat{A})-\mat{G}(\mat{B})}_{\infty} \vee\norm{\mat{F}(\mat{A})-\mat{F}(\mat{B})}_{\infty} \le C\norm{\mat{A}-\mat{B}}_{\infty}
\end{align*}
for some constant $C>0$.
\end{lem}

\begin{proof}[Proof of Lemma~\ref{lem:2p2pperb}]
The result follows by applying Lemma~\ref{lem:1perb} to all coordiniates and applying Lemma~\ref{lem:22perb} to all $2\times 2$ submatrices.
\end{proof}

\begin{lem}\label{lem:linear_ind}
If $\vect{v}_1,\ldots,\vect{v}_n \in \mathbb{R}^d$ satisfy that $\norm{\vect{v}_i}_2 = 1$ and non-parallel (meaning $\vect{v}_i \notin \text{span}(\vect{v}_j)$ for $i \neq j$), then the matrix $\begin{bmatrix}
\vectorize{\vect{v}_1^{\otimes n}},\ldots, \vectorize{\vect{v}_n^{\otimes n}} 
\end{bmatrix}
\in \mathbb{R}^{d^n \times n}$ has rank-$n$.
\end{lem}
\begin{proof}[Proof of Lemma~\ref{lem:linear_ind}]
We prove by induction. For $n=2$, $v_1 v_1 ^\top, v_2 v_2 ^\top$ are linearly independent under the non-parallel assumption. 
By induction suppose $\{\vectorize{\vect{v}_1^{\otimes n-1}},\ldots, \vectorize{\vect{v}_{n-1}^{\otimes n-1}} \}$ are linearly independent. 
Suppose the conclusion does not hold, then there exists $\alpha_1,\ldots,\alpha_n \in \mathbb{R}$ not identically $0$, such that \begin{align*}
\sum_{i=1}^n \alpha_i \vectorize{\vect{v}_i^{\otimes n}} = 0,
\end{align*}
which implies for $p=1,\ldots,d$\begin{align*}
	\sum_{i=1}^{n} (\alpha_i \vect{v}_{i,p}) \vectorize{\vect{v}_i^{\otimes (n-1)}} = 0.
\end{align*}
Note by induction hypothesis any size $(n-1)$ subset of\\ $\left\{\vectorize{\vect{v}_1^{\otimes (n-1)}},\ldots,\vectorize{\vect{v}_n^{\otimes (n-1)}}\right\}$ is linearly independent.
This implies if $\alpha_i \vect{v}_{i,p} = 0$ for some $i \in [n]$ and $p\in[d]$, then we must have $\alpha_{j} \vect{v}_{j,p} = 0$ for all $j \in [n]$.
Combining this observation with the assumption that every $\vect{v}_i$ is non-zero, there must exist  $p \in [d]$ such that $\vect{v}_{i,p} \neq 0$ for all $i \in [n]$. Without loss of generality, we assume $\vect{v}_{i,1} \neq 0 $ for all $i \in [n]$.

Next, note if there exists $\alpha_i = 0$, then we have $\alpha_j =0$ for all $j \in [n]$ because $\vect{v}_{j,p} \neq 0$ for all $j \in [n]$ and the linear independence induction hypothesis. Therefore from now on we assume $\alpha_i \neq 0$ for all $i \in [n]$.

For any $p \in [d]$ ,
we have \begin{align*}
\sum_{i=1}^{n} (\alpha_i \vect{v}_{i,p}) \vectorize{\vect{v}_i^{\otimes (n-1)}} = 0\text{ and } \sum_{i=1}^{n} (\alpha_i \vect{v}_{i,1}) \vectorize{\vect{v}_i^{\otimes (n-1)}} = 0.
\end{align*}

By multiplying the second equation by $\frac{\vect{v}_{1,p}}{\vect{v}_{1,1}}$ and subtracting, \begin{align*}
\sum_{i=2}^{n} (\alpha_i \vect{v}_{i,p} -\alpha_i \frac{\vect{v}_{1,p}}{\vect{v}_{1,1}} \vect{v}_{i,1}) \vectorize{\vect{v}_i^{\otimes (n-1)}}=0.
\end{align*}
Using the linear independence induction hypothesis,  we know for $i =2,\ldots,n$:\begin{align*}
\frac{\vect{v}_{i,p}}{\vect{v}_{1,1}}= \frac{\vect{v}_{1,p} }{\vect{v}_{1,1}}.
\end{align*} 
Therefore we know \begin{align*}
\frac{\vect{v}_{1,p}}{\vect{v}_{1,1}} = \cdots = \frac{\vect{v}_{n,p}}{\vect{v}_{n,1}}.
\end{align*}

Thus there exists $c_2,\ldots,c_d \in \mathbb{R}^d$ such that \begin{align*}
	\vect{v}_{i,p} = c_p\vect{v}_{i,1} \text{ for all $i \in[n]$.}
\end{align*}
Note this implies all $\vect{v}_i$, $i \in [n]$ are on the same line.
This contradicts with the non-parallel assumption.
\end{proof}